\documentclass{article}

\usepackage[final]{neurips_2025}

\usepackage[utf8]{inputenc} 
\usepackage[T1]{fontenc}    
\usepackage{hyperref}       
\usepackage{url}            
\usepackage{booktabs}       
\usepackage{amsfonts}       
\usepackage{nicefrac}       
\usepackage{microtype}      
\usepackage{xcolor}         

\usepackage{graphicx}
\usepackage{subfigure}
\usepackage{booktabs} 

\usepackage[ruled]{algorithm2e}
\usepackage{float}

\usepackage{amsmath}
\usepackage{amssymb}
\usepackage{amsthm}
\usepackage{fmtcount}
\usepackage{mathtools}
\usepackage{bigints}    
\usepackage[capitalize,noabbrev]{cleveref}
\usepackage{multirow}

\theoremstyle{plain}
\newtheorem{theorem}{Theorem}[section]

\newtheorem{lemma}[theorem]{Lemma}

\theoremstyle{definition}

\theoremstyle{remark}

\def\cC{\mathcal C}

\def\cF{\mathcal F}

\def\cH{\mathcal H}

\def\cN{\mathcal N}

\def\cX{\mathcal X}

\newcommand{\E}{\mathbb{E}}

\newcommand{\bc}{\begin{center}}
\newcommand{\ec}{\end{center}}
\newcommand{\be}{\begin{equation}}
\newcommand{\ee}{\end{equation}}
\newcommand{\been}{\begin{equation*}}
\newcommand{\eeen}{\end{equation*}}
\newcommand{\ba}{\begin{array}}
\newcommand{\ea}{\end{array}}
\newcommand{\bean}{\setlength\arraycolsep{2pt}\begin{eqnarray*}}
\newcommand{\eean}{\end{eqnarray*}}
\newcommand{\bea}{\setlength\arraycolsep{2pt}\begin{eqnarray}}
\newcommand{\eea}{\end{eqnarray}}
\newcommand{\ben}{\begin{enumerate}}
\newcommand{\een}{\end{enumerate}}
\newcommand{\bed}{\begin{itemize}}
\newcommand{\eed}{\end{itemize}}

\usepackage{kotex}

\title{Knowledge Distillation of Uncertainty \\ using Deep Latent Factor Model}

\author{%
  Sehyun Park  \\
  Department of Statistics\\
  Seoul National University\\
  \texttt{ps\_hyen@snu.ac.kr}
  \And
  Jongjin Lee\\
  Samsung Research\\
  \texttt{ga0408@snu.ac.kr}
  \And
  Yunseop Shin\\
  Department of Statistics\\
  Seoul National University\\
  \texttt{dbstjq48@snu.ac.kr}
  \AND
  Ilsang Ohn\\
  Department of Statistics\\
  Inha University\\
  \texttt{ilsang.ohn@inha.ac.kr}
  \And
  Yongdai Kim\thanks{Corresponding author.}\\
  Department of Statistics\\
  Seoul National University\\
  \texttt{ydkim0903@gmail.com}
}

\begin{document}

\maketitle

\begin{abstract}
Deep ensembles deliver state-of-the-art, reliable uncertainty quantification, but their heavy computational and memory requirements hinder their practical deployments to real applications such as on-device AI. 
Knowledge distillation compresses an ensemble into small student models, but existing techniques struggle to preserve uncertainty partly because reducing the size of DNNs typically results in variation reduction. 
To resolve this limitation, we introduce a new method of distribution distillation (i.e. compressing a teacher ensemble into a student distribution instead of a student ensemble) called Gaussian distillation, which estimates the distribution of a teacher ensemble through a special Gaussian process called the deep latent factor model (DLF) \footnote{The source code of DLF is publicly available at \url{https://github.com/sehyun1094/DLF}} by treating each member of the teacher ensemble as a realization of a certain stochastic process.
The mean and covariance functions in the DLF model are estimated stably by using the expectation-maximization (EM) algorithm. 
By using multiple benchmark datasets, we demonstrate that the proposed Gaussian distillation outperforms existing baselines. In addition, we illustrate that Gaussian distillation works well for fine-tuning of language models and distribution shift problems.
\end{abstract}

\section{Introduction} \label{section: intro}

While DNNs have succeeded tremendously in various AI tasks, the rapid increase in their model sizes has raised a concern about high computational resource demands, which limits  their applications to real world applications such as
on-device AI \citep{sanh2019distilbert}, and thus developers have increasingly compressed large-scale language models into much smaller models \citep{sanh2019distilbert, gu2023minillm, team2023gemini, abdin2024phi}. 
A representative tool for compression is knowledge distillation (KD), which constructs a smaller DNN that mimics a given large-scale DNN \citep{hinton2015distilling}. 

Another concern is that the inherent over-parameterization of a single DNN makes them susceptible to overfitting, leading to overconfident predictions. 
When training predictive models, it is essential to learn models not only accurate but also reliable.
For reliable prediction, proper quantification of uncertainty has become an important topic in AI research \citep{gal2016uncertainty, malinin2018predictive, mariet2020distilling}. 
Deep ensemble (an ensemble of multiple DNNs) has received much attention not only for its strong predictive performance but also for its ability to quantify prediction uncertainty \citep{lakshminarayanan2017simple, hendrycks2019benchmarking}.
An ensemble of DNNs can mitigate overconfident predictions by reflecting the uncertainty (i.e., the variation of multiple predictions made by members of an ensemble) when making a final decision.

Since deep ensemble requires even more computational resources than DNNs, KD of deep ensemble is
necessary for improving its applicability.
Several works have focused on distilling a given teacher ensemble to a student ensemble instead of 
distilling a teacher ensemble to a single student DNN to keep the uncertainty as much as possible \citep{mariet2020distilling, tran2020hydra, wen2020batchensemble, nam2021diversity, nam2022improving, zhang2023adaptive}. Algorithms for KD of deep ensemble can be roughly categorized into two approaches: one-to-one distillation and distribution distillation. 
One-to-one distillation compresses each member in a teacher ensemble to a smaller DNN, which becomes a member of a student ensemble. 
Various weight-sharing architectures for student DNNs, along with their corresponding learning algorithms, have been proposed  \citep{tran2020hydra, nam2021diversity, nam2022improving, ferianc2022simple}. 
On the other hand, distribution distillation treats each ensemble member in a teacher ensemble as an independent realization of a certain distribution whose parameters are modeled by a student DNN. 
\citep{malinin2019ensemble} and \citep{ryabinin2021scaling} assume that the conditional class probability vector of each member in a teacher ensemble follows a Dirichlet distribution and devise a method to estimate the parameters in the Dirichlet distribution using a student DNN.

There are still limitations in existing KD methods for deep ensemble.
One-to-one distillation methods tend to lose a significant amount of uncertainty in a teacher ensemble when they compress large DNNs into smaller ones, 
while performance of Dirichlet distillation (the distribution distillation with a Dirichlet distribution) is inferior to one-to-one distillation partly \citep{tran2020hydra, nam2022improving} because of instability in learning the parameters in the Dirichlet distribution.

The aim of this paper is to propose a new distribution distillation method that is numerically stable in learning and superior to other baselines in uncertainty quantification. 
In our proposed method, we treat each member in a teacher ensemble as an independent realization of a Gaussian process and estimate the mean and covariance functions of the Gaussian process based on observed predictions of 
members in a teacher ensemble.
For this purpose, we propose the deep latent factor (DLF) model where the mean and covariance functions are modeled by a student DNN and implement an EM algorithm to estimate the maximum likelihood estimator (MLE) of the student DNN. We call our method {\it Gaussian distillation}.

Our contributions are summarized as follows.
\begin{itemize}
	\item We propose a new distribution distillation method based on a specially designed Gaussian process called the DLF model that achieves superior performance in uncertainty quantification to other baselines.
	\item We develop an EM algorithm to estimate the student DNN in the DLF model. In particular,
	we propose a way of finding a good initial solution by maximizing
    the penalized complete log-likelihood.
	\item We do numerical experiments to show that Gaussian distillation outperforms other baselines for both regression and classification. We also illustrate that Gaussian distribution is a useful tool for fine-tuning language models.
    \item We apply the pre-trained DLF to distribution shift problems and show numerically that it outperforms baselines.
\end{itemize}

\section{Preliminaries} \label{section: preliminaries}

\subsection{Prediction uncertainty}

In a nutshell, quantifying prediction uncertainty in supervised tasks involves efficiently estimating the predictive distribution of the output $y$ given a new input denoted as $p(y|\boldsymbol{x}^{\text new})$.
The variation in the predictive distribution can be used as a measure of uncertainty in prediction.

A typical way of estimating the predictive distribution begins with a parametric generative model for the input and output pair.
Let $p(y|\boldsymbol{x},\theta)$ be the conditional distribution of the output $y\in \mathcal{Y}$ given an input $\boldsymbol{x} \in \mathcal{X}\subset \mathbb{R}^d,$ where 
$\theta \in \Theta$ is an unknown parameter.
Then, we try to estimate $\theta$ based on training data $(\boldsymbol{x}_1,y_1),\ldots,(\boldsymbol{x}_m,y_m)$ such that $p(y|\boldsymbol{x},\hat\theta)$ is as close as possible to $p^*(y|\boldsymbol{x}),$ where $\hat\theta$ is an estimate of $\theta$ and $p^*(y|\boldsymbol{x})$ is the true conditional distribution.
For example, the MLE minimizes the empirical KL divergence between $p(y|\boldsymbol{x},\theta)$ and $p^*(y|\boldsymbol{x}).$

It is well known, however, that the variation in $p(y|\boldsymbol{x},\hat\theta)$ is smaller than that in $p^*(y|\boldsymbol{x})$ because $p(y|\boldsymbol{x},\hat\theta)$ does not take into account the uncertainty in estimating $\hat\theta$. Thus, making a decision solely with $p(y|\boldsymbol{x},\hat\theta)$ would lead in overconfident results.
A proper uncertainty quantification in prediction should consider not only uncertainty in $p^*(y|\boldsymbol{x})$ (aleatory) \citep{gal2016uncertainty, malinin2018predictive} but also uncertainty in $\hat\theta$ (epistemic).

A popular way of considering both aleatory and epistemic uncertainties in prediction is to use an ensemble.
We construct multiple estimates $\hat\theta_1,\ldots,\hat\theta_n$ of $\theta$ and then estimate the predictive distribution as
$\hat{p}(y|\boldsymbol{x})= \sum_{i=1}^n p(y|\boldsymbol{x},\hat\theta_i)/n,$ which we call
the averaged prediction model.
For deep learning, the two most representative methods of constructing multiple estimates are deep ensemble \citep{lakshminarayanan2017simple, dietterich2000ensemble, laurent2022packed} and Bayesian DNNs \citep{neal1996monte, welling2011bayesian, gal2016dropout, wilson2020bayesian, izmailov2021bayesian, sharma2023bayesian, kong2023masked}.
Deep ensemble generates multiple estimates by learning a DNN with different initial parameter, while Bayesian DNNs generate $\hat\theta$s from the posterior distribution.
In this paper, we focus on deep ensemble, but our proposed method can be applied to Bayesian DNNs without modification.

\subsection{Review of ensemble distillation}

As mentioned in Introduction, deep ensemble has an intrinsic limitation in its practical applications due to high computational costs and times along with demands for substantial memory to store and process multiple prediction models.
To resolve this problem, KD of deep ensemble has received much attention \citep{hinton2015distilling,mariet2020distilling,tran2020hydra,wen2020batchensemble, nam2021diversity, nam2022improving, zhang2023adaptive,malinin2019ensemble}. 
A basic idea of KD of deep ensemble is to approximate large DNNs in a teacher ensemble by smaller student DNNs.
A naive approach of KD of deep ensemble is to approximate the averaged prediction model
$\hat{p}(y|\boldsymbol{x})$ of a teacher ensemble by a small single DNN \citep{zhang2023adaptive, fukuda2017efficient, kwon2020adaptive}.
This naive approach, however, does not perform well since it is hard to distill the uncertainty in $\hat{p}(y|\boldsymbol{x})$ into a single student DNN.

A remedy is to distill a teacher ensemble
into a student ensemble.
Several methods have been proposed for this purpose, which can be roughly divided into two categories that are explained in the subsequent subsections.

\subsubsection{One-to-one distillation}\label{one-to-one distillation}
The main idea of one-to-one distillation is to construct multiple student models, each of which corresponds to each teacher model. That is for given $n$ many teacher models $p_{i}^{(t)}(y|\boldsymbol{x}), i=1,\ldots,n,$ $n$ many student models $p_{i}^{(s)}(y|\boldsymbol{x})$ are constructed. 
To save computation time and memory further,
various special neural network architectures for $n$ student models $p_{i}^{(s)}(y|\boldsymbol{x}), i=1,\ldots,n$
have been proposed. Examples are \textit{Hydra} \citep{tran2020hydra}, 
\textit{Batch Ensembles} (BE) and \textit{Latent Batch Ensemble} (LBE) \citep{wen2020batchensemble, nam2022improving}.
See Appendix \ref{appendix:one-to-one distillation} for details.
\subsubsection{Distribution distillation}

Distribution distillation assumes that teacher models are independent realizations of a stochastic model with
unknown parameters modeled by a student DNN and
estimates the student DNN based on the prediction values of the teacher models \citep{malinin2019ensemble, ryabinin2021scaling}.
To be more specific, for classification problems, we 
assume that  $\big(p^{(t)}_i(y|\boldsymbol{x}\big), y=1,\ldots,c), i=1,\ldots,n$ for a given $\boldsymbol{x}$
are independently generated from the Dirichlet distribution with 
parameters $\alpha_1(\boldsymbol{x}),\ldots,\alpha_c(\boldsymbol{x})$ and 
model these parameters by a student DNN.
Once the student DNN is learned, ensemble members
are generated from the learned Dirichlet distribution
and aggregated in the prediction phase. We call this method {\it{Dirichlet distillation}}.
See Appendix \ref{appendix:distribution distillation} for details.

\section{The Proposed Method} \label{section:proposed method}

We propose a new method of distribution distillation.
The main idea of the proposed method is that we treat members in a teacher ensemble as independent realizations of a Gaussian process and estimate the mean and covariance functions of the Gaussian process by a student DNN.
Then, in the inference phase, we generate ensemble members from the estimated Gaussian process.
We call our proposed method \textit{Gaussian distillation}.
See Figure \ref{fig:overal process DLF} for the overall process of Gaussian distillation.
\begin{figure}[h]
    \centering
    \includegraphics[width=\linewidth]{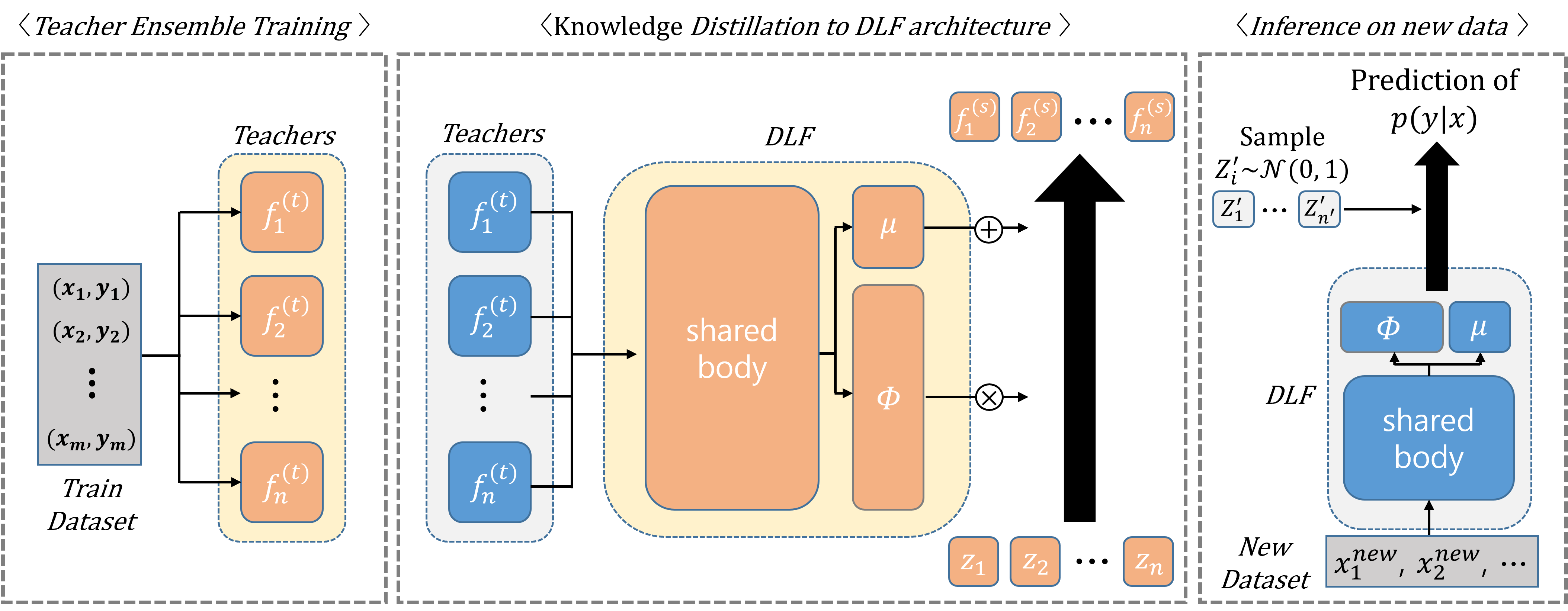}
    \caption{Overall process of Gaussian distillation}
    \label{fig:overal process DLF}
\end{figure}

A technical difficulty of this idea is to model and estimate the covariance function. 
To resolve this problem, we use the DLF, which is an extension of the standard linear factor model \citep{bartholomew2011latent}, where the mean and factor loading are modeled by a student DNN.

For the probabilistic model of data, we consider $y=f(\boldsymbol{x})+\epsilon,$
where $\epsilon\sim \mathcal{N}(0,\sigma_\epsilon^2)$ for regression problems and
$p(y|\boldsymbol{x})=\exp(f_y(\boldsymbol{x}))/\sum_{v=1}^c \exp(f_v(\boldsymbol{x}))$
for classification problems.
Thus, a teacher ensemble for regression problems consists of multiple teacher models 
for $f$ as well as multiple estimates of $\sigma_\epsilon^2$, while
a teacher ensemble for classification problems consists of multiple teacher models for
multivariate functions $\boldsymbol{f}(\cdot)=(f_1(\cdot),\ldots, f_c(\cdot))$.



\subsection{Deep Latent Factor model}\label{Deep Latent Factor model}

In this subsection, we introduce special Gaussian processes for $f(\cdot)$ and $\boldsymbol{f}(\cdot),$ respectively.

\paragraph{Univariate case}
The DLF model for a univariate random function $f:\mathcal{X} \rightarrow \mathbb{R}$ is defined as
\begin{equation}
	f(\cdot)=\mu_{\theta}(\cdot)+\Phi_{\theta}(\cdot)^\top \boldsymbol{Z},
\end{equation}
where $\mu_{\theta}(\cdot):\mathcal{X}\rightarrow \mathbb{R}$ is the mean function, $\Phi_{\theta}(\cdot):\mathcal{X} \rightarrow \mathbb{R}^q$ is the factor loading function and $\boldsymbol{Z} \sim \mathcal{N}_q(\boldsymbol{0}, \mathbb{I}_q)$ is the latent factor. Here, $\mathcal{N}_q$ is the $q$-dimensional Gaussian distribution and $\mathbb{I}_q$ is the $q$-dimensional identity matrix. In the DLF model, we set $(\mu_{\theta}(\cdot),\Phi_{\theta}(\cdot)^\top)$ by a student DNN parameterized by $\theta$ which has $q+1$ output nodes. 

It is easy to see that the DLF is a Gaussian process with mean function $\mu_{\theta}(\cdot)$
and covariance function $\Sigma_{\theta}(\cdot,\cdot)=\Phi_{\theta}(\cdot) ^\top\Phi_{\theta}(\cdot).$
Once we have $n$ many teacher models $f_1(\cdot),\ldots, f_n(\cdot),$
we assume them to be independent realizations of the DLF model and estimate the mean and factor loading functions. 

\paragraph{Multivariate case}
The DLF model for a multivariate function $\boldsymbol{f}(\cdot)=(f_1(\cdot),\ldots,f_c(\cdot))^\top$
is  defined as 
\begin{equation}
    \boldsymbol{f}(\cdot) = \mu_{\theta}(\cdot) + L \boldsymbol{Z} \Phi_{\theta}(\cdot) , 
\end{equation} 
where $\mu_{\theta}(\cdot):\mathcal{X}\rightarrow \mathbb{R}^{c}$ is the mean function,
$\Phi_{\theta}(\cdot):\mathcal{X} \rightarrow \mathbb{R}^q$ is the factor loading function, $L \in \mathbb{R}^{c \times c}$ is a lower-triangular matrix and $\boldsymbol{Z} \sim  \mathcal{MN}_{c, q}(0,\mathbb{I}_{c},\mathbb{I}_{q}).$  Here,  $\mathcal{MN}_{c, q}(0,\mathbb{I}_{c},\mathbb{I}_{q})$ is a matrix-variate Gaussian distribution.
It can be shown that the DLF model is a multivariate Gaussian process $\mathcal{MGP}_{c}(\mu_{\theta}, \Sigma, \Lambda)$ with the mean function $\mu_{\theta}(\cdot),$  covariance function $\Sigma(\cdot,\cdot)=\Phi_{\theta}(\cdot)^\top\Phi_{\theta}(\cdot)$ and parameter matrix $\Lambda = L L^\top$.
For the definition of multivariate Gaussian process, see \citep{chen2023multivariate}.

\subsection{Estimation of the mean and factor loading}\label{Estimation of the mean and factor loading}

The main idea of Gaussian distillation is to estimate
the mean and factor loading functions by maximizing the corresponding log-likelihood, assuming that teacher models are independent realizations of the DLF. 
For optimization, we use the EM algorithm \citep{rubin1982algorithms}.
In this section, we explain the EM algorithm for Gaussian distillation.
For ease of notation, we only consider the univariate DLF model, and refer to Appendix \ref{EM algorithm when Multivariate} for the multivariate DLF.

Suppose that $n$ many teacher models $f_1(\cdot),\ldots, f_n(\cdot)$ are given.
Gaussian distillation consists of three steps.
The first step is to choose $m$-many design points $\mathcal{D}^{\mathrm{design}} = \{\boldsymbol{x}_{1}^{(d)}, \ldots, \boldsymbol{x}_{m}^{(d)}\}$.
We will discuss how to choose the design points in Section \ref{choice of design points}. 
The second step is to calculate the vectors of prediction values of each teacher model at the design points to have $\boldsymbol{f}_i=\big(f_i(\boldsymbol{x}_j^{(d)}), j=1,\ldots,m\big)^\top$ for $i=1,\ldots,n.$
The final step is to estimate the parameter $\theta$ 
in the DLF model assuming that
$f_1(\cdot),\ldots, f_n(\cdot)$ are independent realizations of a random function following the DLF model.
Since $\boldsymbol{f}_i$s are independent Gaussian random vectors, the MLE can be obtained by use of the EM algorithm as follows.

To make the EM algorithm numerically stable, we consider the noisy DLF model which assumes that $\boldsymbol{f}_i=\boldsymbol{\tilde{f}}_i+\boldsymbol{v}_i,$ where $\boldsymbol{v}_i \sim \mathcal{N}_m(\mathbf{0}, \sigma_f^2 \mathbb{I}_m)$ and $\boldsymbol{\tilde{f}}_i=(\tilde{f}_i(\boldsymbol{x}_1),\ldots,\tilde{f}_i(\boldsymbol{x}_m))^\top$ with $\tilde{f}_i(\cdot)$s following the DLF model. 
Specifically, each $\tilde{f}_i$ is expressed as $\tilde{f}_i(\cdot) = \mu_{\theta}(\cdot) + \Phi_{\theta}(\cdot)^\top \boldsymbol{z}_i,$ where $\boldsymbol{z}_i \sim \mathcal{N}_q(\mathbf{0}, \mathbb{I}_q)$ denotes the latent factor corresponding to the $i$-th function realization.
Then, we obtain the MLE of the parameter $\theta$ in the mean and factor loading functions as well as $\sigma_f^2.$
We abuse the notation to write $\theta=(\theta,\sigma_f^2)$ unless there is any confusion.

The complete log-likelihood is given as
\begin{equation}
\begin{split}    
    \ell^{com}(\theta | \boldsymbol{f}_{1:n}, \boldsymbol{z}_{1:n}) 
    =&-\frac{nm}{2} \log(2 \pi \sigma_{f}^{2}) - \frac{nq}{2} \log(2 \pi)  - \frac{ \sum_{i=1}^{n}\boldsymbol{z}_{i}^{\top} \boldsymbol{z}_{i}}{2} \\
    & - \frac{ \sum_{i=1}^{n} (\boldsymbol{f}_i - \boldsymbol{\mu}_{\theta} - \boldsymbol{\Phi}_{\theta} \boldsymbol{z}_{i}    )^\top (\boldsymbol{f}_i - \boldsymbol{\mu}_{\theta} - \boldsymbol{\Phi}_{\theta} \boldsymbol{z}_{i}  ) }{2 \sigma_{f}^{2}},
\end{split}
\end{equation}
where $\boldsymbol{f}_{1:n}=\{\boldsymbol{f}_1,\ldots,\boldsymbol{f}_n\}, 
\boldsymbol{z}_{1:n}=\{\boldsymbol{z}_1,\ldots,\boldsymbol{z}_n\}, \boldsymbol{\mu}_{\theta}=
(\mu_{\theta}(\boldsymbol{x}_1^{(d)}),\ldots, \mu_{\theta}(\boldsymbol{x}_m^{(d)}))^\top$
and $\boldsymbol{\Phi}_{\theta} = (\Phi_{\theta}(\boldsymbol{x}_1^{(d)}),\ldots,\Phi_{\theta}(\boldsymbol{x}^{(d)}_m))^\top$ is an $m \times q$ matrix.

For a given parameter $\theta^{(t-1)}$ at time $t-1,$
the E-step is to calculate the conditional expectation of the complete log-likelihood 
 $Q(\theta|\theta^{(t-1)})=\mathbb{E}_{\boldsymbol{z}_{1:n} \mid \boldsymbol{f}_{1:n}, \theta^{(t-1)}}[\ell^{com}(\theta | \boldsymbol{f}_{1:n}, \boldsymbol{z}_{1:n})],$ whose formula is given in Appendix \ref{Appendix:Estimation of the mean and factor loading}. 
In the M-step, we update $\theta^{(t)}$ by a stochastic gradient descent algorithm on mini-batches.
The EM algorithm is summarized in Algorithm \ref{alg:1} in Appendix \ref{Appendix:Estimation of the mean and factor loading}.

\paragraph{Choice of the initial parameter} Note that the EM algorithm may converge to a local optimum, and the choice of an initial solution significantly impacts the final estimate. For the DLF model, where the factor loading involves a complex DNN structure, this issue becomes even worse. Moreover, the identifiability issue of the factor loading
makes initializing the EM algorithm from a well-chosen starting point become even more crucial.

For searching a good initial solution, we pretrain the DLF model by maximizing the
following penalized complete log-likelihood with respect to $\boldsymbol{\theta}$ and $\boldsymbol{z}_i$s where
\begin{equation}
    \ell^{pen}(\boldsymbol{\theta},\boldsymbol{z}_{1:n}|\boldsymbol{f}_{1:n})=
       \ell^{com}(\boldsymbol{\theta}|\boldsymbol{f}_{1:n},\boldsymbol{z}_{1:n})
       - \lambda \mathcal{D}_{\operatorname{MMD}}(\boldsymbol{z}_{1:n},\boldsymbol{z}_{1:n}^{\prime} )
\end{equation}
for $\lambda>0,$ where $\boldsymbol{z}_{1:n}^{\prime}$ are samples generated from the standard Gaussian distribution and
$\mathcal{D}_{\operatorname{MMD}}$ is the Maximum Mean Discrepancy (MMD) with the RBF kernel.
The term MMD is introduced to make the distribution of the estimated $\boldsymbol{z}_{1:n}$ similar to the standard Gaussian distribution.

\paragraph{KD for $\sigma_\epsilon^2$} For regression problems, we need a KD method for $\sigma_\epsilon^2.$
Let $\sigma_{\epsilon,1}^2,\ldots,\sigma_{\epsilon,n}^2$ be estimators of $\sigma_\epsilon^2$ provided by each teacher model. We assume that $\sigma_{\epsilon,i}^2, i=1,\ldots,n$ are independently generated from the inverse
gamma distribution and estimate the parameter in the distribution accordingly.
In the inference phase, we generate ensemble members of $\sigma_\epsilon^2$ from the estimated inverse gamma distribution.

\subsection{Comparison with other baselines}
\label{sec3.4}

\paragraph{Comparison with Hydra}  
Recall that we model $(\mu_\theta(\cdot),\Phi_\theta(\cdot)^\top)^\top$
by a DNN with $q+1$ many heads. Note that the DLF model assumes
$f_i(\cdot)=\mu_\theta(\cdot)+\Phi_\theta(\cdot)^\top \boldsymbol{z}_i$
for $i=1,\ldots,n,$ and thus we can interpret $(\mu_\theta(\cdot),\Phi_\theta(\cdot)^\top)^\top$
as the body and $\boldsymbol{z}_i$s as the model-specific weights at the head.
In view of sharing the body, the DLF model is quite similar to the conventional multi-head structure used in Hydra \citep{tran2020hydra}. The main difference is that the DLF model treats $\boldsymbol{z}_i$s as random quantities and thus
integrates out before estimating the MLE while Hydra treats $\boldsymbol{z}_i$s as fixed effects and
estimates $\theta$ and $\boldsymbol{z}_i$s simultaneously by minimizing a given loss.
It is well-known that treating random effects as fixed effects is highly susceptible to bias \citep{breslow1995bias, lin1996bias, engel1998simple, skrondal2004generalized}. Our experimental results in Section \ref{experiments} amply demonstrate that treating $\boldsymbol{z}_i$s as random is better than
treating $\boldsymbol{z}_i$s as fixed effects.

\paragraph{Comparison with Dirichlet distillation}
At least, there are two advantages of Gaussian distillation compared to  Dirichlet distillation.
Gaussian distillation can be applied to both regression and classification models while Dirichlet distillation is only applicable to classification problems.
The second advantage is that estimation of the mean and factor loading in the DLF model is easier than estimation of the parameter in the Dirichlet distribution owing to the nice EM algorithm. 
This stability makes Gaussian distillation perform better than Dirichlet distillation. 
The inferior performance of Dirichlet distillation compared to one-to-one distillation, as reported in \citep{tran2020hydra, nam2021diversity}, is confirmed by our experiments in Section \ref{experiments}.

\subsection{Choice of design points}\label{choice of design points}

Let $\hat{\mu}(\cdot)$ and $\hat{\Sigma}(\cdot,\cdot)$ be the estimate of $\mu_*(\cdot)$ and $\Sigma_*(\cdot,\cdot)$ by the DLF model with design points $\mathcal{D}^{\mathrm{design}}$.
For a given $\boldsymbol{x}\in \mathcal{X}$, let $\hat{p}_{\boldsymbol{x}}$ and $p_{*,\boldsymbol{x}}$ be the distributions of $f(\boldsymbol{x})$ under the assumption that $f(\cdot)$ is a Gaussian process with the parameters $(\hat{\mu},\hat\Sigma)$ and $(\mu_*,\Sigma_*),$ respectively.
In Theorem \ref{theory-new} in Appendix \ref{Theorical Results}, we prove that $
\sup_{\boldsymbol{x}\in \mathcal{D}^{\mathrm{design}}} d_1(\hat{p}_{\boldsymbol{x}},p_{*,\boldsymbol{x}})$ converges to 0 as $n\rightarrow \infty$ if we choose the architecture of a student DNN for $(\mu_\theta(\cdot),\Phi_\theta(\cdot)^{\top})$ appropriately, where $d_1$ is the $\ell_1$ metric.

For $\boldsymbol{x}\not\in \mathcal{D}^{\mathrm{design}},$ if
$\hat{\mu}$ and $\mu_*$, as well as $\hat\Sigma$ and $\Sigma_*$ are (coordinate-wise) Lipschitz, it can be shown that
$ d_1(\hat{p}_{\boldsymbol{x}},p_{*,\boldsymbol{x}})
\le d_1(\hat{p}_{\boldsymbol{x}_{(1)}},p_{*,\boldsymbol{x}_{(1)}})+C \|\boldsymbol{x}-\boldsymbol{x}_{(1)}\|$
for a positive constant $C,$ where $\boldsymbol{x}_{(1)}$ is the nearest point in $\mathcal{D}^{\mathrm{design}}$ to 
$\boldsymbol{x}.$ See Theorem \ref{theory-new} in Appendix \ref{Theorical Results}.
Note that the term $ \|\boldsymbol{x}-\boldsymbol{x}_{(1)}\|$ is affected by the choice of design points.
Suppose that $\boldsymbol{x}$ is a realization of a random vector $\boldsymbol{X}\sim \mathbb{P}.$
Then, the expected nearest-neighbor distance 
$\E_{\boldsymbol{X}\sim \mathbb{P}} \|\boldsymbol{X}-\boldsymbol{X}_{(1)}\|$ 
becomes smaller when the design points are located
in a higher-density region of $\mathbb{P}$. This observation suggests that
design points similar to test data would be better. Validation data (dataset whose distribution is the same as
the training data used for learning a teacher ensemble)  would be a promising candidate for the design points. 
See Appendix \ref{ablation:design points type} for numerical experiments.

\section{Application to distribution shift problems} \label{application distribution shift}

The pre-trained DLF can be applied to distribution shift problems.
We say that given new data is shifted in distribution if the distribution of new data is different from that of training data. Distribution shift problems, whose aim is to efficiently learn a prediction model on new data when the size of new data is small, have been studied extensively \citep{lei2021near, wu2021online, bai2022adapting, baby2023online, garg2023rlsbench, rosenfeld2023almost}.
A popular method is to learn a DNN on training data first and retrain the head of the DNN on new data while the body is fixed \citep{kumar2022fine, lee2022surgical}.

Note that the DLF model is given as
\begin{equation}
\label{eq:ds-1}
f_j(\cdot)=\hat\mu_j(\cdot)+\sum_{k=1}^{q} \sum_{l=1}^c \hat{\Phi}_k(\cdot) \hat{L}_{jl} z_{jlk},
\end{equation}
for $j=1,\ldots,c,$
where $z_{jkl}$s are independent standard Gaussian random variables.
For distribution shift problems, we can treat $(\hat{\mu}(\cdot),\hat{\Phi}(\cdot)^\top, \hat{L})$ as a learned body and
$z_j$s are the weights in the prediction head.
Then we learn only the weights of the head on new data while fixing the body.
In Section \ref{Application to distribution shift problems}, we show empirically that this method outperforms
its competitors. The superior performance of the DLF model for distribution shift problems indicates that Gaussian distillation is good at not only uncertainty quantification but 
estimating the feature vector $(\hat{\mu}(\cdot),\hat{\Phi}(\cdot)^\top).$

\section{Experiments} \label{experiments}

In this section, we investigate  Gaussian distillation by analyzing multiple benchmark datasets.
We compare Gaussian distillation with existing baselines including the naive distillation
(one-to-one distillation without sharing weights between student DNNs, small-Ens), Hydra \citep{tran2020hydra} and BE \citep{wen2020batchensemble} for regression and classification problems as well as fine-tuning of language models in view of uncertainty quantification. For classification, we also evaluate Proxy–Dirichlet Distillation (Proxy-End$^{2}$) \citep{ryabinin2021scaling} and Ensemble Distillation via Flow Matching (EDFM) \citep{parkensemble}.
In addition, we show that a pre-trained DLF outperforms its competitors for distribution shift problems.

\subsection{Uncertainty quantification for regression and classification problems}

\subsubsection{Regression case}

\paragraph{Datasets}
We analyze six benchmark datasets from the UCI repository \citep{asuncion2007uci}
including Boston housing, Concrete, Energy, Wine, Power Plant, and Kin8nm. 
Each dataset is randomly split into 90\% training and 10\% testing, and teacher models are trained following the experimental protocol of \citep{bui2016deep}. 
We repeat this procedure 10 times to obtain 10 measures of the evaluation metrics of each methods and report the averages (with the standard errors). 
See Appendix \ref{experimental details: regression case} for details of implementation.

\paragraph{Results}
Table \ref{regression_result} presents the results of the four evaluation metrics (see Appendix \ref{Evaluation metric:Reg} for the definitions) 
for performance and uncertainty quantification. DLF outperforms Hydra and BE in most cases.
Even when it is not the best, DLF is at least the second best.
The coverage probabilities of Hydra and BE are sometimes much lower than those of DLF
(Boston housing and Concrete for Hydra, and Boston housing and kim8nm for BE),
which suggests that deterministic distillation methods fail to fully preserve the uncertainty in a teacher ensemble. 
This observation is not surprising since variation of smaller models is in general smaller than that of larger models and weight sharing would reduce the variation further.
Thus, a way of adding additional uncertainty to a student ensemble is needed, and distribution distillation is such a solution.

{\renewcommand{\arraystretch}{1.2}
\begin{table}[h]
\centering
\caption{Results on UCI benchmark datasets. ($\ast$ : closer to the coverage probability of a teacher ensemble is better)}
\label{regression_result}
\resizebox{\textwidth}{!}{%
\begin{tabular}{c|l|cccccc}
\toprule
\multirow{2}{*}{Metric} &
  \multirow{2}{*}{Method} &
  \multicolumn{6}{c}{Datasets} \\
\cline{3-8}
 &  & Boston housing & Concrete & Energy & Wine & Power & Kin8nm \\
\midrule
\multirow{5}{*}{RMSE $\downarrow$} 
& {\color[HTML]{00B0F0} Teachers}
  & {\color[HTML]{00B0F0} 2.5786}    & {\color[HTML]{00B0F0} 5.6191}    & {\color[HTML]{00B0F0} 0.5692}    & {\color[HTML]{00B0F0} 0.5497}    & {\color[HTML]{00B0F0} 4.2197}    & {\color[HTML]{00B0F0} 0.0794}    \\
& small-Ens 
  & 2.7280 (0.0184) & 5.6952 (0.0494) & 0.6367 (0.0182) & 0.6002 (0.0083) & 4.2430 (0.0864) & 0.0865 (0.0007) \\
& Hydra       
  & 2.8346 (0.0835) & 6.0558 (0.1366) & 0.6549 (0.0360) & 0.5689 (0.0114) & 4.2284 (0.0074) & 0.0932 (0.0023) \\
& BE          
  & 2.8375 (0.0729) & 5.9777 (0.1475) & 0.6661 (0.0354) & 0.556  (0.0187) & 4.2367 (0.0041) & 0.0962 (0.0059) \\
& DLF         
  & \textbf{2.6687 (0.1700)} & \textbf{5.6047 (0.1771)} & \textbf{0.5659 (0.0239)} & \textbf{0.5506 (0.0104)} & \textbf{4.2211 (0.0010)} & \textbf{0.0825 (0.0010)} \\
\midrule
\multirow{5}{*}{NLL $\downarrow$}
& {\color[HTML]{00B0F0} Teachers}
  & {\color[HTML]{00B0F0} 2.3850}    & {\color[HTML]{00B0F0} 3.1134}    & {\color[HTML]{00B0F0} 0.8533}    & {\color[HTML]{00B0F0} 0.7980}    & {\color[HTML]{00B0F0} 2.8586}    & {\color[HTML]{00B0F0} -1.1109}   \\
& small-Ens 
  & \textbf{2.4150 (0.0085)} & 3.1672 (0.0167) & 0.9829 (0.0263) & 0.8861 (0.0160) & 2.8622 (0.0191) & -1.032 (0.0108) \\
& Hydra       
  & 2.4843 (0.0478) & 3.2586 (0.0293) & 0.9914 (0.0660) & 0.8322 (0.0166) & 2.8604 (0.0017) & -0.9434 (0.0307) \\
& BE         
  & 2.4892 (0.0408) & 3.2241 (0.0332) & 0.9879 (0.0511) & \textbf{0.8065 (0.0138)} & 2.8628 (0.0010) & -0.9115 (0.0808) \\
& DLF         
  & 2.4346 (0.1147) & \textbf{3.1584 (0.0463)} & \textbf{0.8525 (0.0471)} & 0.8230 (0.0199) & \textbf{2.8591 (0.0010)} & \textbf{-1.0754 (0.0126)} \\
\midrule
\multirow{5}{*}{CRPS $\downarrow$}
& {\color[HTML]{00B0F0} Teachers}
  & {\color[HTML]{00B0F0} 1.4425}    & {\color[HTML]{00B0F0} 2.9926}    & {\color[HTML]{00B0F0} 0.3137}    & {\color[HTML]{00B0F0} 0.2962}    & {\color[HTML]{00B0F0} 2.3360}    & {\color[HTML]{00B0F0} 0.0443}    \\
& small-Ens 
  & 1.5233 (0.0087) & \textbf{2.9953 (0.0304)} & 0.3461 (0.0077) & 0.3307 (0.0067) & 2.3392 (0.0321) & 0.0475 (0.0005) \\
& Hydra       
  & 1.6041 (0.0494) & 3.3084 (0.0639) & 0.3622 (0.0216) & 0.3075 (0.0040) & 2.3405 (0.0048) & 0.0518 (0.0012) \\
& BE          
  & 1.6158 (0.0490) & 3.2320 (0.0894) & 0.3617 (0.0168) & 0.3020 (0.0059) & 2.3483 (0.0024) & 0.0532 (0.0035) \\
& DLF         
  & \textbf{1.4317 (0.1029)} & 3.0622 (0.0954) & \textbf{0.3163 (0.0119)} & \textbf{0.2980 (0.0043)} & \textbf{2.3364 (0.0010)} & \textbf{0.0458 (0.0005)} \\
\midrule
\multirow{5}{*}{\shortstack{95\%\\Coverage\\Probability $\ast$}} 
& {\color[HTML]{00B0F0} Teachers} 
  & {\color[HTML]{00B0F0} 0.9608}    & {\color[HTML]{00B0F0} 0.9515}    & {\color[HTML]{00B0F0} 1.0000}    & {\color[HTML]{00B0F0} 0.9750}    & {\color[HTML]{00B0F0} 0.9697}    & {\color[HTML]{00B0F0} 0.9610}    \\
& small-Ens 
  & \textbf{0.9408 (0.0016)} & \textbf{0.9431 (0.0041)} & 0.9921 (0.0011) & 0.9569 (0.0042) & 0.9669 (0.0012) & \textbf{0.9649 (0.0016)} \\
& Hydra       
  & 0.8995 (0.0109) & 0.9097 (0.0115) & 0.9948 (0.0091) & 0.9594 (0.0053) & 0.9782 (0.0003) & 0.9407 (0.0046) \\
& BE          
  & 0.9093 (0.0090) & 0.9282 (0.0099) & 0.9922 (0.0110) & 0.9612 (0.0092) & 0.9778 (0.0014) & 0.9080 (0.0213) \\
& DLF         
  & 0.9240 (0.0154) & 0.9291 (0.0125) & \textbf{1.0000 (0.0000)} & \textbf{0.9681 (0.0055)} & \textbf{0.9711 (0.0053)} & 0.9761 (0.0036) \\
\bottomrule
\end{tabular}%
}
\end{table}
}

\subsubsection{Classification case}



\paragraph{Datasets}
CIFAR-10 and CIFAR-100 consist of 50,000 training and 10,000 test images.
In this experiment, the training data are further split into 80\% training and 20\% validation, and 
teacher models are trained on the training data and the number of epochs is determined by the validation data.
Implementation details of the distillation methods are given in Appendix \ref{experimental details: classification case}.
Experiments are repeated with 5 different random initializations for each method.

\paragraph{Results}
As shown in Table \ref{classification_result}, 
DLF outperforms the other baselines consistently in terms of not only uncertainty quantification but also accuracy.
In particular, improvements of DLF with respect to ECE are noticeable.
The definitions of the evaluation metrics are given in Appendix \ref{Evaluation metric:Cls}.

{\renewcommand{\arraystretch}{1.2}
\begin{table}[h]
\centering
\caption{Results on CIFAR-10 and CIFAR-100.}
\label{classification_result}
\resizebox{0.65\textwidth}{!}{%
\begin{tabular}{c|l|ccc}
\toprule
dataset      & method                        & Acc(\%) $\uparrow$            & NLL $\downarrow$                   & ECE(\%) $\downarrow$            \\
\midrule
\multirow{7}{*}{CIFAR-10}
 & {\color[HTML]{00B0F0} Teachers} & {\color[HTML]{00B0F0} 94.24} & {\color[HTML]{00B0F0} 0.1539}  & {\color[HTML]{00B0F0} 0.9}    \\
 & small-Ens                     & 92.87 (0.35)       & 0.2377 (0.0019)       & 3.93 (0.14)         \\
 & Hydra                          & 93.16 (0.06)       & 0.2660 (0.0063)       & 4.15 (0.01)         \\
 & LBE                            & 93.25 (0.26)       & 0.2480 (0.0053)       & 4.11 (0.10)         \\
 & $\text{Proxy-EnD}^2$               & 90.92 (0.28)       & 0.2861 (0.0027)       & 2.08 (0.19)         \\
 & EDFM                           & 90.62 (0.23)       & 0.2858 (0.0025)       & 2.78 (0.17)         \\
 & DLF                            & \textbf{93.40 (0.14)} & \textbf{0.2246 (0.0023)} & \textbf{2.79 (0.20)} \\ 
\midrule
\multirow{7}{*}{CIFAR-100}
 & {\color[HTML]{00B0F0} Teachers} & {\color[HTML]{00B0F0} 81.36} & {\color[HTML]{00B0F0} 0.7167}  & {\color[HTML]{00B0F0} 1.41}   \\
 & small-Ens                     & 79.29 (0.26)       & 1.0413 (0.0145)       & 12.90 (0.24)         \\
 & Hydra                          & 77.42 (0.15)       & 1.2912 (0.0272)       & 12.70 (0.37)         \\
 & LBE                            & 79.58 (0.40)       & 1.0110 (0.0087)        & 13.42 (0.42)        \\
 & $\text{Proxy-EnD}^2$               & 67.62 (0.22)       & 1.2355 (0.0151)       & \textbf{7.35 (0.25)}         \\
 & EDFM                           & 64.17 (0.32)       & 1.6741 (0.0242)       & 11.35 0.47)         \\
 & DLF                            & \textbf{79.68 (0.23)} & \textbf{0.8974 (0.0042)} & 9.45 (0.31)  \\
\bottomrule
\end{tabular}%
}
\end{table}
}

\subsection{Application to fine-tuning of language models}

In this section, we apply the proposed distillation framework to downstream binary classification tasks using pretrained language models. Given pre-trained teacher and student language models, 
fine-tuned teacher and student models are obtained using Low-Rank Adaptation (LoRA) \citep{hu2022lora}.
For the teacher and student pre-trained language models,  "\textit{RoBERTa}" \citep{liu2019roberta} and "\textit{DistilRoBERTa}" \cite{sanh2019distilbert} \footnote{\url{https://www.huggingface.co/distilroberta-base}} are used. As a teacher ensemble,
we obtain four fine-tuned models by combining LoRA and \textit{RoBERTa} with randomly selected initializations for each task. Then, an ensemble of fine-tuned student language models is constructed by applying LoRA to \textit{DistilRoBERTa} for each
distillation method.

\paragraph{Datasets}
We analyze three GLUE \cite{wang2018glue} and SuperGLUE \cite{wang2019superglue} sub-tasks: RTE, MRPC, and WiC. All three datasets are binary classification tasks.
Implementation details of the distillation methods are given in Appendix \ref{fine-tuning of language models}.

\paragraph{Results}
As shown in Table \ref{language model_result}, 
Gaussian distillation outperforms Hydra and LBE with large margins.
We conjecture that the performance gap of Gaussian distillation to Hydra and LBE
would become larger when the complexity gap between teacher and student models becomes larger.
This is a reasonable conjecture since smaller models could preserve less variations in teacher models.
Gaussian distillation would add additional variations to the student models 
through variations of the latent vector $\boldsymbol{Z}.$

\begin{table}[]
\centering
\caption{Results on GLUE and SuperGLUE benchmark datasets}
\label{language model_result}
\resizebox{0.65\textwidth}{!}{%
\begin{tabular}{c|l|ccc}
\toprule
dataset    & method                         & Acc (\%) $\uparrow$            & NLL $\downarrow$                  & ECE (\%) $\downarrow$            \\
\midrule
\multirow{5}{*}{RTE}
 & {\color[HTML]{00B0F0} Teachers} & {\color[HTML]{00B0F0} 75.09} & {\color[HTML]{00B0F0} 0.8401}  & {\color[HTML]{00B0F0} 17.70}   \\
 & small-Ens                     & 67.15 (0.0057)      & 0.6739 (0.0124)      & 13.09 (0.0148)      \\
 & Hydra                         & 62.82 (0.1650)       & 0.9034 (0.1733)      & 23.31 (0.4907)       \\
 & LBE                           & 65.97 (0.1406)       & 0.9235 (0.0664)      & 25.63 (0.1909)       \\
 & DLF                            & \textbf{67.06 (0.1040)} & \textbf{0.6658 (0.0762)} & \textbf{9.74 (0.5050)}  \\
\midrule
\multirow{5}{*}{MRPC}
 & {\color[HTML]{00B0F0} Teachers} & {\color[HTML]{00B0F0} 87.25} & {\color[HTML]{00B0F0} 0.3435}  & {\color[HTML]{00B0F0} 4.77}    \\
 & small-Ens                     & 83.092 (0.0172)      & 0.4596 (0.0396)      & \textbf{9.84 (0.0164)} \\
 & Hydra                         & 82.19 (0.0357)      & 0.6429 (0.1274)      & 11.79 (0.0136)      \\
 & LBE                           & 82.23 (0.0153)      & 0.6527 (0.0544)      & 13.6 (0.0114)       \\
 & DLF                            & \textbf{83.094 (0.0035)} & \textbf{0.4526 (0.0113)} & 10.72 (0.0392)      \\
\midrule
\multirow{5}{*}{WiC}
 & {\color[HTML]{00B0F0} Teachers} & {\color[HTML]{00B0F0} 68.03} & {\color[HTML]{00B0F0} 0.6395}  & {\color[HTML]{00B0F0} 9.14}    \\
 & small-Ens                     & 65.02 (0.0062)      & \textbf{0.7674 (0.0206)} & 16.69 (0.0116)      \\
 & Hydra                         & 65.05 (0.0210)       & 1.1628 (0.0355)      & 26.26 (0.0191)      \\
 & LBE                           & 65.36 (0.1209)       & 0.8809 (0.0950)       & 20.28 (0.0300)       \\
 & DLF                            & \textbf{66.18 (0.0146)} & 0.7706 (0.0543)      & \textbf{15.99 (0.0198)} \\
\bottomrule
\end{tabular}%
}
\end{table}

\subsection{Application to distribution shift problems}\label{Application to distribution shift problems}

We compare DLF with two baselines which fine-tune only the head on new data while the body is learned by
either (1) a standard DNN or (2) applying Hydra on training data of CIFAR-10.
For distribution-shifted new data, we swap the labels of CIFAR-10 as is done by \citep{lee2022surgical}.
See implementation details in Appendix \ref{Application to distribution shift}.
The results are given in Figure \ref{fig:map_results} which amply show that DLF is superior. 
It is interesting to see that DLF outperforms even when the sample size of the new data is large,
which implies that the learned body by DLF is qualitatively different from those by DNN and Hydra.
We do not know the reason but an implication is that Gaussian distillation
is good at learning not only quantifying uncertainty but also learning the feature vector (i.e. the body).

\begin{figure}[h]
    \centering
    \includegraphics[width=\textwidth]{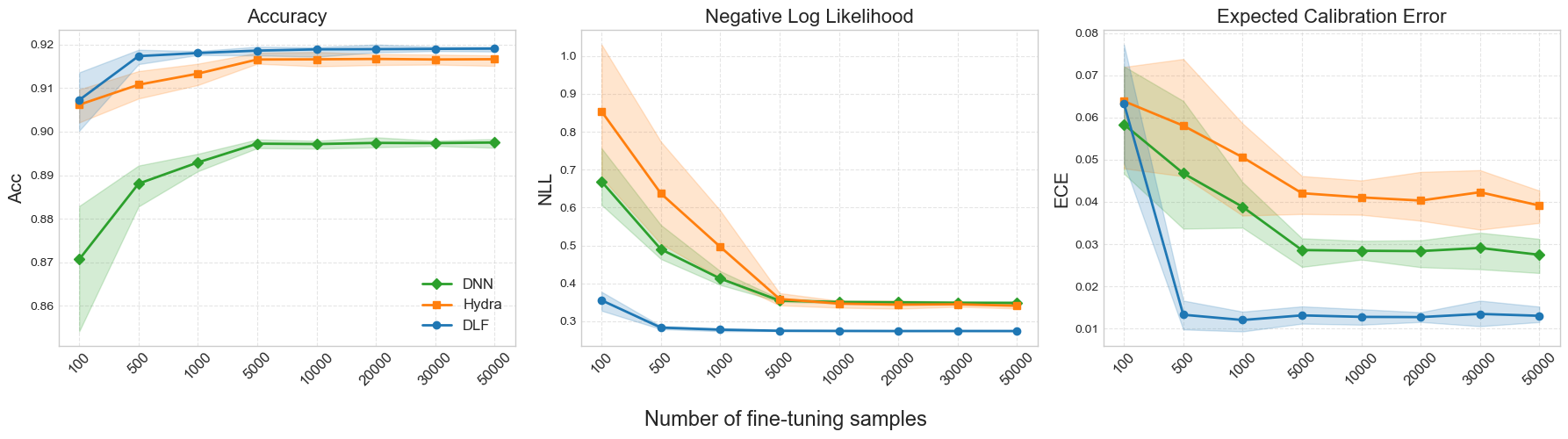}
    \caption{Comparison of performances of the three learning algorithms (DNN, Hydra and DLF) pretrained on CIFAR-10
    and fine-tuned on CIFAR10-Flip as the sample size of CIFAR 10-Flip data varies. 
    The solid curves are the means and the shaded bands are the min-max spreads obtained from 5 training models on 5 randomly selected new data of CIFAR10-Flip.}
    \label{fig:map_results}
\end{figure}

\subsection{Ablation studies}\label{Ablation studies main}
In Appendix \ref{ablation study}, we present the results of ablation studies including the sensitivity of Gaussian distillation
to the choice of design points, the dimension of latent factor, the architecture size of student DNNs and
the number of ensemble members used in a teacher and student ensemble.
In addition, we compare the results when the initial solution is randomly selected in Gaussian distillation.

\section{Conclusion} \label{section: conclusion}
We proposed a novel method for distilling deep ensembles, specifically addressing the challenges associated with computational costs, inference time, and storage capacities inherent in traditional deep ensemble approaches. The key innovation lies in modeling the covariance structure of deep ensembles through the DLF model, 
enabling efficient preservation of uncertainty in a teacher ensemble with significantly reduced inference costs. 

There are several future research topics. 
First, in this paper, we only focused on deep ensembles. 
It would be valuable to consider Bayesian DNNs, as they provide a framework for uncertainty quantification \citep{izmailov2021bayesian, sharma2023bayesian, kong2023masked} and can potentially serve as a prior for on-device posterior updates.
Second, for distillation of a fine-tuned language model, we used \textit{DistilRoBERTa} \cite{sanh2019distilbert}, a pretrained distilled language model. 
It would be promising to distill the pretrained language model and the model for LoRA simultaneously.
Third, the DLF could be used for online Bayesian learning by approximating the posterior with respect to old data by the DLF and using it for the prior of new data. We will pursue this idea in a near future.

\section*{Acknowledgements}
This work was partly supported by the National Research Foundation of Korea(NRF) grant funded by the Korea government(MSIT) (No. 2022R1A5A7083908), the National Research Foundation of Korea(NRF) grant funded by the Korea government(MSIT) (RS-2025-00556079), and by Institute of Information \& communications Technology Planning \& Evaluation (IITP) grant funded by the Korea government(MSIT) [NO.RS-2021-II211343, Artificial Intelligence Graduate School Program (Seoul National University)].



\bibliographystyle{unsrt}
\bibliography{bibliography}


\newpage
\section*{NeurIPS Paper Checklist}

\begin{enumerate}

\item {\bf Claims}
    \item[] Question: Do the main claims made in the abstract and introduction accurately reflect the paper's contributions and scope?
    \item[] Answer: \answerYes{} 
    \item[] Justification: In Introduction, we provide the method of this paper and its key contributions. Abstract also briefly mentions these points.
    \item[] Guidelines:
    \begin{itemize}
        \item The answer NA means that the abstract and introduction do not include the claims made in the paper.
        \item The abstract and/or introduction should clearly state the claims made, including the contributions made in the paper and important assumptions and limitations. A No or NA answer to this question will not be perceived well by the reviewers. 
        \item The claims made should match theoretical and experimental results, and reflect how much the results can be expected to generalize to other settings. 
        \item It is fine to include aspirational goals as motivation as long as it is clear that these goals are not attained by the paper. 
    \end{itemize}

\item {\bf Limitations}
    \item[] Question: Does the paper discuss the limitations of the work performed by the authors?
    \item[] Answer: \answerYes{} 
    \item[] Justification: We discuss a potential limitation of our work in Section \ref{section: conclusion}.
    \item[] Guidelines:
    \begin{itemize}
        \item The answer NA means that the paper has no limitation while the answer No means that the paper has limitations, but those are not discussed in the paper. 
        \item The authors are encouraged to create a separate "Limitations" section in their paper.
        \item The paper should point out any strong assumptions and how robust the results are to violations of these assumptions (e.g., independence assumptions, noiseless settings, model well-specification, asymptotic approximations only holding locally). The authors should reflect on how these assumptions might be violated in practice and what the implications would be.
        \item The authors should reflect on the scope of the claims made, e.g., if the approach was only tested on a few datasets or with a few runs. In general, empirical results often depend on implicit assumptions, which should be articulated.
        \item The authors should reflect on the factors that influence the performance of the approach. For example, a facial recognition algorithm may perform poorly when image resolution is low or images are taken in low lighting. Or a speech-to-text system might not be used reliably to provide closed captions for online lectures because it fails to handle technical jargon.
        \item The authors should discuss the computational efficiency of the proposed algorithms and how they scale with dataset size.
        \item If applicable, the authors should discuss possible limitations of their approach to address problems of privacy and fairness.
        \item While the authors might fear that complete honesty about limitations might be used by reviewers as grounds for rejection, a worse outcome might be that reviewers discover limitations that aren't acknowledged in the paper. The authors should use their best judgment and recognize that individual actions in favor of transparency play an important role in developing norms that preserve the integrity of the community. Reviewers will be specifically instructed to not penalize honesty concerning limitations.
    \end{itemize}

\item {\bf Theory Assumptions and Proofs}
    \item[] Question: For each theoretical result, does the paper provide the full set of assumptions and a complete (and correct) proof?
    \item[] Answer: \answerYes{} 
    \item[] Justification: In Appendix \ref{sec4}, we provide the full set of assumption and a complete proof of the main theorem in this paper. Also, for the theoretical claims in Section \ref{choice of design points}, we provide complete mathematical proofs in Appendix \ref{Theorical Results}.
    \item[] Guidelines:
    \begin{itemize}
        \item The answer NA means that the paper does not include theoretical results. 
        \item All the theorems, formulas, and proofs in the paper should be numbered and cross-referenced.
        \item All assumptions should be clearly stated or referenced in the statement of any theorems.
        \item The proofs can either appear in the main paper or the supplemental material, but if they appear in the supplemental material, the authors are encouraged to provide a short proof sketch to provide intuition. 
        \item Inversely, any informal proof provided in the core of the paper should be complemented by formal proofs provided in appendix or supplemental material.
        \item Theorems and Lemmas that the proof relies upon should be properly referenced. 
    \end{itemize}

    \item {\bf Experimental Result Reproducibility}
    \item[] Question: Does the paper fully disclose all the information needed to reproduce the main experimental results of the paper to the extent that it affects the main claims and/or conclusions of the paper (regardless of whether the code and data are provided or not)?
    \item[] Answer: \answerYes{} 
    \item[] Justification: In Appendix \ref{experimental details}, we provide details of our experimental setups.
    \item[] Guidelines: 
    \begin{itemize}
        \item The answer NA means that the paper does not include experiments.
        \item If the paper includes experiments, a No answer to this question will not be perceived well by the reviewers: Making the paper reproducible is important, regardless of whether the code and data are provided or not.
        \item If the contribution is a dataset and/or model, the authors should describe the steps taken to make their results reproducible or verifiable. 
        \item Depending on the contribution, reproducibility can be accomplished in various ways. For example, if the contribution is a novel architecture, describing the architecture fully might suffice, or if the contribution is a specific model and empirical evaluation, it may be necessary to either make it possible for others to replicate the model with the same dataset, or provide access to the model. In general. releasing code and data is often one good way to accomplish this, but reproducibility can also be provided via detailed instructions for how to replicate the results, access to a hosted model (e.g., in the case of a large language model), releasing of a model checkpoint, or other means that are appropriate to the research performed.
        \item While NeurIPS does not require releasing code, the conference does require all submissions to provide some reasonable avenue for reproducibility, which may depend on the nature of the contribution. For example
        \begin{enumerate}
            \item If the contribution is primarily a new algorithm, the paper should make it clear how to reproduce that algorithm.
            \item If the contribution is primarily a new model architecture, the paper should describe the architecture clearly and fully.
            \item If the contribution is a new model (e.g., a large language model), then there should either be a way to access this model for reproducing the results or a way to reproduce the model (e.g., with an open-source dataset or instructions for how to construct the dataset).
            \item We recognize that reproducibility may be tricky in some cases, in which case authors are welcome to describe the particular way they provide for reproducibility. In the case of closed-source models, it may be that access to the model is limited in some way (e.g., to registered users), but it should be possible for other researchers to have some path to reproducing or verifying the results.
        \end{enumerate}
    \end{itemize}

\item {\bf Open access to data and code}
    \item[] Question: Does the paper provide open access to the data and code, with sufficient instructions to faithfully reproduce the main experimental results, as described in supplemental material?
    \item[] Answer: \answerYes{} 
    \item[] Justification: We provide the source codes of our proposed algorithm in the supplementary material. Furthermore, we will publicly upload them after acceptance.
    \item[] Guidelines:
    \begin{itemize}
        \item The answer NA means that paper does not include experiments requiring code.
        \item Please see the NeurIPS code and data submission guidelines (\url{https://nips.cc/public/guides/CodeSubmissionPolicy}) for more details.
        \item While we encourage the release of code and data, we understand that this might not be possible, so “No” is an acceptable answer. Papers cannot be rejected simply for not including code, unless this is central to the contribution (e.g., for a new open-source benchmark).
        \item The instructions should contain the exact command and environment needed to run to reproduce the results. See the NeurIPS code and data submission guidelines (\url{https://nips.cc/public/guides/CodeSubmissionPolicy}) for more details.
        \item The authors should provide instructions on data access and preparation, including how to access the raw data, preprocessed data, intermediate data, and generated data, etc.
        \item The authors should provide scripts to reproduce all experimental results for the new proposed method and baselines. If only a subset of experiments are reproducible, they should state which ones are omitted from the script and why.
        \item At submission time, to preserve anonymity, the authors should release anonymized versions (if applicable).
        \item Providing as much information as possible in supplemental material (appended to the paper) is recommended, but including URLs to data and code is permitted.
    \end{itemize}

\item {\bf Experimental Setting/Details}
    \item[] Question: Does the paper specify all the training and test details (e.g., data splits, hyperparameters, how they were chosen, type of optimizer, etc.) necessary to understand the results?
    \item[] Answer: \answerYes{} 
    \item[] Justification: In Appendix \ref{experimental details}, we provide details of our experimental setups.
    \item[] Guidelines:
    \begin{itemize}
        \item The answer NA means that the paper does not include experiments.
        \item The experimental setting should be presented in the core of the paper to a level of detail that is necessary to appreciate the results and make sense of them.
        \item The full details can be provided either with the code, in appendix, or as supplemental material.
    \end{itemize}

\item {\bf Experiment Statistical Significance}
    \item[] Question: Does the paper report error bars suitably and correctly defined or other appropriate information about the statistical significance of the experiments?
    \item[] Answer: \answerYes{} 
    \item[] Justification: In Section \ref{experiments}, we report the mean and standard deviation from multiple runs. And in Section \ref{Application to distribution shift problems} and Appendix \ref{ablation study}, we also plot the max–min area in figures.
    \item[] Guidelines:
    \begin{itemize}
        \item The answer NA means that the paper does not include experiments.
        \item The authors should answer "Yes" if the results are accompanied by error bars, confidence intervals, or statistical significance tests, at least for the experiments that support the main claims of the paper.
        \item The factors of variability that the error bars are capturing should be clearly stated (for example, train/test split, initialization, random drawing of some parameter, or overall run with given experimental conditions).
        \item The method for calculating the error bars should be explained (closed form formula, call to a library function, bootstrap, etc.)
        \item The assumptions made should be given (e.g., Normally distributed errors).
        \item It should be clear whether the error bar is the standard deviation or the standard error of the mean.
        \item It is OK to report 1-sigma error bars, but one should state it. The authors should preferably report a 2-sigma error bar than state that they have a 96\% CI, if the hypothesis of Normality of errors is not verified.
        \item For asymmetric distributions, the authors should be careful not to show in tables or figures symmetric error bars that would yield results that are out of range (e.g. negative error rates).
        \item If error bars are reported in tables or plots, The authors should explain in the text how they were calculated and reference the corresponding figures or tables in the text.
    \end{itemize}

\item {\bf Experiments Compute Resources}
    \item[] Question: For each experiment, does the paper provide sufficient information on the computer resources (type of compute workers, memory, time of execution) needed to reproduce the experiments?
    \item[] Answer: \answerYes{} 
    \item[] Justification: In Appendix \ref{experimental details}, we report the number of parameters for each regression and classification case. And we provide the used computing resources in Appendix \ref{appendix: resource}
    \item[] Guidelines:
    \begin{itemize}
        \item The answer NA means that the paper does not include experiments.
        \item The paper should indicate the type of compute workers CPU or GPU, internal cluster, or cloud provider, including relevant memory and storage.
        \item The paper should provide the amount of compute required for each of the individual experimental runs as well as estimate the total compute. 
        \item The paper should disclose whether the full research project required more compute than the experiments reported in the paper (e.g., preliminary or failed experiments that didn't make it into the paper). 
    \end{itemize}
    
\item {\bf Code Of Ethics}
    \item[] Question: Does the research conducted in the paper conform, in every respect, with the NeurIPS Code of Ethics \url{https://neurips.cc/public/EthicsGuidelines}?
    \item[] Answer: \answerYes{} 
    \item[] Justification: We thoroughly and carefully checked the Code of Ethics.
    \item[] Guidelines:
    \begin{itemize}
        \item The answer NA means that the authors have not reviewed the NeurIPS Code of Ethics.
        \item If the authors answer No, they should explain the special circumstances that require a deviation from the Code of Ethics.
        \item The authors should make sure to preserve anonymity (e.g., if there is a special consideration due to laws or regulations in their jurisdiction).
    \end{itemize}

\item {\bf Broader Impacts}
    \item[] Question: Does the paper discuss both potential positive societal impacts and negative societal impacts of the work performed?
    \item[] Answer: \answerYes{} 
    \item[] Justification: We give several societal broader impacts of our work in Section \ref{section: conclusion}.
    \item[] Guidelines:
    \begin{itemize}
        \item The answer NA means that there is no societal impact of the work performed.
        \item If the authors answer NA or No, they should explain why their work has no societal impact or why the paper does not address societal impact.
        \item Examples of negative societal impacts include potential malicious or unintended uses (e.g., disinformation, generating fake profiles, surveillance), fairness considerations (e.g., deployment of technologies that could make decisions that unfairly impact specific groups), privacy considerations, and security considerations.
        \item The conference expects that many papers will be foundational research and not tied to particular applications, let alone deployments. However, if there is a direct path to any negative applications, the authors should point it out. For example, it is legitimate to point out that an improvement in the quality of generative models could be used to generate deepfakes for disinformation. On the other hand, it is not needed to point out that a generic algorithm for optimizing neural networks could enable people to train models that generate Deepfakes faster.
        \item The authors should consider possible harms that could arise when the technology is being used as intended and functioning correctly, harms that could arise when the technology is being used as intended but gives incorrect results, and harms following from (intentional or unintentional) misuse of the technology.
        \item If there are negative societal impacts, the authors could also discuss possible mitigation strategies (e.g., gated release of models, providing defenses in addition to attacks, mechanisms for monitoring misuse, mechanisms to monitor how a system learns from feedback over time, improving the efficiency and accessibility of ML).
    \end{itemize}
    
\item {\bf Safeguards}
    \item[] Question: Does the paper describe safeguards that have been put in place for responsible release of data or models that have a high risk for misuse (e.g., pretrained language models, image generators, or scraped datasets)?
    \item[] Answer: \answerYes{} 
    \item[] Justification: We give several societal broader impacts of our work in Section \ref{section: conclusion}.
    \item[] Guidelines:
    \begin{itemize}
        \item The answer NA means that the paper poses no such risks.
        \item Released models that have a high risk for misuse or dual-use should be released with necessary safeguards to allow for controlled use of the model, for example by requiring that users adhere to usage guidelines or restrictions to access the model or implementing safety filters. 
        \item Datasets that have been scraped from the Internet could pose safety risks. The authors should describe how they avoided releasing unsafe images.
        \item We recognize that providing effective safeguards is challenging, and many papers do not require this, but we encourage authors to take this into account and make a best faith effort.
    \end{itemize}

\item {\bf Licenses for existing assets}
    \item[] Question: Are the creators or original owners of assets (e.g., code, data, models), used in the paper, properly credited and are the license and terms of use explicitly mentioned and properly respected?
    \item[] Answer: \answerYes{} 
    \item[] Justification: We mention all the owners/references/urls of methods/codes/datasets used in Section \ref{experiments} and Appendix \ref{experimental details}.
    \item[] Guidelines:
    \begin{itemize}
        \item The answer NA means that the paper does not use existing assets.
        \item The authors should cite the original paper that produced the code package or dataset.
        \item The authors should state which version of the asset is used and, if possible, include a URL.
        \item The name of the license (e.g., CC-BY 4.0) should be included for each asset.
        \item For scraped data from a particular source (e.g., website), the copyright and terms of service of that source should be provided.
        \item If assets are released, the license, copyright information, and terms of use in the package should be provided. For popular datasets, \url{paperswithcode.com/datasets} has curated licenses for some datasets. Their licensing guide can help determine the license of a dataset.
        \item For existing datasets that are re-packaged, both the original license and the license of the derived asset (if it has changed) should be provided.
        \item If this information is not available online, the authors are encouraged to reach out to the asset's creators.
    \end{itemize}

\item {\bf New Assets}
    \item[] Question: Are new assets introduced in the paper well documented and is the documentation provided alongside the assets?
    \item[] Answer: \answerYes{} 
    \item[] Justification: We provide the source codes of our proposed algorithm in the supplementary material. Furthermore, we will publicly upload them after acceptance.
    \item[] Guidelines:
    \begin{itemize}
        \item The answer NA means that the paper does not release new assets.
        \item Researchers should communicate the details of the dataset/code/model as part of their submissions via structured templates. This includes details about training, license, limitations, etc. 
        \item The paper should discuss whether and how consent was obtained from people whose asset is used.
        \item At submission time, remember to anonymize your assets (if applicable). You can either create an anonymized URL or include an anonymized zip file.
    \end{itemize}

\item {\bf Crowdsourcing and Research with Human Subjects}
    \item[] Question: For crowdsourcing experiments and research with human subjects, does the paper include the full text of instructions given to participants and screenshots, if applicable, as well as details about compensation (if any)? 
    \item[] Answer: \answerNA{} 
    \item[] Justification: This study does not involve crowdsourcing nor research with human subjects.
    \item[] Guidelines:
    \begin{itemize}
        \item The answer NA means that the paper does not involve crowdsourcing nor research with human subjects.
        \item Including this information in the supplemental material is fine, but if the main contribution of the paper involves human subjects, then as much detail as possible should be included in the main paper. 
        \item According to the NeurIPS Code of Ethics, workers involved in data collection, curation, or other labor should be paid at least the minimum wage in the country of the data collector. 
    \end{itemize}

\item {\bf Institutional Review Board (IRB) Approvals or Equivalent for Research with Human Subjects}
    \item[] Question: Does the paper describe potential risks incurred by study participants, whether such risks were disclosed to the subjects, and whether Institutional Review Board (IRB) approvals (or an equivalent approval/review based on the requirements of your country or institution) were obtained?
    \item[] Answer: \answerNA{} 
    \item[] Justification: This study does not involve crowdsourcing nor research with human subjects.
    \item[] Guidelines: 
    \begin{itemize}
        \item The answer NA means that the paper does not involve crowdsourcing nor research with human subjects.
        \item Depending on the country in which research is conducted, IRB approval (or equivalent) may be required for any human subjects research. If you obtained IRB approval, you should clearly state this in the paper. 
        \item We recognize that the procedures for this may vary significantly between institutions and locations, and we expect authors to adhere to the NeurIPS Code of Ethics and the guidelines for their institution. 
        \item For initial submissions, do not include any information that would break anonymity (if applicable), such as the institution conducting the review.
    \end{itemize}

\end{enumerate}


\newpage
\appendix
\begin{center}
    \textsc{\LARGE Appendix}
\end{center}

\section{Review of Ensemble Distillation}

\subsection{One-to-one Distillation} \label{appendix:one-to-one distillation}

For given teacher members \(p^{(t)}_i, i=1,\dots,n\), student \(p^{(s)}_i, i=1,\dots,n\) are trained to minimize
\begin{equation*}
   \sum_{i=1}^n \mathbb{E}_{\boldsymbol{x}} \left[  \operatorname{KL} \left(p_{i}^{(t)}(y|\boldsymbol{x}) || p_{i}^{(s)}(y|\boldsymbol{x}) \right) \right]
\end{equation*}
where $\mathbb{E}_{\boldsymbol{x}}$ is the expectation operator of a certain probability distribution
on the input space $\mathcal{X}.$ 
The following two methods use specially designed DNNs for student models which share the weights
between student models.

\subsubsection{Hydra}
Hydra \citep{tran2020hydra} employs a multi-head architecture in which a single shared body network extracts common features and a set of $n$ distinct linear heads generates predictions for each ensemble member.
Formally, the student DNN is parameterized by 
\begin{equation*}
    \theta_{\mathrm{hydra}} = \{\theta_{\mathrm{body}},\,\theta_{\mathrm{head},1},\dots,\theta_{\mathrm{head},n}\},
\end{equation*}
where $\theta_{\mathrm{body}}$ is used across all heads.
At inference, the body is evaluated once, and each head $h_{\theta_{\mathrm{head}, i}}(\cdot)$ produces a member-specific output $p^{(s)}_i(y\mid \boldsymbol x) = h_{\theta_{\mathrm{head}, i}}\bigl(b_{\theta_{\mathrm{body}}}(\boldsymbol x)\bigr)$. 
This design captures ensemble diversity while reducing both computation and memory compared to maintaining $n$ independent networks.

\subsubsection{Batch Ensemble and Latent Batch Ensemble}
\citep{wen2020batchensemble} introduces Batch Ensemble (BE) to reduce the memory and computational burden of deep ensembles.
In this architecture, all student networks share a core weight matrix $\boldsymbol W^l$ at each layer, while individual ensemble members are differentiated by rank-one perturbations.
Specifically, for the $i$-th student at layer $l$,
\begin{equation*}
    \boldsymbol W^l_i = \boldsymbol W^l \circ \bigl(\boldsymbol r^l_i\,{\boldsymbol s^l_i}^\top\bigr),
\end{equation*}
where $\boldsymbol r^l_i\in\mathbb R^{d_{\mathrm{out}}}$ and $\boldsymbol s^l_i\in\mathbb R^{d_{\mathrm{in}}}$ modulate the rows and columns of $\boldsymbol W^l$, respectively. 
This construction enables each sub-network to maintain member-specific behaviors while reusing the majority of parameters, yielding significant savings in both storage and inference costs compared to training $n$ independent models. 

Building on BE, \citep{nam2022improving} proposes Latent Batch Ensemble (LBE), which further compresses the ensemble at inference time. 
Instead of maintaining $n$ distinct perturbations, LBE computes the average rank-one mask across all students:
\begin{equation*}
	\boldsymbol{W}^{l}_{s} = \boldsymbol{W}^{l}\circ \left(\frac{1}{n}\sum_{i=1}^n\boldsymbol{r}^{l}_{i} {\boldsymbol{s}^{l}_{i}}^{\top}\right).
\end{equation*} 
The resulting single student network requires only one forward pass per input while capturing the ensemble’s mean perturbation in weight space. 
Empirical results demonstrate that LBE matches or exceeds the calibration performance of standard BE, with inference cost reduced by a factor of $n$.
However, the proposed Latent Batch Ensemble is a specialized method designed for ensemble distillation for classification problems.

\subsection{Distribution Distillation} \label{appendix:distribution distillation}
Distribution distillation frames the output of an ensemble at the input $\boldsymbol{x}$ as samples from an input‐dependent Dirichlet distribution \citep{malinin2019ensemble, ryabinin2021scaling}. Let $\boldsymbol{f}_i, i=1,\ldots,n$ be given teacher models for a classification
task. For a given $\boldsymbol{x},$ let $\boldsymbol{\pi}_i(\boldsymbol{x})=\mathrm{softmax}\left( \boldsymbol{f}_i(\boldsymbol{x})\right),$  and we assume
that $\boldsymbol{\pi}_1(\boldsymbol{x}),\ldots,\boldsymbol{\pi}_n(\boldsymbol{x})$ independently follow
the Dirichlet distribution with parameter $\boldsymbol\alpha(\boldsymbol x)
=(\alpha_1(\boldsymbol x),\ldots, \alpha_c(\boldsymbol x)).$ Then, we model
$\boldsymbol\alpha(\boldsymbol x)$ by a student DNN $\boldsymbol{\Psi}(\boldsymbol x |\theta)$
and estimate $\theta$ by maximizing
\begin{equation*}
    \mathcal{L}(\theta, \mathcal{D}^{\mathrm{design}})
    = \sum_{j=1}^m \sum_{i=1}^n \bigl[\ln p(\boldsymbol \pi_i(\boldsymbol{x}_j^{(d)})| {\boldsymbol\Psi(
    \boldsymbol{x}_j^{(d)} |\theta)})\bigr],
\end{equation*}
where $p(\boldsymbol{\pi}|\boldsymbol{\alpha})$ is the density of the Dirichlet distribution with parameter
$\boldsymbol{\alpha}.$ See \citep{malinin2019ensemble, ryabinin2021scaling} for details.

\subsubsection{Proxy-Dirichlet Distillation}

\citep{ryabinin2021scaling} proposed Proxy–Dirichlet Distillation ($\text{Proxy-EnD}^2$) to mitigate the convergence difficulties of standard Dirichlet distillation \citep{malinin2019ensemble} when the number of classes is very high. The method constructs a proxy Dirichlet distribution from the teacher models and trains the student DNN $\boldsymbol{\Psi}(\boldsymbol x |\theta)$ by minimizing the reverse KL divergence between $p(\boldsymbol{\pi}|\boldsymbol{\Psi}(\boldsymbol x |\theta))$ and the proxy distribution.

The proxy distribution is obtained from teacher models as
\begin{equation*}
\hat{\beta}_l(\boldsymbol {x}) = \frac{1}{n} \sum_{i=1}^{n} \pi_{il}(\boldsymbol{x}),
\end{equation*}
\begin{equation*}
\hat{\alpha}_{l}(\boldsymbol {x})
=\frac{\hat{\beta}_{l}(\boldsymbol {x})\,(c-1)}
{2 \sum_{l=1}^{c} \hat{\beta}_{l}(\boldsymbol {x})\!\left( \ln \hat{\beta}_{l}(\boldsymbol {x}) - \frac{1}{n} \sum_{i=1}^{n} \ln \pi_{il}(\boldsymbol {x}) \right)} , \quad l=1,2,..\cdots,c.
\end{equation*}
Then, estimate $\theta$ by minimizing
\begin{equation*}
\sum_{j=1}^{m}
\operatorname{KL}\!\left(
  p\!\left(\boldsymbol{\pi}\mid
    \exp\bigl(\boldsymbol{\Psi}(\boldsymbol{x}^{(d)}_{j}\mid\theta)\bigr)
  \right)
  \;\middle\|\;
  p\!\left(\boldsymbol{\pi}\mid \hat{\boldsymbol{\alpha}}(\boldsymbol{x}^{(d)}_{j})+1\right)
\right),
\end{equation*}
where $\hat{\boldsymbol{\alpha}}(\cdot) = \left[\hat{\alpha}_1(\cdot), \dots, \hat{\alpha}_c(\cdot) \right].$

\subsubsection{Ensemble Distribution Distillation via Flow Matching}

Ensemble Distribution Distillation via Flow Matching (EDFM) \citep{parkensemble} models a conditional distribution of logits $\boldsymbol{h}_{1} $ given $\boldsymbol{x}$ over $\mathbb{R}^{c}$. The method constructs the target distribution $p_1(\boldsymbol{h} | \boldsymbol{x})$ as an empirical distribution using $\boldsymbol{f}_i(\boldsymbol{x}).$


They model using Flow matching such that
\begin{equation*}
    d \boldsymbol{h}_{t}(\boldsymbol{x}) = \boldsymbol{\Psi}_{\theta}(t, \boldsymbol{h}_{t}(\boldsymbol{x}), \boldsymbol{x}) dt ,\,\, t \in [0,1] ,\,\, \boldsymbol{h}_{0}(\boldsymbol{x}) \sim \mathcal{N}_{c}(0, \sigma^{2} \mathbb{I}_{c}) \text{ and } \boldsymbol{h} | \boldsymbol{x} \overset{d}{\equiv} \boldsymbol{h}_{1}(\boldsymbol{x})
\end{equation*}
where $\boldsymbol{\Psi}_{\theta}$ is student DNN. \\
Then, the network is trained by minimizing the conditional flow matching loss

\begin{equation*}
    \sum_{j=1}^{m} \mathbb{E}_{\boldsymbol{h}_{0} , \boldsymbol{h}_{1}(\boldsymbol{x}^{(d)}_{j}) , t} \left[ \lambda(t) \|  \boldsymbol{\Psi}_{\theta}(t, \boldsymbol{h}_{t}(\boldsymbol{x}^{(d)}_{j}), \boldsymbol{x}^{(d)}_{j}) - (\boldsymbol{h}_{1}(\boldsymbol{x}^{(d)}_{j}) - \boldsymbol{h}_{0})\|^{2}  \right]
\end{equation*}
where $\boldsymbol{h}_{0} \sim \mathcal{N}_{c}(0, \sigma^{2} \mathbb{I}_{c}), \boldsymbol{h}_{1}(\boldsymbol{x}^{(d)}_{j}) \sim p_{1}(\cdot | \boldsymbol{x}^{(d)}_{j}), t \sim U[0,1], \boldsymbol{h}_{t} (\boldsymbol{x}^{(d)}_{j}) =  t \boldsymbol{h}_{1}(\boldsymbol{x}^{(d)}_{j}) + (1-t) \boldsymbol{h}_{0}$ and $\lambda$ denotes a time dependent weighting function.\\
Sampling proceeds by numerically integrating the learned flow $\boldsymbol{\Psi}_{\theta}(t, \boldsymbol{h}_{t}(\boldsymbol{x}), \boldsymbol{x})$ with a standard ordinary differential equation (ODE) solver. 
For more detailed information, see \citep{parkensemble}.

\section{Details of Gaussian distillation}\label{Appendix:Estimation of the mean and factor loading}

\subsection{EM algorithm for the univariate DLF model}\label{Appendix:EM algorithm for the univariate DLF model}

The main idea of Gaussian distillation is to estimate
the mean and factor loading functions based on given teacher models assuming
that teacher models are independent realizations of the DLF and 
estimate the parameter in the student DNNs  by maximizing the corresponding log-likelihood.
For optimization, we use the EM algorithm \citep{rubin1982algorithms}.

Suppose that $n$ many teacher models $f_1(\cdot),\ldots, f_n(\cdot)$ are given.
Gaussian distillation consists of three steps.
The first step is to choose $m$-many design points $\mathcal{D}^{\mathrm{design}} = \{\boldsymbol{x}_{1}^{(d)}, \ldots, \boldsymbol{x}_{m}^{(d)}\}$.
The second step is to calculate the vectors of prediction values of each teacher model at the design points to have $\boldsymbol{f}_i=\big(f_i(\boldsymbol{x}_j^{(d)}), j=1,\ldots,m\big)^\top$ for $i=1,\ldots,n.$
The final step is to estimate the parameter $\theta$ 
in the DLF by the MLE assuming that
$f_1(\cdot),\ldots, f_n(\cdot)$ are independent realizations of a random function following the DLF model.
Since $\boldsymbol{f}_i$s are independent Gaussian random vectors, the MLE can be obtained by use of the EM algorithm as follows.

To make the EM algorithm numerically stable, we consider the noisy DLF model which assumes that $\boldsymbol{f}_i=\boldsymbol{\tilde{f}}_i+\boldsymbol{v}_i,$ where $\boldsymbol{v}_i \sim \mathcal{N}_m(\mathbf{0}, \sigma_f^2 \mathbb{I}_m)$ and $\boldsymbol{\tilde{f}}_i=(\tilde{f}_i(\boldsymbol{x}_1),\ldots,\tilde{f}_i(\boldsymbol{x}_m))^\top$ with $\tilde{f}_i(\cdot)$s following the DLF model. 
Specifically, each $\tilde{f}_i$ is expressed as $\tilde{f}_i(\cdot) = \mu_{\theta}(\cdot) + \Phi_{\theta}(\cdot)^\top \boldsymbol{z}_i,$ where $\boldsymbol{z}_i \sim \mathcal{N}_q(\mathbf{0}, \mathbb{I}_q)$ denotes the latent factor corresponding to the $i$-th function realization.
Then, we obtain the MLE of the parameter $\theta$ in the mean and factor loading functions as well as $\sigma_f^2.$
We abuse the notation to write $\theta=(\theta,\sigma_f^2)$ unless there is any confusion.

The complete log-likelihood is given as
\begin{equation*}
\begin{split}    
    \ell^{com}(\theta | \boldsymbol{f}_{1:n}, \boldsymbol{z}_{1:n}) 
    =&-\frac{nm}{2} \log(2 \pi \sigma_{f}^{2}) - \frac{nq}{2} \log(2 \pi)  - \frac{ \sum_{i=1}^{n}\boldsymbol{z}_{i}^{\top} \boldsymbol{z}_{i}}{2} \\
    & - \frac{ \sum_{i=1}^{n} (\boldsymbol{f}_i - \boldsymbol{\mu}_{\theta} - \boldsymbol{\Phi}_{\theta} \boldsymbol{z}_{i}    )^\top (\boldsymbol{f}_i - \boldsymbol{\mu}_{\theta} - \boldsymbol{\Phi}_{\theta} \boldsymbol{z}_{i}  ) }{2 \sigma_{f}^{2}},
\end{split}
\end{equation*}
where $\boldsymbol{f}_{1:n}=\{\boldsymbol{f}_1,\ldots,\boldsymbol{f}_n\}, 
\boldsymbol{z}_{1:n}=\{\boldsymbol{z}_1,\ldots,\boldsymbol{z}_n\}, \boldsymbol{\mu}_{\theta}=
(\mu_{\theta}(\boldsymbol{x}_1^{(d)}),\ldots, \mu_{\theta}(\boldsymbol{x}_m^{(d)}))^\top$
and $\boldsymbol{\Phi}_{\theta} = (\Phi_{\theta}(\boldsymbol{x}_1^{(d)}),\ldots,\Phi_{\theta}(\boldsymbol{x}^{(d)}_m))^\top$ is an $m \times q$ matrix.

Thus, for a given parameter $\theta^{(t-1)}$  at time $t-1,$
the E-step is to calculate the conditional expectation of the complete log-likelihood which is given as
\begin{equation*}
\begin{split}
    Q(\theta|\theta^{(t-1)})=&\mathbb{E}_{\boldsymbol{z}_{1:n} \mid \boldsymbol{f}_{1:n}, \theta^{(t-1)}}[\ell^{com}(\theta | \boldsymbol{f}_{1:n}, \boldsymbol{z}_{1:n})] \\
    = &-\frac{nm}{2} \log(2 \pi {\sigma_{f}^{2}}) - \frac{nq}{2} \log(2 \pi) - \frac{ \sum_{i=1}^{n} \operatorname{tr} \mathbb{E}[\boldsymbol{z}_{i}\boldsymbol{z}_{i}^{\top}|\theta^{(t-1)},\boldsymbol{f}_{1:n}]}{2} \\
    & - \frac{1}{2{\sigma_{f}^{2}}}\sum_{i=1}^{n} \bigg\{\ (\boldsymbol{f}_i - \boldsymbol{\mu}_{\theta})^\top (\boldsymbol{f}_i - \boldsymbol{\mu}_{\theta}) \\
    & +2(\boldsymbol{f}_i - \boldsymbol{\mu}_{\theta})^\top {\boldsymbol{\Phi}_{\theta}} \mathbb{E}[\boldsymbol{z}_{i}|\theta^{(t-1)},\boldsymbol{f}_{1:n}] + \operatorname{tr} \left(\boldsymbol{\Phi}_{\theta}^\top \boldsymbol{\Phi}_{\theta} \mathbb{E}[\boldsymbol{z}_{i}\boldsymbol{z}_{i}^{\top}|\theta^{(t-1)},\boldsymbol{f}_{1:n}] \right) \bigg\}, 
\end{split}
\end{equation*}
where
\begin{equation*}
	\begin{split}
		\mathbb{E}[\boldsymbol{z}_{i}|\theta^{(t-1)},\boldsymbol{f}_{1:n}] & = \mathbb{V}[\boldsymbol{z}|\theta^{(t-1)}] \boldsymbol{\Phi}_{\theta^{(t-1)}}^{\top}  (\boldsymbol{f}_{i} - \boldsymbol{\mu}_{\theta^{(t-1)}}) /{\sigma_{f}^{2}}^{(t-1)} \\ 
		\mathbb{E}[\boldsymbol{z}_{i}\boldsymbol{z}_{i}^{\top}|\theta^{(t-1)},\boldsymbol{f}_{1:n}] & = \mathbb{V}[\boldsymbol{z}|\theta^{(t-1)}] + \mathbb{E}[\boldsymbol{z}_{i}|\theta^{(t-1)},\boldsymbol{f}_{1:n}]\mathbb{E}[\boldsymbol{z}_{i}|\theta^{(t-1)},\boldsymbol{f}_{1:n}]^{\top} \\    
	\end{split}    
\end{equation*}
with $\mathbb{V}[\boldsymbol{z}|\theta^{(t-1)}] = \left(\mathbb{I}_{q}+{\boldsymbol{\Phi}_{\theta^{(t-1)}}}^{\top}\boldsymbol{\Phi}_{\theta^{(t-1)}}/{{\sigma_{f}^{2}}^{(t-1)}}\right)^{-1}.$ 

In the M-step, we usually update $\theta^{(t)}$ by maximizing $Q(\theta|\theta^{(t-1)})$.
Instead of maximizing $Q(\theta|\theta^{(t-1)})$, we update $\theta$ by using a stochastic gradient descent algorithm (i.e.,
a gradient descent algorithm on a given mini-batches). The EM algorithm is summarized in Algorithm \ref{alg:1}.

\begin{algorithm}[H]
    \SetAlgoLined 
    \KwIn{Design points $\mathcal{D}^{\mathrm{design}} = \{\boldsymbol{x}_1^{(d)}, \dots, \boldsymbol{x}_m^{(d)}\}$, \\Teacher ensemble members $f_1(\cdot),\ldots, f_n(\cdot)$, learning rate $\eta > 0$}
    \KwOut{Estimated parameter $\theta$}
    
    Initialize parameter: $\theta^{(0)}$ \\ 
    \For {$t=1,2 \dots T$}{
        Shuffle the dataset $\mathcal{D}^{\mathrm{design}}$ and divide into mini-batches $\{\mathcal{B}_{1}, \dots, \mathcal{B}_{M}\}$ \\
        \For{$l=1,2,\dots, M$}{
            Calculate prediction values of teacher model $\boldsymbol{f}_{1:n,\mathcal{B}_l},$ 
            where $\boldsymbol{f}_{1:n,\mathcal{B}_l}=(\boldsymbol{f}_{1:n}(\boldsymbol{x}_j^{\mathrm{design}}), \boldsymbol{x}^{\mathrm{design}}_j\in \mathcal{B}_l).$\\
            Calculate $\mathbb{E}[\boldsymbol{z}|\theta^{(t-1)}]$ and $\mathbb{E}[\boldsymbol{z}\boldsymbol{z}^{\top}|\theta^{(t-1)}]$ using $\boldsymbol{f}_{1:n,\mathcal{B}_l}$ \\
            $Q(\theta | \theta^{(t-1)}) := \mathrm{E}_{\boldsymbol{z}_{1:n} \mid \boldsymbol{f}_{1:n,\mathcal{B}_l}, \theta^{(t-1)}}[ \ell^{com}(\theta ; \boldsymbol{f}_{1:n,\mathcal{B}_l}, \boldsymbol{z}_{1:n})]$ \\
            Calculate gradient $g^{(t)} := \nabla_{\theta} \left( - Q(\theta | \theta^{(t-1)}) \right)$ \\
            Update $\theta^{(t)} \leftarrow \theta^{(t-1)} - \eta \cdot g^{(t)}$
        }
    }
    \caption{{EM algorithm for the univariate DLF model}}
    \label{alg:1}
\end{algorithm}

\subsection{Multivariate DLF model}
In this section, we provide details of the multivariate DLF model introduced in Section \ref{Deep Latent Factor model}.
With given the design points $\{\boldsymbol{x}_1^{(d)}, \ldots, \boldsymbol{x}_m^{(d)}\}$, let
$\boldsymbol{f} = \bigl(\boldsymbol f(\boldsymbol{x}_j^{(d)}), j=1,\ldots,m\bigr)^\top$
and
$\boldsymbol{\mu}_\theta = \bigl(\mu_\theta(\boldsymbol{x}_j^{(d)}), j=1,\ldots,m\bigr)^{\top}$
be $m\times c$ matrices, and let
$\boldsymbol{\Phi}_\theta = \bigl(\Phi_\theta(\boldsymbol{x}_j^{(d)}), j=1,\ldots,m\bigr)^\top$
be an $m\times q$ matrix. Then, we assume that $\boldsymbol{f}$ follows the matrix-variate Gaussian distribution
\begin{equation*}
    \boldsymbol{f} \sim \mathcal{MN}_{m,c}(\boldsymbol{\mu}_\theta, \boldsymbol{\Phi}_\theta\boldsymbol{\Phi}_\theta^{\top},  LL^{\top}).
\end{equation*}
Using the properties of the matrix-variate Gaussian distribution, we can vectorize $\boldsymbol{f}$, thereby allowing the multivariate DLF model to be handled the same as the univariate case.

\subsubsection{Vectorization}
According to Definition 4 in \citep{chen2023multivariate}, the vectorization of $\boldsymbol{f}$ can be expressed as follows:
\begin{equation*}
	\operatorname{vec}(\boldsymbol{f}) \sim \mathcal{N}_{mc}(\operatorname{vec}(\mu_{\theta}), LL^{\top} \otimes \boldsymbol{\Phi}_\theta\boldsymbol{\Phi}_\theta^{\top})    
\end{equation*} where $\otimes$ denotes the Kronecker product.
This association arises from a factorization of the covariance matrix of the multivariate Gaussian distribution into the Kronecker product of matrices $\boldsymbol{\Phi}_{\theta}^{\top} \boldsymbol{\Phi}_{\theta}$ and $LL^{\top}$.

\subsubsection{EM algorithm for the multivariate DLF model}\label{EM algorithm when Multivariate}

We explain the EM algorithm for training the multivariate DLF model in the same way as in Appendix \ref{Appendix:EM algorithm for the univariate DLF model}. In the multivariate case, matrix vectorization and its associated properties serve as the central tools, as detailed in Section 10.2.2 of \citep{petersen2008matrix}. 

Suppose that $n$ teacher ensemble members $\boldsymbol f_1(\cdot),\ldots, \boldsymbol f_n(\cdot)$ are given, where each $\boldsymbol f_i(\cdot)\colon\mathcal{X}\to\mathbb{R}^c$ is a multivariate function. As in the univariate case, KD via the multivariate DLF involves three steps. First, we choose $m$ design points $\mathcal{D}^{\mathrm{design}} = \{\boldsymbol{x}_{1}^{(d)}, \ldots, \boldsymbol{x}_{m}^{(d)}\}$. The second step is to calculate the prediction values of $n$ many ensemble members
at the design points to have $\boldsymbol{f}_i=(\boldsymbol f_i(\boldsymbol{x}_j^{(d)}), j=1,\ldots,m)^\top$ for $i=1,\ldots,n.$ Here, unlike in the univariate case, each $\boldsymbol{f}_i$ should be regarded as an $m \times c$ matrix. Similarly to the univariate case, the final step employs the EM algorithm to obtain the MLE.

Similarly to the univariate case, we assume that $\boldsymbol{f}_i=\boldsymbol{\tilde{f}}_i+\boldsymbol{v}_i,$ where
$\boldsymbol{v}_i \sim \mathcal{MN}_{m,c}(\mathbf{0}, \sigma_f^2 \mathbb{I}_m, \mathbb{I}_c)$ and
$\boldsymbol{\tilde{f}}_i=(\boldsymbol{\tilde f}_i(\boldsymbol{x}_1),\ldots,\boldsymbol{\tilde f}_i(\boldsymbol{x}_m))^\top$
with $\boldsymbol{\tilde f}_i(\cdot)$s following the multivariate DLF model. From the property of the matrix-variate Gaussian distribution, we can rewrite the following $\operatorname{vec}(\boldsymbol{f}_i)=\operatorname{vec}(\boldsymbol{\tilde{f}}_i)+\operatorname{vec}(\boldsymbol{v}_i),$ where
$\operatorname{vec}(\boldsymbol{v}_i) \sim \mathcal{N}_{mc}(\mathbf{0}, \sigma_f^2 \mathbb{I}_{mc})$. And in the case of multivariate case, we abuse the notation to write $\theta^{(t)} := (\theta^{(t)},L^{(t)}, {\sigma_{f}}^{(t)})$ like the univariate case. 

From this, the complete log-likelihood is given as
\begin{equation*}
\begin{split}    
    \ell^{com}(\theta | \boldsymbol{f}_{1:n}, \boldsymbol{z}_{1:n}) 
    =&-\frac{cnm}{2} \log(2 \pi \sigma_{f}^{2}) - \frac{cnq}{2} \log(2 \pi)  - \frac{ \sum_{i=1}^{n}\operatorname{vec}(\boldsymbol{z}_{i})^{\top} \operatorname{vec}(\boldsymbol{z}_{i})}{2} \\
    &- \frac{ \sum_{i=1}^{n} (\operatorname{vec}(\boldsymbol{f}_i) - \operatorname{vec}(\boldsymbol{\tilde{f}}_i)  )^\top (\operatorname{vec}(\boldsymbol{f}_i) - \operatorname{vec}(\boldsymbol{\tilde{f}}_i)  ) }{2 \sigma_{f}^{2}} \\
    =&-\frac{cnm}{2} \log(2 \pi \sigma_{f}^{2}) - \frac{cnq}{2} \log(2 \pi)  - \frac{ \sum_{i=1}^{n}\operatorname{vec}(\boldsymbol{z}_{i})^{\top} \operatorname{vec}(\boldsymbol{z}_{i})}{2} \\
    & - \frac{ \sum_{i=1}^{n} \boldsymbol{B}_i^\top \boldsymbol{B}_i }{2 \sigma_{f}^{2}}
    ,
\end{split}
\end{equation*}
where $\boldsymbol{B}_i = \operatorname{vec}(\boldsymbol{f}_i) - \operatorname{vec}(\boldsymbol{\mu}_{\theta}) - (\boldsymbol{\Phi}_{\theta} \otimes L )\operatorname{vec}(\boldsymbol{z}_{i})$, $\boldsymbol{f}_{1:n}=\{\boldsymbol{f}_1,\ldots,\boldsymbol{f}_n\}, 
\boldsymbol{z}_{1:n}=\{\boldsymbol{z}_1,\ldots,\boldsymbol{z}_n\}, \boldsymbol{\mu}_{\theta}=
(\mu_{\theta}(\boldsymbol{x}_1^{(d)}),\ldots, \mu_{\theta}(\boldsymbol{x}_m^{(d)}))^\top$ is an $m \times c$ matrix
and $\boldsymbol{\Phi}_{\theta} = (\Phi_{\theta}(\boldsymbol{x}_1^{(d)}),\ldots,\Phi_{\theta}(\boldsymbol{x}^{(d)}_m))^\top$ is an $m \times q$ matrix.
Thus, for given parameter $\theta^{(t-1)}$ at time $t-1,$
the E-step is to calculate the conditional expectation of the complete log-likelihood which is given as
\begin{equation*}
\begin{split}
    Q(\theta|\theta^{(t-1)})=&\mathbb{E}_{\boldsymbol{z}_{1:n} \mid \boldsymbol{f}_{1:n}, \theta^{(t-1)}}[\ell^{com}(\theta | \boldsymbol{f}_{1:n}, \boldsymbol{z}_{1:n})] \\
    = &-\frac{cnm}{2} \log(2 \pi {\sigma_{f}^{2}}) - \frac{cnq}{2} \log(2 \pi) - \frac{ \sum_{i=1}^{n} \operatorname{tr} \mathbb{E}[ \operatorname{vec}(\boldsymbol{z}_{i}) \operatorname{vec}(\boldsymbol{z}_{i})^{\top}|\theta^{(t-1)},\boldsymbol{f}_{1:n}]}{2} \\
    & - \frac{1}{2{\sigma_{f}^{2}}}\sum_{i=1}^{n} \bigg\{\ (\operatorname{vec}(\boldsymbol{f}_i) - \operatorname{vec}(\boldsymbol{\mu}_{\theta}))^\top (\operatorname{vec}(\boldsymbol{f}_i) - \operatorname{vec}(\boldsymbol{\mu}_{\theta})) \\
    & +2(\operatorname{vec}(\boldsymbol{f}_i) - \operatorname{vec}(\boldsymbol{\mu}_{\theta}))^\top (\boldsymbol{\Phi}_{\theta} \otimes L )\mathbb{E}[\operatorname{vec}(\boldsymbol{z}_{i})|\theta^{(t-1)},\boldsymbol{f}_{1:n}]\\
    &+ \operatorname{tr} \left(( L^\top \otimes\boldsymbol{\Phi}_{\theta}^\top )(\boldsymbol{\Phi}_{\theta} \otimes L ) \mathbb{E}[\operatorname{vec}(\boldsymbol{z}_{i}) \operatorname{vec}(\boldsymbol{z}_{i})^{\top}|\theta^{(t-1)},\boldsymbol{f}_{1:n}] \right) \bigg\}, 
\end{split}
\end{equation*}
where
\begin{equation*}
    \begin{split}
        \mathbb{E}\left[\operatorname{vec}(\boldsymbol{z}_{i})|\theta^{(t-1)},\boldsymbol{f}_{1:n}\right] 
        = & \frac{\mathbb{V}\left[\operatorname{vec}(\boldsymbol{z})|\theta^{(t-1)},\boldsymbol{f}_{1:n}\right]         \left({{L}^{(t-1)}}^{\top} \otimes { \boldsymbol{\Phi}_{\theta^{(t-1)}}^{\top} }\right) \left(\boldsymbol{f}_{i} - \boldsymbol{\mu}_{\theta^{(t-1)}}\right)}{{\sigma_{f}^{2}}^{(t-1)}} \\
        \mathbb{E}\left[\operatorname{vec}(\boldsymbol{z}_{i})\operatorname{vec}(\boldsymbol{z}_{i})^{\top}|\theta^{(t-1)},\boldsymbol{f}_{1:n}\right]  
        = &\mathbb{V}\left[\operatorname{vec}(\boldsymbol{z})|\theta^{(t-1)},\boldsymbol{f}_{1:n}\right] \\
        &+ \mathbb{E}\left[\operatorname{vec}(\boldsymbol{z})|\theta^{(t-1)},\boldsymbol{f}_{1:n}\right]
        \mathbb{E}\left[\operatorname{vec}(\boldsymbol{z})|\theta^{(t-1)},\boldsymbol{f}_{1:n}\right]^{\top}
    \end{split}
\end{equation*}
where $\mathbb{V}\left[\operatorname{vec}(\boldsymbol{z})|\theta^{(t)}\right] = \left(I_{qk}+ \left({{L}^{(t)}}^{\top} \otimes \boldsymbol{\Phi}_{\theta^{(t)}}^{\top}\right)\left({{L}}^{(t)} \otimes \boldsymbol{\Phi}_{\theta^{(t)}}\right) /{\sigma_{f}^{2}}^{(t)}\right)^{-1}$.

In the M-step, instead of maximizing $Q(\theta|\theta^{(t-1)})$, we update $\theta$ by use of a stochastic gradient descent algorithm.


\section{Experimental details} \label{experimental details}
In this section, we describe the overall setup of our experiments in detail, focusing on the datasets, model architectures, training procedures, and evaluation metrics used in regression and classification problems. 
Four baselines are considered: (i) a small ensemble of lightweight networks (“small-Ens”), (ii) Hydra, and (iii) BE for regression and LBE for classification.

\subsection{Regression case} \label{experimental details: regression case}

We consider the standard regression problem 
\begin{equation*}
    y=f^{*}(x)+\epsilon, \epsilon \sim \mathcal{N}\left(0, \sigma_{\epsilon}^2 \right).
\end{equation*}
Suppose that there are $n$ many teacher models $f^{(t)}_{i}(\cdot)$  and $\sigma_{\epsilon, i}^{(t)}$.
Then, the predictive distribution of $y$ given $\boldsymbol{x}$ is constructed as 
$p^{(t)}(y|\boldsymbol{x})=\frac{1}{n} \sum_{i=1}^n \mathcal{N}(y|f^{(t)}_{i}(\boldsymbol{x}),{\sigma_{\epsilon,i}^{(t)}}^{2}),$
where $\mathcal{N}(y|\mu,\sigma^2)$ is the density of the Gaussian distribution with mean $\mu$ and variance $\sigma^2$
The predictive distribution $p^{(s)}(y|\boldsymbol{x})$ based on student ensemble members is defined similarly.

We obtain 50 teacher models of DNNs with two hidden layers and 100 nodes at each layer which are learned by minimizing the sum of squared residuals of the training data with 50 randomly selected initial solutions.
For the design points used in the distillation, we use the training data themselves.
The architecture of student models comprises of an one-hidden-layer MLP with 50 units.
The number of parameters in the student ensemble of each method is summarized in Table \ref{tab:model parameter regression}. Note that the number of parameters of DLF is smaller than those of the other baselines because the dimension of the latent factor is 10 instead of 50. 

\begin{table}[h]
\centering
\caption{The number of parameters in the student ensemble.}
\label{tab:model parameter regression}
\resizebox{\textwidth}{!}{%
\begin{tabular}{c|l|cccccc}
\toprule
\multirow{2}{*}{} &
  \multirow{2}{*}{Method} &
  \multicolumn{6}{c}{Datasets} \\
\cline{3-8}
 &  & Boston housing & Concrete & Energy & Wine & Power & Kin8nm \\
\midrule
\multirow{5}{*}{\# of parameters} 
& Teachers
  & 90,000 & 65,000 & 65,000 & 80,000 & 45,000 & 65,000 \\
& small-Ens 
  & 45,000 & 32,500 & 32,500 & 40,000 & 22,500 & 32,500 \\
& Hydra       
  & 6,650 & 6,150 & 6,150 & 6,450 & 5,750 & 6,150 \\
& BE          
  & 7,351 & 6,601 & 6,601 & 7,051 & 6,001 & 6,001 \\
& DLF(factor dim = 10)
  & 1,862 & 1,612 & 1,612 & 1,762 & 1,412 & 1,612 \\
\bottomrule
\end{tabular}%
}
\end{table}

The Adam \citep{kingma2014adam} is used for the optimization. 

\subsubsection{Dataset} \label{Dataset:Reg}
We consider the following 6 UCI datasets (Boston housing, Concrete, Energy, Wine, Power Plant, Kin8nm) \citep{asuncion2007uci}. 
We divide each dataset into a 9:1 ratio randomly for the training and test data.
The experiment is repeated with 10 random split to have 10 evaluation metrics.

\begin{table}[h]
\centering
\caption{Description of UCI benchmark datasets used in the experiment}
\label{tab:uci_data_discription}
\begin{tabular}{ccc}
\toprule
Dataset        & size & \# of features \\ \midrule
Boston housing & 506  & 13             \\ \midrule
Concrete       & 1030 & 8              \\ \midrule
Energy         & 1030 & 8              \\ \midrule
Wine           & 9568 & 11             \\ \midrule
Power          & 768  & 4              \\ \midrule
Kin8nm         & 8192 & 8              \\ 
\bottomrule
\end{tabular}%
\end{table}

\subsubsection{Evaluation metric}\label{Evaluation metric:Reg}
Let $\{(\boldsymbol{x}_{1}, y_{1}),\cdots ,(\boldsymbol{x}_{m_{\mathrm{test}}}, y_{m_{\mathrm{test}}})\}$ be given test data.

\paragraph{Root Mean Square Error} Root Mean Square Error (RMSE) is defined as
\begin{equation*}
    \mathrm{RMSE} = \sqrt{\frac{1}{m_{\mathrm{test}}}\sum_{j}^{m_{\mathrm{test}}}(y_j - \hat{\mathbb{E}}(y | \boldsymbol{x}_j))^2} 
\end{equation*}
where $\hat{\mathbb{E}}(y | \boldsymbol{x}) = \int 
 y \hat{p}(y |\boldsymbol{x} )dy$ and $\hat{p}(y|\boldsymbol{x})$ is an estimated predictive distribution.

\paragraph{Negative Log Likelihood} Negative Log Likelihood (NLL) is defined as
\begin{equation*}
    \mathrm{NLL} = -\sum_{j=1}^{m_{\mathrm{test}}} \log \hat p(y_j \mid \boldsymbol{x_j}).
\end{equation*}

\paragraph{Continuous Ranked Probability Score} Continuous Ranked Probability Score (CRPS) is defined as
\begin{equation*}
    \mathrm{CRPS} = \frac{1}{m_{\mathrm{test}}}\sum_{j=1}^{m_{\mathrm{test}}} \int_{\mathbb{R}}\left[ \widehat{F}_j(v) - \mathbf{1}(v \geq y_j))\right]^2 dv
\end{equation*}
where $\widehat{F}_j(v) = \int_{-\infty}^{v} \hat{p}(y \mid \boldsymbol{x}_j) dy$.

\subsection{Classification case} \label{experimental details: classification case}

We consider a $c$-class classification problem where
\begin{equation*}
    p(y|\boldsymbol{x},\boldsymbol{f}) = \frac{\exp \left(f_y(\boldsymbol{x}) \right)}{\sum_{l=1}^{c} \exp \left(f_l(\boldsymbol{x}) \right)},\;\;\; y\in \{1,\ldots,c\}    
\end{equation*}
for a given (vector-valued) function $\boldsymbol{f}(\cdot)=(f_1(\cdot),\ldots,f_c(\cdot)).$
For a given teacher ensemble $\boldsymbol{f}_i^{(t)}, i=1,\ldots,n,$ the teacher predictive distribution is estimated by
$p^{(t)}(y|\boldsymbol{x})=\sum_{i=1}^n p(y|\boldsymbol{x},\boldsymbol{f}_i^{(t)})/n.$
The student predictive distribution for a given student ensemble is defined similarly.

\subsubsection{Dataset} \label{Dataset:Cls}
In classification settings, we analyze two CIFAR datasets \citep{krizhevsky2009learning}. 
Each dataset contains 50,000 training and 10,000 test images of natural scenes, sized $32 \times 32$ pixels. 

\begin{table}[h]
\centering
\caption{Description of CIFAR-10 and CIFAR-100}
\label{tab:cifar_data_discription}
\begin{tabular}{cccc}
\toprule
Dataset        & Train size & Test size & \# of labels \\
\midrule
CIFAR-10 & 50000 & 10000  & 10           \\
\midrule
CIFAR-100 & 50000 & 10000  & 100              \\ 
\bottomrule
\end{tabular}%
\end{table}

We follow the set-up of experiments in \citep{nam2022improving}.
As a teacher, we use an ensemble of four neural networks, where each model is a Wide-ResNet (WRN) \citep{zagoruyko2016wide}. Specifically, WRN-28-1 is used for CIFAR-10 and WRN-28-4 is used for CIFAR-100. 
And, the student model uses the WRN-16-1 network for CIFAR-10 and the WRN-28-1 network for CIFAR-100. 

Training lasts 200 epochs on a single GPU using SGD with Nesterov momentum of 0.9, weight decay of $5 \times 10^{-4}$, and batch size of 128. 
A one-cycle cosine annealing schedule with a five-epoch linear warm-up (from 0.001 to 0.1) is employed. 
The number of parameters in each ensemble is summarized in Table \ref{tab:model parameter classification}.

\begin{table}[h]
\centering
\caption{Number of model parameters in each ensemble.}
\label{tab:model parameter classification}
\begin{tabular}{c|l|cc}
\toprule
\multirow{2}{*}{} &
  \multirow{2}{*}{Method} &
  \multicolumn{2}{c}{Datasets} \\
\cline{3-4}
 &  & CIFAR-10 & CIFAR-100 \\
\midrule
\multirow{5}{*}{\# of parameters} 
& Teachers
  & 1.48M & 23.488M \\
& small-Ens 
  & 0.70M & 1.50M \\
& Hydra       
  & 0.18M & 0.39M \\
& LBE          
  & 0.18M & 0.38M \\
& DLF(factor dim = 8)
  & 0.18M  & 0.38M \\
\bottomrule
\end{tabular}%
\end{table}

\subsubsection{Evaluation metric}\label{Evaluation metric:Cls}

\paragraph{Accuracy} Accuracy (ACC) is defined as
\begin{equation*}
\mathrm{ACC} = \frac{1}{m_{\mathrm{test}}} \sum_{j=1}^{m_{\mathrm{test}}} \mathbf{1}(y_j = \hat{y}_j),
\end{equation*}
where $\hat{y}_j = \arg\max_y \hat{p}(y| \boldsymbol{x}_j)$.

\paragraph{Negative Log-Likelihood} Negative log-likelihood (NLL) is defined as
\begin{equation*}
\mathrm{NLL} = -\frac{1}{m_{\mathrm{test}}}\sum_{j=1}^{m_{\mathrm{test}}}\sum_{k=1}^{c} \mathrm{1}(y_j = k) \log \hat{p}(y=k\mid \boldsymbol{x}_j).
\end{equation*}

\paragraph{Expected Calibration Error} Expected Calibration Error (ECE) \citep{guo2017calibration} is defined as
\begin{equation*}
\mathrm{ECE} = \sum_{l=1}^{M} \frac{|B_l|}{m_{\mathrm{test}}}\,\left| \frac{1}{| B_l|}\sum_{(\boldsymbol{x}_j,y_j)\in B_l}\mathbf{1}(y_j = \hat y_j) - \frac{1}{| B_l|}\sum_{(\boldsymbol{x}_j,y_j)\in B_l}p(y_j\mid \boldsymbol{x}_j) \right|,
\end{equation*}
where $\{B_1,\dots,B_M\}$ is a partition of the test data $\mathcal{D}_{\mathrm{test}}$ such that $B_l = \bigl\{(\boldsymbol{x},y)\in \mathcal{D}_{\mathrm{test}} \mid p(\hat{y}\mid \boldsymbol{x})\in \bigl((l-1)/M,l/M\bigr]\bigr\}$. In this work, $M$ is set to be 15.

\subsection{Fine-tuning of language models} \label{fine-tuning of language models}
The experiments are conducted on three datasets from the GLUE \cite{wang2018glue} and SuperGLUE \cite{wang2019superglue} benchmark : RTE, MRPC, and WiC. Table~\ref{tab:glue_datasets} summarizes the details of each dataset.

\begin{table}[h]
\centering
\caption{Description of GLUE and SuperGLUE benchmark datasets used in the experiments}
\label{tab:glue_datasets}
\renewcommand{\arraystretch}{1.3}
\begin{tabular}{p{1.5cm} p{3cm} p{7.5cm}}
\toprule
\textbf{Dataset} & \textbf{Task Type} & \textbf{Description} \\
\midrule
RTE & Recognizing Textual Entailment & Determines whether a given hypothesis can be inferred from a given premise sentence. \\
\cmidrule(lr){1-3}
MRPC & Paraphrase Detection & Identifies whether two sentences are semantically equivalent. The dataset consists of sentence pairs from news sources. \\
\cmidrule(lr){1-3}
WiC & Word Sense Disambiguation & Determines whether a specific word used in two different contexts has the same meaning. Contextual understanding of word senses is essential. \\
\bottomrule
\end{tabular}
\end{table}

\vspace{1em}

\paragraph{Model Construction}
For teacher models, \textit{RoBERTa} \citep{liu2019roberta} is fine-tuned with the Low-Rank Adaptation (LoRA) method \citep{hu2022lora} on each task. For each dataset, four fine-tuned models are trained using different random initializations to construct teacher ensemble.

Student models are based on \textit{DistilRoBERTa} \citep{sanh2019distilbert}, which is also fine-tuned with LoRA. 
Features are extracted through the shared backbone, and task-specific prediction heads are constructed depending on the design of each distillation method.

\paragraph{Training and Evaluation Settings}
The experimental settings, including learning rate, batch size, number of epochs, and the rank of LoRA, follow the configuration used in \citep{yang2023bayesian}. 
Model performances are evaluated using the three metrics described in Appendix \ref{experimental details: classification case}: Acc, NLL and ECE.

\subsection{Application to distribution shift}\label{Application to distribution shift}

To apply the framework introduced in Section \ref{Application to distribution shift problems} to classification tasks, we consider the following model:
\begin{equation*}
   y \mid x, \mathbf{W} \sim \mathrm{Multinomial}\bigl(\operatorname{softmax}\bigl(f(x; \mathbf{W})\bigr)\bigr), 
\end{equation*}
\begin{equation*}
    f(x; \mathbf{W}) = \mu_{\theta}(x) + L \mathbf{W} \Phi_{\theta}(x),
\end{equation*}
where $\mathbf{W}$ is the weight matrix. 
For new data $\mathcal{D}^{new},$ we estimate $\mathbf{W}$ by minimizing the cross-entropy on new data
while $\theta$ and $L$ are fixed at $\hat\theta$ and $\hat{L}$ estimated on training data.

For numerical study, we consider a distribution shift scenario where the conditional distribution $P(X\mid Y)$ changes \citep{moreno2012unifying}.
To generate synthetic data, we use CIFAR-10 and flip the labels
by $y \mapsto 9-y$ to construct CIFAR 10-Flip dataset \citep{lee2022surgical}.
We first learn the body by a DNN, Hydra, and DLF on the training data of 50,000 CIFAR-10 images under the WRN-16-1 architecture
for the teacher and student ensembles. Then, we train the weights of the linear head
on new data while the body is fixed.


\subsection{Hardwares} \label{appendix: resource}
All our experiments are done through Python 3.9.16 with Intel(R) Xeon(R) Silver 4310 CPU @ 2.10GHz, NVIDIA TITAN Xp GPU and 128GB RAM.

\section{Ablation Study} \label{ablation study}
We do the following ablation studies on Boston housing data.
\begin{itemize}
    \item We investigate how the choice of design points affects distillation performance.
    \item The effect of the choice of latent factor dimension affects distillation performance.
    \item The effect of the MMD-based initialization is investigated to assess its role in stabilizing the EM algorithm during training.
    \item The capacity of the student models is varied by adjusting the network width, and the effect of this change is analyzed for each baseline.
    \item We investigate the effect of ensemble size by varying the number of ensemble members used in the teacher and student ensembles.
\end{itemize}

Quantitative results are presented using RMSE, NLL, and CRPS, aggregated over ten independent runs.

\subsection{Choice of design points} 
\subsubsection{Comparison of different design points selection methods} \label{ablation:design points type}
We investigate how different strategies of selection design points influence the performance of the Gaussian distillation.
The entire dataset is partitioned into three disjoint subsets: $\mathcal{D} = \mathcal{D}^{\mathrm{train}}_{\mathrm{teacher}} \oplus \mathcal{D}^{\mathrm{train}}_{\mathrm{new}} \oplus \mathcal{D}^{\mathrm{test}}$, with a fixed ratio of $4.5:4.5:1$.
The teacher ensemble is trained on $\mathcal{D}^{\mathrm{train}}_{\mathrm{teacher}}$, and the student model is distilled using various forms of design points $\mathcal{D}^{\mathrm{design}}$.
To analyze the impact of the choice of design points, the four distinct strategies for selecting $\mathcal{D}^{\mathrm{design}}$ are considered:
\begin{itemize}
    \item \textbf{Design 1:} Directly using the teacher training data, $\mathcal{D}^{\mathrm{design}} = \mathcal{D}^{\mathrm{train}}_{\mathrm{teacher}}$.
    \item \textbf{Design 2:} Using mixup samples from $\mathcal{D}^{\mathrm{train}}_{\mathrm{teacher}}$. $\mathcal{D}^{\mathrm{design}} \subset \mathrm{mixup}\{\mathcal{D}^{\mathrm{train}}_{\mathrm{teacher}}\}$.
    \item \textbf{Design 3:} Using a training data not used in training teacher ensemble, $\mathcal{D}^{\mathrm{design}} = \mathcal{D}^{\mathrm{train}}_{\mathrm{new}}$.
    \item \textbf{Design 4:} Using mixup samples from $\mathcal{D}^{\mathrm{train}}_{\mathrm{new}}$. $\mathcal{D}^{\mathrm{design}} \subset \mathrm{mixup}\{\mathcal{D}^{\mathrm{train}}_{\mathrm{new}}\}$
\end{itemize}


Here, $\mathrm{mixup}\{\mathcal{D}^{\mathrm{train}}_{\mathrm{teacher}}\}$ denotes the set of samples generated by linearly combining two randomly selected samples from $\mathcal{D}^{\mathrm{train}}_{\mathrm{teacher}}$.
For each index $j$, we randomly draw two data pairs $(\boldsymbol{x}_j, y_j)$ and $(\boldsymbol{x}_j^c, y_j^c)$ from $\mathcal{D}^{\mathrm{train}}_{\mathrm{teacher}}$ and form the mixed sample
$$(\boldsymbol{x}_j^m, y_j^m) := \lambda_j(\boldsymbol{x}_j, y_j) + (1 - \lambda_j)(\boldsymbol{x}_j^c, y_j^c), \quad \lambda_j \in [0, 1].$$
The set $\mathrm{mixup}\{\mathcal{D}^{\mathrm{train}}_{\mathrm{teacher}}\}$ consists of all mixed pairs $(\boldsymbol{x}_j^m, y_j^m)$, with 
the number of generated design points in Designs 2 and 4 matched to the size of $\mathcal{D}^{\mathrm{train}}_{\mathrm{teacher}}$.
These strategies are designed to cover both scenarios where design points are reused from the teacher training data and where additional or perturbed data are incorporated.
All methods are evaluated on the reserved test data $\mathcal{D}^{\mathrm{test}}$.

\begin{table}[h]
\centering
\caption{Performances of Gaussian distillation for different strategies of the design point selection}
\label{tab:design points type}
\resizebox{0.7\textwidth}{!}{%
\begin{tabular}{c|c|c|c}
\toprule
Design Strategy        & RMSE            & NLL             & CRPS            \\ \midrule
Design 1         & 3.2316 (0.0587) & 2.6317 (0.0267) & 1.7896 (0.0271) \\ 
Design 2 & 3.5856 (0.1263) & 2.8003 (0.0638) & 1.9648 (0.0681) \\ 
Design 3         & \textbf{2.8786 (0.0322)} & \textbf{2.4816 (0.0129)} & \textbf{1.6427 (0.0165)} \\ 
Design 4 & 3.1787 (0.2498) & 2.6122 (0.1144) & 1.792 (0.1310)  \\ 
\bottomrule
\end{tabular}%
}
\end{table}

\begin{figure}[h]
    \centering
    \includegraphics[width=\textwidth]{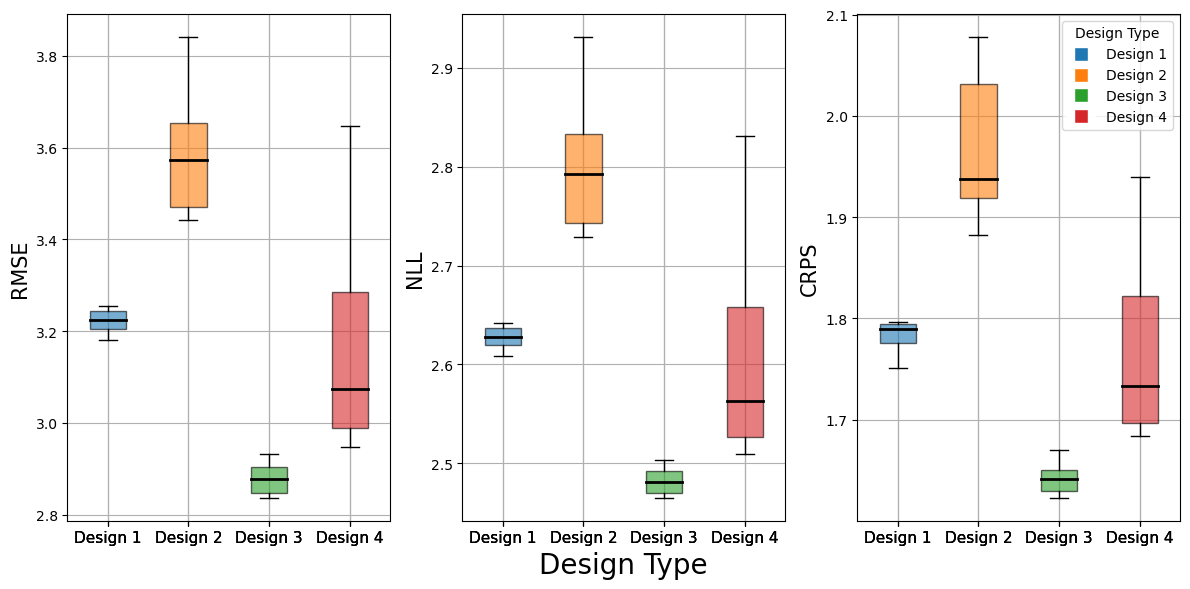}
    \caption{Box plot of evaluation metrics (RMSE, NLL, CRPS) for different strategies of the design point selection}
    \label{fig:ablation_design_type}
\end{figure}

The results summarized in Table~\ref{tab:design points type} and visualized in Figure~\ref{fig:ablation_design_type} illustrate the effect of different choices of design points on the performance of the Gaussian distillation..
Among the four strategies, the use of new training data (Design 3) for design points consistently yields the best performance across all metrics.
In contrast, strategies involving mixup consistently yield inferior performances regardless of whether teacher training data (Design 2) or new data (Design 4) are used.

To sum up, the results suggest that using new training data is the best for Gaussian distillation.
However, the size of the teacher training data becomes smaller, which might lead to performance degradation. 
In practice, we could find the optimal partition of teacher training data and new training data based on additional validation data.

\subsubsection{Comparison for different selection of the number of design points} \label{ablation:design ratio}
We also investigate the effect of the number of design points on the student model.
As in the previous experimental setup, the teacher ensemble is trained on $\mathcal{D}^{\mathrm{train}}_{\mathrm{teacher}}$, and the student model is distilled using various sizes of design sets $\mathcal{D}^{\mathrm{design}}$.
We only consider the design type $\mathcal{D}^{\mathrm{train}}_{\mathrm{teacher}}$ (Design 1) and $\mathcal{D}^{\mathrm{train}}_{\mathrm{new}}$ (Design 3).
The number of design points is controlled by the \emph{design ratio}
\begin{equation*}
    r \in \{ 0.2, 0.4, 0.6, 0.8, 1.0\}
\end{equation*}
meaning that Gaussian distillation uses $r \times |\mathcal{D}^{\mathrm{design}}|$ samples for distillation.

\begin{table}[h]
\centering
\caption{Performance for Gaussian distillation for the number of design points}
\label{tab:design ratio}
\resizebox{1\textwidth}{!}{%
\begin{tabular}{c|lll|lll}
\toprule
Design Type &
  \multicolumn{3}{c|}{Design 1} &
  \multicolumn{3}{c}{Design 3} \\ \midrule
design\_ratio &
  \multicolumn{1}{c|}{RMSE} &
  \multicolumn{1}{c|}{NLL} &
  \multicolumn{1}{c|}{CRPS} &
  \multicolumn{1}{c|}{RMSE} &
  \multicolumn{1}{c|}{NLL} &
  \multicolumn{1}{c}{CRPS} \\ \midrule
0.2 &
  \multicolumn{1}{l|}{3.9829 (0.4823)} &
  \multicolumn{1}{l|}{3.0244 (0.2738)} &
  2.1773 (0.1704) &
  \multicolumn{1}{l|}{3.8643 (0.3277)} &
  \multicolumn{1}{l|}{2.951 (0.1728)} &
  2.145 (0.1349) \\
0.4 &
  \multicolumn{1}{l|}{3.576 (0.3231)} &
  \multicolumn{1}{l|}{2.8017 (0.1612)} &
  1.9818 (0.1349) &
  \multicolumn{1}{l|}{3.0703 (0.0735)} &
  \multicolumn{1}{l|}{2.5612 (0.0316)} &
  1.7464 (0.0396) \\
0.6 &
  \multicolumn{1}{l|}{3.3996 (0.1750)} &
  \multicolumn{1}{l|}{2.711 (0.0834)} &
  1.9122 (0.0573) &
  \multicolumn{1}{l|}{3.0055 (0.0634)} &
  \multicolumn{1}{l|}{2.5337 (0.0267)} &
  1.7161 (0.0243) \\
0.8 &
  \multicolumn{1}{l|}{3.2548 (0.1407)} &
  \multicolumn{1}{l|}{2.6433 (0.0642)} &
  1.8444 (0.0663) &
  \multicolumn{1}{l|}{3.0229 (0.0945)} &
  \multicolumn{1}{l|}{2.5413 (0.0395)} &
  1.7197 (0.0444) \\
1 &
  \multicolumn{1}{l|}{3.2316 (0.0587)} &
  \multicolumn{1}{l|}{2.6317 (0.0267)} &
  1.7896 (0.0271) &
  \multicolumn{1}{l|}{2.8786 (0.0322)} &
  \multicolumn{1}{l|}{2.4816 (0.0129)} &
  1.6427 (0.0165) \\ 
  \bottomrule
\end{tabular}%
}
\end{table}
\begin{figure}[h]
    \centering
    \includegraphics[width=\textwidth]{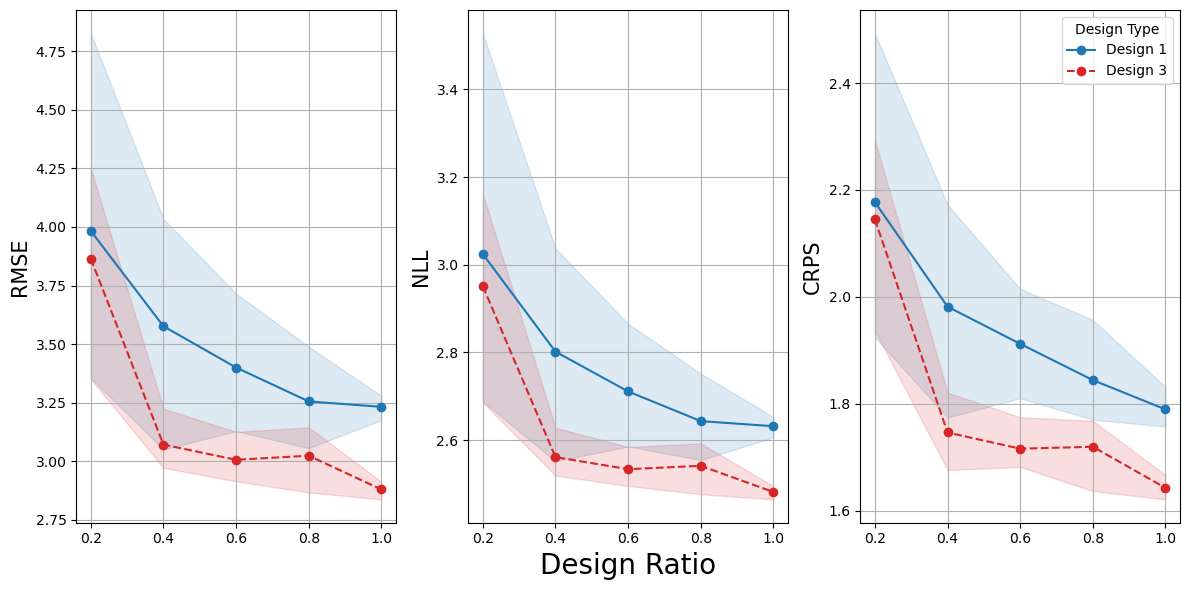}
    \caption{Evaluation metrics (RMSE, NLL, CRPS) versus the design ratio.}
    \label{fig:ablation_design_number}
\end{figure}

The results summarized in Table~\ref{tab:design ratio} and Figure~\ref{fig:latent_factor_dim} show that increasing the design ratio consistently improves the performance of the Gaussian distillation.
That is, the larger the number of design points, the better the performance of Gaussian distillation is.

\subsection{Dimension of the latent factor} \label{ablation:dimension of latent factor}
We investigate the influence of the dimension of the latent factor on the performance of Gaussian distillation. 
The result visualized in Figure \ref{fig:latent_factor_dim} indicates that the performance of Gaussian distillation is not
sensitive to the dimension of the latent factor unless the dimension is too small.

\begin{figure}[h]
    \centering
    \includegraphics[width=\textwidth]{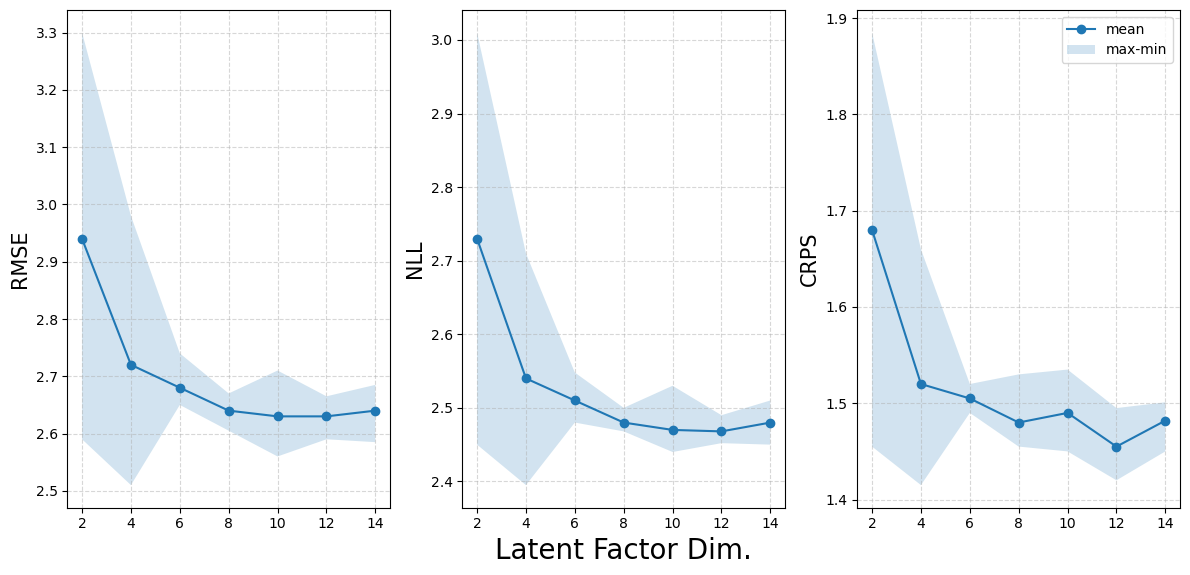}
    \caption{Evaluation metrics (RMSE, NLL, CRPS) versus latent factor dimension.}
    \label{fig:latent_factor_dim}
\end{figure}

\subsection{MMD vs random initial}
We compare the proposed MMD initialization in Section \ref{Estimation of the mean and factor loading} with the random initialization.
As shown in Figure \ref{fig:mmd}, the MMD initialization strategy outperforms the random initialization unless the dimension of the latent factor is too small. 
In addition, the variations of the evaluation metrics for the MMD initialization are much smaller than those of the random initialization. 
That is, the MMD initialization is indispensable for the superior performance of Gaussian distillation.

\begin{figure}[h]
    \centering
    \includegraphics[width=\textwidth]{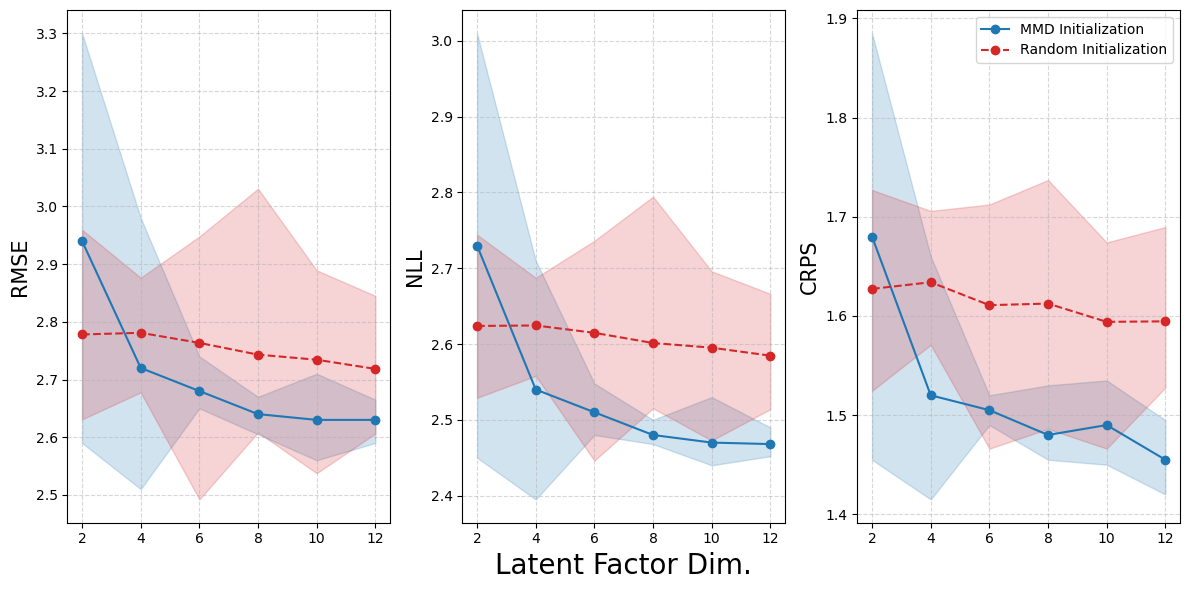}
    \caption{Evaluation metrics (RMSE, NLL, CRPS) versus latent factor size with or without the MMD initialization
    }
    \label{fig:mmd}
\end{figure}

\subsection{Capacity of student models} \label{ablation:size of student}
The effect of the capacity of student models is examined by varying the number of nodes in the one-layer DNN architecture. We increase the number of nodes gradually from 50 to 60, 70, 80, 90, and 100, and obtain the evaluation metrics of the distillation methods. 
The results in Figure \ref{fig:model size} show that the performances of DLF and small-Ens keep improving as the number of nodes increases, while the performances of Hydra and BE are saturated. 
Apparently, DLF behaves similarly to small-Ens, which is interesting since small-Ens demands heavier computation and much more storage.

\begin{figure}[h]
    \centering
    \includegraphics[width=\textwidth]{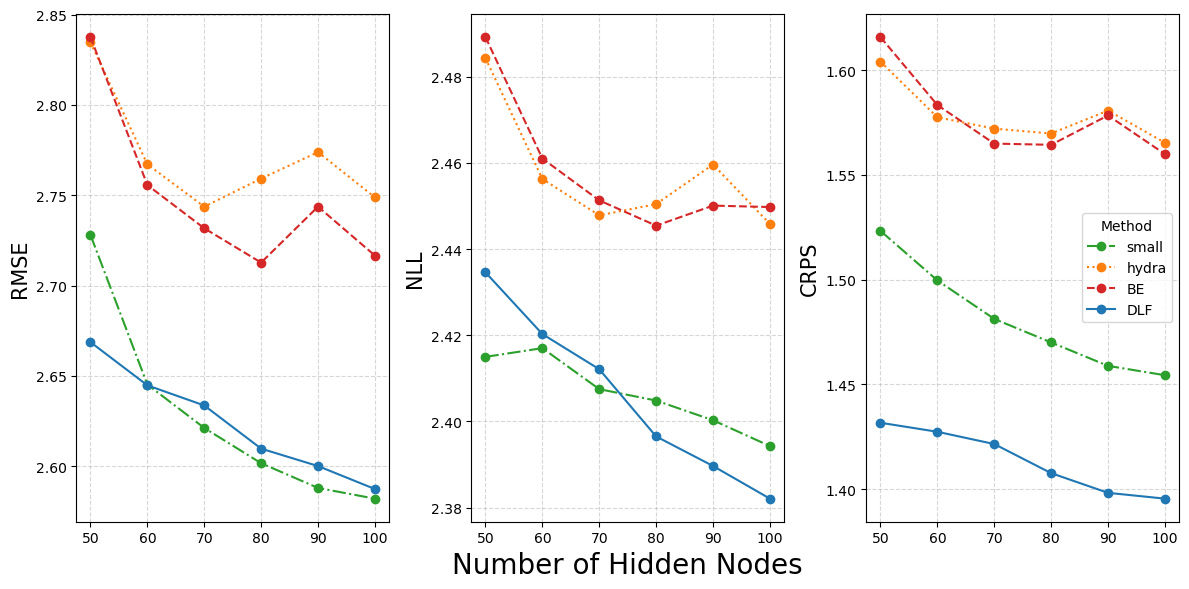}
    \caption{Evaluation metrics (RMSE, NLL, CRPS) versus the number of hidden nodes. 
    }
    \label{fig:model size}
\end{figure}

\subsection{Number of ensemble members} \label{ablation:number of ensemble}
To evaluate the effect of ensemble size, we evaluate the performance of the KD methods by varying the number of ensemble members from 10 to 50.
As we can see from Figure \ref{fig:ensemble number}, for all methods, the performances keep improving as the ensemble size increases.
Note that the number of weights in the DLF model is not proportional to the ensemble sizes (instead, it is proportional to the dimension of the latent factor), while it is proportional for the other baselines. This is an additional advantage of Gaussian distillation.

\begin{figure}[h]
    \centering
    \includegraphics[width=\textwidth]{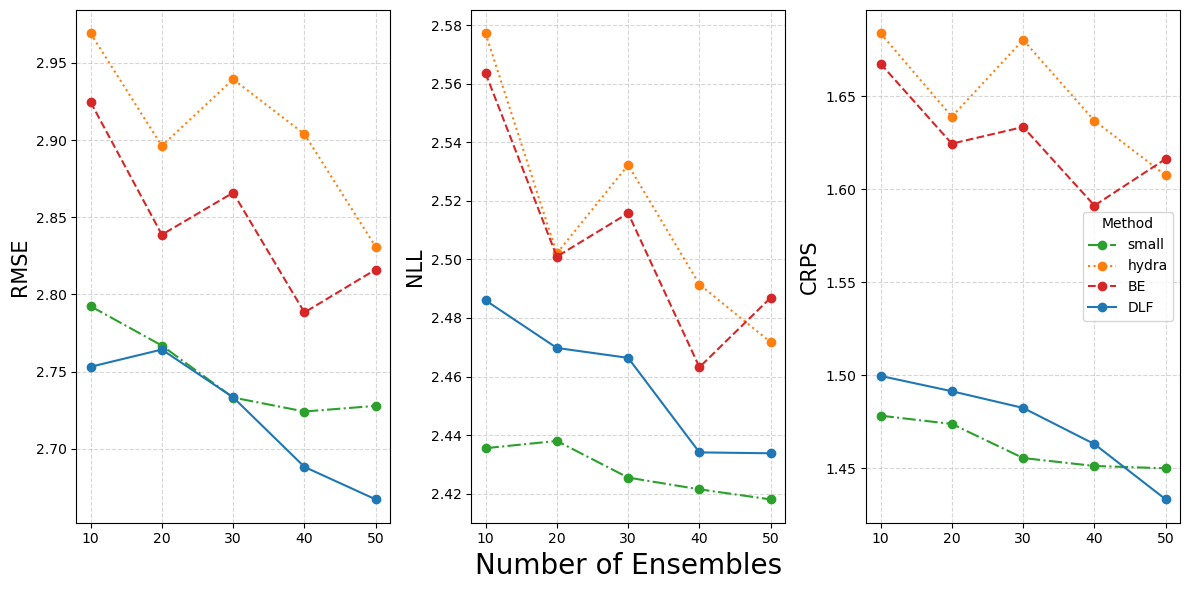}
    \caption{Evaluation metrics (RMSE, NLL, CRPS) versus the number of ensemble members.
    }
    \label{fig:ensemble number}
\end{figure}

\subsection{Application to large dataset}
To further examine the scalability and robustness of our proposed model, we conducted an additional experiment on a larger and more complex dataset, Tiny ImageNet \citep{le2015tiny}. 
This dataset extends the original ImageNet hierarchy but contains a reduced subset of classes and image resolutions, making it a challenging yet computationally manageable benchmark for evaluating model generalization on large-scale visual domains.

Tiny ImageNet consists of 200 classes with 100,000 training images and 10,000 test images, where each image is resized to 64×64 pixels. 
To ensure fair evaluation and avoid overfitting, we randomly split the training set into 80\% for training and 20\% for validation.

We compared our method against Small Ensemble, Hydra, LBE, and $\text{Proxy-EnD}^2$. 
The teacher network was implemented using WRN-28-4, while the student network employed WRN-16-4. 
Both teacher and student models were trained under an ensemble size of four, following the same distillation pipeline described in Appendix \ref{experimental details: classification case}.

{\renewcommand{\arraystretch}{1.2}
\begin{table}[h]
\centering
\caption{Results on Tiny ImageNet.}
\label{tiny_imagenet_result}
\resizebox{0.65\textwidth}{!}{%
\begin{tabular}{c|l|ccc}
\toprule
dataset      & method                        & Acc(\%) $\uparrow$            & NLL $\downarrow$                   & ECE(\%) $\downarrow$            \\
\midrule
\multirow{6}{*}{Tiny ImageNet}
 & {\color[HTML]{00B0F0} Teacher} & {\color[HTML]{00B0F0} 68.74} & {\color[HTML]{00B0F0} 1.2806}  & {\color[HTML]{00B0F0} 3.22}    \\
 & Small-Ens                     & 57.11 (0.0027)       & 1.854 (0.0032)       & 13.51 (0.0032)         \\
 & Hydra                          & 56.80 (0.0037)       & 1.7503 (0.0092)       & 2.75 (0.092)         \\
 & LBE                            & 46.08 (0.0119)       & 2.24 (0.0407)       & 2.22 (0.0053)         \\
 & $\text{Proxy-EnD}^2$               & 51.64 (0.0022)       & 2.0195 (0.0069)       & 2.77 (0.0038)         \\
 & DLF                            & \textbf{58.34 (0.2121)} & \textbf{1.6737 (0.0043)} & \textbf{1.92 (0.1098)} \\ 
\bottomrule
\end{tabular}%
}
\end{table}
}

Table \ref{tiny_imagenet_result} shows that our approach outperforms all other methods across all measures, even when the number of classes and samples is large, suggesting that the proposed framework generalizes well beyond small-scale datasets.
In contrast, both LBE and $\text{Proxy-EnD}^2$ exhibit a significant drop in performance when applied to complex data, confirming that our latent factor modeling remains effective even in such challenging settings.

\section{Uncertainty quantification}
\label{ablation:uncertainty quantification}
We conduct experiments to evaluate whether each model appropriately quantifies uncertainty by detecting out-of-distribution (OOD) data in a classification problem. We adopt predictive mutual information, which estimates epistemic uncertainty \citep{gal2016uncertainty}, as the OOD detection score :
\begin{align*}
\widehat{\mathbb{I}}\big[y,\theta \mid \boldsymbol{x}\big]
&:= -\sum_{y=1}^{c}\!\left(\frac{1}{n}\sum_{i=1}^{n} p(y \mid \boldsymbol{x}, \hat{\theta}_{i})\right)
     \log\!\left(\frac{1}{n}\sum_{i=1}^{c} p(y \mid \boldsymbol{x}, \hat{\theta}_{i})\right) \\
&\quad + \frac{1}{n}  \sum_{y=1}^{c} \sum_{i=1}^{n} p(y \mid \boldsymbol{x}, \hat{\theta}_{i})
     \log p(y \mid \boldsymbol{x}, \hat{\theta}_{i}).
\end{align*}
Specifically, we denote the in-distribution (ID) dataset by $\{\boldsymbol{x}^{\text{in}}_{1},...,\boldsymbol{x}^{\text{in}}_{m_{1}}\}$, the OOD dataset by $\{\boldsymbol{x}^{\text{out}}_{1},...,\boldsymbol{x}^{\text{out}}_{m_{2}}\}$ and $y$ is the output of the model with corresponding  predictive mutual information $\widehat{\mathbb{I}}\big[y,\theta \mid \boldsymbol{x}^{\text{in}}_{i}\big]$ for $i=1,...,m_{1}$ and $\widehat{\mathbb{I}}\big[y,\theta \mid \boldsymbol{x}^{\text{out}}_{i}\big]$ for $i=1,...,m_{2}$. \\
For evaluation, we assign label 0 to ID data and label 1 to OOD data, meaning that lower predictive mutual information indicates ID whereas higher values indicate OOD. We then compute the AUROC between these labels and the predictive mutual information scores.\\
We use CIFAR-10 as the ID and SVHN, CIFAR-100, and Tiny ImageNet as OOD. 
All models are trained on the CIFAR-10 training set with 50,000 images, and OOD detection is evaluated on the CIFAR-10 test set with 10,000 images versus the test sets of the OOD datasets, where SVHN has 26,032 images, CIFAR-100 has 10,000 images, and Tiny ImageNet has 10,000 images. We repeat the entire procedure five times with different random seeds and report the mean of AUROC and its standard error. 
Table~\ref{tab:ood detection} show that DLF performs best on SVHN and remains competitive across CIFAR-100 and Tiny ImageNet.

\begin{table}[h]
\centering
\caption{AUROC Results on in-distribution and out-of-distribution detection.}
\label{tab:ood detection}
\resizebox{0.65\textwidth}{!}{%
\begin{tabular}{l|ccc}
\toprule
\multirow{2}{*}{Method} & \multicolumn{3}{c}{Out-Of-Distribution Datasets} \\
\cline{2-4}
 & SVHN & CIFAR-100 & Tiny ImageNet \\
\midrule 
{\color[HTML]{00B0F0} Teachers} & {\color[HTML]{00B0F0} 0.9403} & {\color[HTML]{00B0F0} 0.8604} & {\color[HTML]{00B0F0} 0.9400} \\
small-Ens & 0.9093 (0.0037) & 0.8000 (0.0036) & 0.7956 (0.0023) \\
Hydra & 0.7178 (0.0107) & 0.6644 (0.0083) & 0.6709 (0.0059)  \\
LBE & 0.8329 (0.0249) & 0.8107 (0.0130) & 0.8152 (0.0141) \\
$\text{Proxy-EnD}^2$ & 0.8427 (0.0344) & \textbf{0.8427 (0.0183)}& \textbf{0.8456 (0.0175)} \\
DLF & \textbf{0.9359 (0.0212)} & 0.8357 (0.0062)  & 0.8291 (0.0067)
 \\
\bottomrule
\end{tabular}%
}
\end{table}

\section{Computational Cost}
In this section, we report a comparison of the computational costs between our proposed method and baseline method during training.

\subsection{Training time}
First, we evaluate the training time on the CIFAR-10 and CIFAR-100 datasets.
For a fair comparison,  all experiments are conducted under the same hardware environment.
We follow the experimental setup described in Appendix \ref{experimental details: classification case}.
Each experiment is repeated four times, and the average training times are reported in hours.

\begin{table}[h]
\centering
\caption{Comparison of training time on CIFAR-10 and CIFAR-100 datasets.}
\label{tab:training_time}
\resizebox{0.45\textwidth}{!}{%
\begin{tabular}{l|c|c}
\toprule
\multirow{2}{*}{Method} & CIFAR-10 & CIFAR-100 \\ \cline{2-3}
& Training Time & Training Time \\ \midrule
Teachers & 10.2 & 17.4 \\ 
small-Ens & 4.6 & 12.5 \\ 
Hydra & 3.1 & 10.28 \\ 
LBE & 6.5 & 21.5 \\ 
$\text{Proxy-EnD}^2$ & 4.6 & 9.2 \\ 
DLF (Ours) & 4.2 & 11.28 \\ 
\bottomrule
\end{tabular}%
}
\end{table}

As shown in Table \ref{tab:training_time}, our proposed DLF model also exhibits a relatively short training time compared to small ensemble, Hydra, and $\text{Proxy-EnD}^2$.
In the DLF model, $\mu_{\theta}$ and $\Phi_{\theta}$ are local networks whose output sizes depend only on the target and latent dimensions; hence, they do not scale with the student model size.
Moreover, a key advantage of DLF is that it can  be effectively trained using EM algorithm.
As detailed in Appendix~\ref{Appendix:Estimation of the mean and factor loading}, the E-step has a closed-form solution, making the overall training process computationally efficient.
In contrast, LBE requires more computational resources and longer training time as the model becomes larger, which will be discussed further in the following subsection.

\subsection{Floating-point operations}
For a further analyze computational complexity, we compute Floating-point operations (FLOPs) per training for Hydra, Batch Ensemble and DLF.
FLOPs provide a hardware-independent measure of computational cost. For example, ResNet~\cite{he2016deep} and EfficientNet \cite{tan2019efficientnet} use FLOPs to evaluate and compare model efficiency across architectures.

Let $c$ denote the number of classes, $B$ the batch size, and $n$ the number of ensemble members.
The latent dimension in DLF is denoted by $q$.
$F_{\text{body}}$ and $F_{\text{head}}$ represent the FLOPs for one forward pass through the share body (e.g., Wide-ResNet \cite{zagoruyko2016wide}) and one Hydra head, respectively.
In the DLF model, $F_{\mu, \Phi}$ corresponds to the FLOPs for the small fully connected layers that produce $\mu_{\theta}$ and $\Phi_{\theta}$.
Finally, $\alpha$ indicates the multiplier accounting for both forward and backward passes.
Then, the FLOPs per training step for Hydra, Batch Ensemble, and DLF are  formulated as follows:
\begin{itemize}
    \item \textbf{Hydra:} 
    $$FLOPs_{\text{Hydra}} = \alpha B (F_{\text{body}} + n F_{\text{head}})$$
    \item \textbf{Batch Ensemble : }
    $$FLOPs_{\text{BE}} \approx \alpha B (1 + \epsilon ) F_{\text{body}}$$
    where $\epsilon$ (extra cost from rank-1 matrix) is usually $\leq 0.05$.
    \vspace{1mm}
    \item \textbf{DLF:}
    $$FLOPs_{\text{DLF}} = \alpha B (F_{\text{body}} + F_{\mu, \Phi}) + B q^2 + q^3 + n(Bcq + q^2) $$
\end{itemize}

Note that when the shared backbone is a Wide-ResNet, the corresponding $F_{\text{body}}$ is typically on the order of $10^9$ FLOPs. In Batch Ensemble, the additional cost comes from the overhead factor $(1 + \epsilon)$ applied to the full backbone. 
Although $\epsilon$ is usually small (typically less than 0.05), its effect is not negligible due to the large scale of $F_{\text{body}}$.

In contrast, Hydra shares the backbone and only adds a small cost from the lightweight head modules. 
Since $F_{\text{head}} \ll F_{\text{body}}$, the total complexity remains dominated by the shared body, even with large ensemble sizes $n$.

Similarly, although DLF requires additional computations such as $F_{\mu, \Phi}$, and matrix operations such as $Bq^2$, $q^3$, and $n(Bcq + q^2)$, 
their contributions are negligible since $q$ is typically small (often less than 20 in practice) and $F_{\mu, \Phi}, Bq^2, q^3, n(Bcq + q^2) \ll F_{\text{body}}$. 
Consequently, the dominant cost in DLF also comes from the shared backbone.

\section{Theoretical results} \label{sec4}
 
We intend to similarly investigate the convergence of the sieve maximum likelihood estimation (MLE), as discussed by \citep{chae2023likelihood}, using the results of estimating smooth functions within the sparse deep neural network (DNN) function class proposed by \citep{schmidt2020nonparametric} and sieve MLE's convergence rates by \citep{wong1995probability}. We will investigate the convergence rate in terms of decaying rates of the eigenvalues.

\paragraph{Notation} For a natural number $m$, we define $[m]= \{ 1, 2, \ldots, m \}.$  For a $m\times q$-dimensional matrix  $\boldsymbol{A}$, we denote the spectral norm of the matrix $\boldsymbol{A}$ by $\|\boldsymbol{A}\|_2$ and the Frobenius norm by $\|\boldsymbol{A}\|_F$, that is, $\|\boldsymbol{A}\|_2=\sup_{\boldsymbol{z}\in\mathbb{R}^q:\|\boldsymbol{z}\|_2=1}\|\boldsymbol{A}\boldsymbol{z}\|_2$ and $\|\boldsymbol{A}\|_F=\sqrt{\text{Tr}(\boldsymbol{A}^\top\boldsymbol{A})}$. For a square matrix $\boldsymbol{A}$, let $\lambda_{\min}(\boldsymbol{A})$ denote the smallest eigenvalue of $\boldsymbol{A}$.  
For two positive sequences $(a_n)_{n\in \mathbb{N}}$ and $(b_n)_{n\in \mathbb{N}}$, we write $a_n\lesssim b_n$ or  $b_n\gtrsim a_n$, if there exists a positive constant $C>0$ such that $a_n\le Cb_n$ for any $n\in \mathbb{N}$. We write $a_n\asymp b_n$ if both $a_n\gtrsim b_n$ and $a_n\lesssim b_n$ hold.
For a vector-valued function $\boldsymbol{f}=(f_1,\dots, f_m)^\top$ defined on a domain $\mathcal{X}$, we denote the ``elementwise-maximum'' supremum norm  by $\|\boldsymbol{f}\|_\infty=\max_{1\le j\le m}\|f_j\|_\infty = \max_{1\le j\le m}\sup_{x\in\mathcal{X}}|f_j(x)|$.  The H\"older space of smoothness $\beta>0$ with domain $[0,1]^d$ and radius $K>0$ is defined as, letting $s$ be the smallest integer larger than or equal to $\beta-1$, 
	\begin{equation*}
	\cH^\beta_d(K):=\left\{f\in\cC_d^s:\|f\|_{\cH^\beta_d}\le K\right\},
	\end{equation*}
where $\cC_d^s$ denotes the set of $s$-times differentiable functions on $[0,1]^d$ and $\|\cdot\|_{\cH^\beta_d}$ denotes the {H\"older norm} defined by
    \begin{align*}
    \|f\|_{\cH^\beta(\cX)}
    &=\max_{(\alpha_1,\dots,\alpha_d)\in\mathbb{N}_0^{d}:\sum_{j=1}^d\alpha_j<\beta}\|\partial^{\alpha_1}\cdots\partial^{\alpha_d}f\|_{\infty}\\
    &\quad +
        \max_{(\alpha_1,\dots,\alpha_d)\in\mathbb{N}_0^{d}:\sum_{j=1}^d\alpha_j= s}\sup_{\boldsymbol{x}_1,\boldsymbol{x}_2\in [0,1]^d, \boldsymbol{x}_1\neq\boldsymbol{x}_2 }\frac{|\partial^{\alpha_1}\cdots\partial^{\alpha_d}f(\boldsymbol{x}_1)-\partial^{\alpha_1}\cdots\partial^{\alpha_d}f(\boldsymbol{x}_2)|}{\|\boldsymbol{x}_1-\boldsymbol{x}_2\|_\infty^{\beta- s}},
    \end{align*}
where $\mathbb{N}_0=\{0,1,2,\dots\}$. 
Let $\cF$ be a given class of functions defined on $\cX$.  A collection $\{f_i:i\in[N]\}$ is called a $\delta$-covering set of $\cF$ with respect to a certain norm $\|\cdot\|$ defined on $\cX$, if, for all $f\in\cF$, there exists $f_i$ in the collection such that $\|f-f_i\|\le\delta$. The cardinality of the minimal $\delta$-covering set is called the $\delta$-covering number of $\cF$ with respect to   the  norm $\|\cdot\|$ , and is denoted by $\cN(\delta, \cF, \|\cdot\|)$. A collection $\{(l_i,u_i):i\in[N]\}$ of pairs of functions with $l_i\le u_i$ is called a $\delta$-bracketing set of $\cF$ with respect to a norm $\|\cdot\|$ if, for all $f\in\cF$, there exists $(l_i,u_i)$ in the collection such that $l_i\le f\le u_i$ and $\|l_i-u_i\|<\delta$. The cardinality of the minimal $\delta$-bracketing set is called the $\delta$-bracketing number of $\cF$ with respect to  the  norm $\|\cdot\|$, and is denoted by $\cN_{[]}(\delta, \cF, \|\cdot\|)$. For two probability density functions $p_1$ and $p_2$, let us denote the Hellinger distance between them by $h(p_1,p_2)=[\int (p_1^{1/2}(\boldsymbol{x})-p_2^{1/2}(\boldsymbol{x}))^2d\boldsymbol{x}]^{1/2}$.

\subsection{Problem formulation}
Let $\tilde{f}_1,\dots, \tilde{f}_n$ be $n$ independent realizations of the Gaussian process  with mean function $\mu_{*}(\cdot)$ and covariance kernel $\Sigma_{*}(\cdot,\cdot)$. Suppose that we have $m$ many $d$-dimensional design points $\mathcal{D}=\{\boldsymbol{x}_1^{(d)},\ldots,\boldsymbol{x}_m^{(d)}\}$ (where we omit the superscript ``design'' unlike the main body of the manuscript for simplicity). We assume that without loss of generality,  $\boldsymbol{x}_i^{(d)}\in[0,1]^d$ for every $j\in[m]$ by appropriate normalization. Then we observe $\boldsymbol{f}_{i\mid \mathcal{D}}=
    (\tilde{f}_i(\boldsymbol{x}_1^{(d)})+v_{i1},\ldots, \tilde{f}_i(\boldsymbol{x}_m^{(d)})+v_{im})^\top, $ 
where $v_{im}$s are independent Gaussian random variables with mean 0 and variance $\sigma_*^2.$  Note that $\boldsymbol{f}_{i\mid \mathcal{D}}$ follows the multivariate normal distribution $p_{*}:=\mathcal{N}(\boldsymbol{\mu}_{* \mid \mathcal{D}}, \boldsymbol{\Sigma}_{* \mid \mathcal{D}})$, where 
	\begin{align*}
		\boldsymbol{\mu}_{* \mid \mathcal{D}} &= (\mu_{*}(\boldsymbol{x}_u^{(d)}))_{u \in[m]} \in \mathbb{R}^{m} \\
		\boldsymbol{\Sigma}_{* \mid \mathcal{D}} & = (\Sigma_{*}(\boldsymbol{x}_u^{(d)},\boldsymbol{x}_v^{(d)}))_{u\in[m],v\in[m]} + \sigma^2_* \mathbb{I}_m \in \mathbb{R}^{m \times m}. 
	\end{align*} 
    Let $\lambda_{*,j}$ and $\psi_{*,j}(\cdot)$ be the $j$-th eigenvalues and eigenfunctions of the kernel $\Sigma_{*}$, ordered by their magnitude and $\phi_{*,j}(\cdot) = \sqrt{\lambda_{j}}\psi_{*,j}(\cdot)$ be the scaled eigenfunctions. Then, the covariance matrix can be decomposed into the $q$-leading part and the low-rank part
	\begin{align*}
		\boldsymbol{\Sigma}_{* \mid \mathcal{D}} = & \boldsymbol{\Phi}_{* \mid \mathcal{D}} \boldsymbol{\Phi}_{* \mid \mathcal{D}}^{\top} + \boldsymbol{\Sigma}_{*>q \mid \mathcal{D}}  +\sigma_{*}^{2} \mathbb{I}_{m}
	\end{align*} 
    where $\boldsymbol{\Phi}_{* \mid \mathcal{D}}=(\boldsymbol{\phi}_{*}(\boldsymbol{x}_1^{(d)}),\dots, \boldsymbol{\phi}_{*}(\boldsymbol{x}_m^{(d)}))^\top \in \mathbb{R}^{m \times q}$
    with $\boldsymbol{\phi}_{*}(\boldsymbol{x})=(\phi_{*j}(\boldsymbol{x}))_{ j\in [q]}$
    and
    $\boldsymbol{\Sigma}_{*>q \mid \mathcal{D}}=\left(\sum_{j> q} \phi_{*j}(\boldsymbol{x}_u^{(d)})\phi_{*j}(\boldsymbol{x}_v^{(d)})\right)_{u\in[m],v\in[m]}.$ For the sake of notational simplicity, we consider the case where $\sigma_*^2$ is known.
The proof can be extended easily for the case of unknown $\sigma_*^2.$
    
    Our aim is to estimate $p^*$ based on observations $\boldsymbol{f}_{1\mid m},\ldots, \boldsymbol{f}_{n\mid m}$
    by modeling $\mu_*$ and $\boldsymbol{\phi}_*$ by a specially design DNN. 
    For given $\boldsymbol{\mu}\in \mathbb{R}^m$ and $\boldsymbol{\Sigma}\in \mathbb{R}^{m\times m},$
    let $p_{\boldsymbol{\mu},\boldsymbol{\Sigma}}$ be the density of the Gaussian
    distribution with mean $\boldsymbol{\mu}$ and covariance matrix $\boldsymbol{\Sigma}.$
    For give mean function $\mu_\theta(\cdot)$ and a vector of $q$ many scaled eigenfunctions $\boldsymbol{\phi}_\theta(\cdot)$
    parameterized by $\theta,$ we consider $p_{\boldsymbol{\mu}_{\theta \mid \mathcal{D}}, \boldsymbol{\Sigma}_{\theta \mid \mathcal{D}}},$
    where
	\begin{align*}
		\boldsymbol{\mu}_{\theta \mid \mathcal{D}} &= (\mu_{\theta}(\boldsymbol{x}_u^{(d)}))_{u \in[m]} \in \mathbb{R}^{m} \\
		\boldsymbol{\Sigma}_{\theta \mid \mathcal{D}} &= \boldsymbol{\Phi}_{\theta \mid \mathcal{D}} \boldsymbol{\Phi}_{\theta \mid \mathcal{D}}^{\top} +\sigma_*^2\mathbb{I}_{m} \in \mathbb{R}^{m \times m}
	\end{align*}
    with $\boldsymbol{\Phi}_{\theta \mid \mathcal{D}}=(\boldsymbol{\phi}_\theta(\boldsymbol{x}_1^{(d)}),\dots, \boldsymbol{\phi}_\theta(\boldsymbol{x}_m^{(d)}))^\top$.
    We model $(\mu_\theta(\cdot),\boldsymbol{\phi}_\theta(\cdot))$ by a DNN with $q+1$ many outputs
    and estimate $\theta$ by a (sieve) maximum likelihood estimator (MLE) that is defined as
    $\hat{\theta}=\operatorname{argmax}_{\theta \in \Theta_{n}} \ell_{n}\left(\theta \right),$ where    
	\begin{equation*}
        \ell_{n}\left(\theta \right) =-\frac{n}{2} \log |\boldsymbol{\Sigma}_{\theta\mid \mathcal{D}}|
        -\frac{1}{2} \sum_{i=1}^n (\boldsymbol{f}_{i\mid \mathcal{D}}-\mu_{\theta\mid \mathcal{D}})^{\top} \boldsymbol{\Sigma}_{\theta\mid \mathcal{D}}^{-1} (\boldsymbol{f}_{i \mid \mathcal{D}}-\boldsymbol{\mu}_{\theta\mid \mathcal{D}}).
	\end{equation*} 
       Here, the sieve $\Theta_n$ depends on the architecture of DNNs. 
       We will prove that the estimated Gaussian distribution converges to the true Gaussian distribution $p_*$
       in probability as $n\rightarrow \infty$ under regularity conditions while $\mathcal{D}$ is fixed,
       provided that the sieve $\Theta_n$ is selected appropriately.
 	\subsection{Results}\label{Theorical Results}

For the sieve $\Theta_n,$ we consider a set of parameters, whose elements are in $[-1,1]$ (following \citep{schmidt2020nonparametric}), of sparse DNNs with $L_n$ many layers, $r_n$ many nodes at each hidden layers, $S_n$ many nonzero elements and $q_n+1$ output nodes. When we would like to clarify such architectural choices, we will sometimes use the notation $\Theta_n=\Theta(L_n,r_n,S_n,q_n)$. For a DNN parameter $\theta\in\Theta_n$, we let $\boldsymbol{g}_{\theta}$ be the corresponding realized DNN function, but for a technical reason, the outputs of this function are truncated at $[-B,B],$ so that  $\boldsymbol{g}_{\theta}$ is a function from $\mathbb{R}^d$ to $[-B,B]^{q+1}$. We denote by $\mathcal{G}(\Theta_n)=\{\boldsymbol{g}_{\theta}:\theta\in\Theta_n\}$. Such sparse DNNs have been considered in many previous studies \citep[e.g.,][]{chae2023likelihood, schmidt2020nonparametric, kim2021fast} to investigate the asymptotic properties of DNNs.

Given the design points $\mathcal{D}=\{\boldsymbol{x}_1^{(d)},\ldots,\boldsymbol{x}_m^{(d)}\},$ for each DNN parameter $\theta\in \Theta_n$ with the realized DNN $\boldsymbol{g}_{\theta}=(g_{\theta,1},\dots, g_{\theta,q_n+1})^\top$, we define the $m$-dimensional vector $\boldsymbol{\mu}_{\theta\mid \mathcal{D}}=(g_{\theta,1}(\boldsymbol{x}_u^{(d)}))_{u\in [m]}$ and $m\times m$ symmetric matrix
$\boldsymbol{\Sigma}_{\theta\mid \mathcal{D}}=\boldsymbol{\Phi}_{\theta\mid \mathcal{D}}\boldsymbol{\Phi}_{\theta\mid \mathcal{D}}^\top+\sigma^2_* \mathbb{I}_{m}$, where $\boldsymbol{\Phi}_{\theta\mid \mathcal{D}}=(g_{\theta,j+1}(\boldsymbol{x}_u^{(d)}))_{u\in[m],j\in[q_n]}.$ For notational simplicity, we write $p_{\theta\mid \mathcal{D}}=p_{\boldsymbol{\mu}_{\theta\mid \mathcal{D}},\boldsymbol{\Sigma}_{\theta\mid \mathcal{D}}}$, the density of the Gaussian distribution with
mean $\mu_{\theta\mid \mathcal{D}}$ and covariance matrix $\Sigma_{\theta\mid \mathcal{D}}.$ We let the class of such Gaussian distributions $\mathcal{P}(\Theta_n;\mathcal{D})=\{ p_{\theta\mid \mathcal{D}}: \theta\in \Theta_n\}.$

	\begin{lemma}
		\label{lemma 4.1.}
    Let $\mathcal{D}$ be an arbitrary set of $m$ design points.
       There exists an absolute constant $C_1>0$ such that
		for any $\delta \in(0, C_1/q_n)$, the following holds
		\begin{equation*}
                \log \cN_{[]}\left(\delta,\mathcal{P}(\Theta_n;\mathcal{D}), h\right) \le \log \cN\left(\frac{\sigma_{*}^2\delta}{26\max\{2mq_nB,\sqrt{m}\sigma_{*}\}} , \mathcal{G}(\Theta_n),\|\cdot \|_{\infty}\right).
		\end{equation*}
	\end{lemma}
	
\begin{theorem}
\label{theorem 4.2.}
Suppose that $\mu_*$ and $\phi_{*,j}, j=1,\ldots$ belong to $\mathcal{H}_d^\beta (K).$ Consider  the sieve MLE $\hat{p}=p_{\hat\theta}$ over $\Theta_n=\Theta(L_n, r_n, S_n, q_n)$ with $L_n\asymp \log n$ and $r_n, S_n, q_n\lesssim n$. Define 
	\begin{equation*}
			(\epsilon^{*}_{n})^2 = q_n\left(\frac{S_n}{\log n}\right)^{-2\beta/d} + \sum_{j>q_n} \lambda_{j}^{2} + S_n\frac{(\log n)^2 }{n} .
		\end{equation*}  
Assume that $q_n\to\infty$,  $\epsilon^{*}_{n}q_n\to0$ and $n(\epsilon^{*}_{n})^2\to\infty$ as $n\to\infty$. Then, we have
		\begin{equation*}
			P_*\left(h\left(\hat{p}, p_*\right)>C_{2} \epsilon_n^*\right)\to0
		\end{equation*}
	 as $n\rightarrow \infty$ for some absolute constant $C_2>0$.
	\end{theorem}

We can make $\epsilon_n^*$ converge to 0
by letting $q_n$ and $S_n$  diverge with a appropriate speed provided the eigenvalues $\lambda_j, j \ge q_n$
converge to 0 sufficiently fast (e.g. $\lambda_{j}\asymp \exp(-j)$).

The upper bound of Theorem \ref{theorem 4.2.} is about the estimated Gaussian distribution at the design points $\mathcal{D}^{\mathrm{design}}$. For prediction, we need an upper bound of the estimated Gaussian distribution at a new point $\boldsymbol{x}$. The following theorem, whose proof is given in Appendix \ref{Proof: thm4.2}, provides an upper bound.

\begin{theorem}[Upper bound at a new input]
   \label{theory-new}
   If $\hat\mu$ and $\hat\Phi_j, j=1,\ldots,q_n$ are Lipschitz, the probability of
    \begin{equation*}
        d_1(\hat{p}_{\boldsymbol{x}}, p_{*,\boldsymbol{x}}) \le 
     d_1(\hat{p}_{\boldsymbol{x}_{(1)}}, p_{*,\boldsymbol{x}_{(1)}})
           + C_3 \|\boldsymbol{x}-\boldsymbol{x}_{(1)}\|
    \end{equation*}
      for a certain positive constant $C_3$ and 
    \begin{equation*}
        \sup_{\boldsymbol{x}\in \mathcal{D}^{\mathrm{design}}} d_1(\hat{p}_{\boldsymbol{x}}, p_{*,\boldsymbol{x}}) \le C_2 \epsilon_n^*
    \end{equation*}
    converges to 1  as $n\rightarrow \infty,$
    where $d_1(g,h)=\int_{z} |g(z)-h(z)|dz$ for given two probability densities on $\mathbb{R}.$
\end{theorem}    
	
\subsection{Auxiliary Lemmas}

Before proving Lemma \ref{lemma 4.1.}, we introduce the following two lemmas.
\begin{lemma}
\label{lemma A.1.}
If $\boldsymbol{\Sigma}_{2}-\boldsymbol{\Sigma}_{1}$ is positive definite, then
\begin{equation*}
    \frac{p_{\boldsymbol{\mu}_1, \boldsymbol{\Sigma}_1}(x)}{p_{\boldsymbol{\mu}_2, \boldsymbol{\Sigma}_2}(x)}
    \leq \sqrt{\frac{| \boldsymbol{\Sigma}_2|}{| \boldsymbol{\Sigma}_1|}}\exp\left(\frac{1}{2}(\boldsymbol{\mu}_2-\boldsymbol{\mu}_1)^\top(\boldsymbol{\Sigma}_2-\boldsymbol{\Sigma}_1)^{-1} (\boldsymbol{\mu}_2-\boldsymbol{\mu}_1)\right)
\end{equation*}
\end{lemma}

\begin{proof}
Define $\boldsymbol{\mu}_* =(\boldsymbol{\Sigma}_1^{-1}-\boldsymbol{\Sigma}_2^{-1})^{-1}(\boldsymbol{\Sigma}_1^{-1}\boldsymbol{\mu}_1-\boldsymbol{\Sigma}_2^{-1}\boldsymbol{\mu}_2) $. Note that by assumption $\boldsymbol{\Sigma}_2-\boldsymbol{\Sigma}_1$ is invertible, and thus we have
    \begin{equation*}
    \begin{split}
        &(x-\boldsymbol{\mu}_1)^\top\boldsymbol{\Sigma}_1^{-1} (x-\boldsymbol{\mu}_1)
        - (x-\boldsymbol{\mu}_2)^\top\boldsymbol{\Sigma}_2^{-1} (x-\boldsymbol{\mu}_2)\\
        &=(x-\boldsymbol{\mu}_*)^\top(\boldsymbol{\Sigma}_1^{-1}-\boldsymbol{\Sigma}_2^{-1})(x-\boldsymbol{\mu}_*)
        -\boldsymbol{\mu}_*^\top(\boldsymbol{\Sigma}_1^{-1}-\boldsymbol{\Sigma}_2^{-1})\boldsymbol{\mu}_*
        +\boldsymbol{\mu}_1^\top\boldsymbol{\Sigma}_1^{-1}\boldsymbol{\mu}_1-\boldsymbol{\mu}_2^\top\boldsymbol{\Sigma}_2^{-1}\boldsymbol{\mu}_2.
    \end{split}
    \end{equation*}
The sum of the second and third terms is further simplified  as
\begin{equation*}
    \begin{split}
    -\boldsymbol{\mu}_*^{\top} &(\boldsymbol{\Sigma}_1^{-1}-\boldsymbol{\Sigma}_2^{-1})\boldsymbol{\mu}_* +\boldsymbol{\mu}_1^\top\boldsymbol{\Sigma}_1^{-1}\boldsymbol{\mu}_1-\boldsymbol{\mu}_2^\top\boldsymbol{\Sigma}_2^{-1}\boldsymbol{\mu}_2  \\
    -\mu_*^{\top}&\left(\boldsymbol{\Sigma}_1^{-1}-\boldsymbol{\Sigma}_2^{-1}\right) \boldsymbol{\mu}_*+\boldsymbol{\mu}_1^{\top} \boldsymbol{\Sigma}_1^{-1} \boldsymbol{\mu}_1-\boldsymbol{\mu}_2^{\top} \boldsymbol{\Sigma}_2^{-1} \boldsymbol{\mu}_2 \\
    = &-\left(\boldsymbol{\Sigma}_1^{-1} \boldsymbol{\mu}_1-\boldsymbol{\Sigma}_2^{-1} \boldsymbol{\mu}_2\right)^{\top}\left(\boldsymbol{\Sigma}_1^{-1}-\boldsymbol{\Sigma}_2^{-1}\right)^{-1}\left(\boldsymbol{\Sigma}_1^{-1} \boldsymbol{\mu}_1-\boldsymbol{\Sigma}_2^{-1} \boldsymbol{\mu}_2\right)+\boldsymbol{\mu}_1^{\top} \boldsymbol{\Sigma}_1^{-1} \boldsymbol{\mu}_1-\boldsymbol{\mu}_2^{\top} \boldsymbol{\Sigma}_2^{-1} \boldsymbol{\mu}_2 \\
    = &-\boldsymbol{\mu}_1^{\top} \boldsymbol{\Sigma}_1^{-1}\left(\mathbf{I}-\boldsymbol{\Sigma}_1 \boldsymbol{\Sigma}_2^{-1}\right)^{-1} \boldsymbol{\mu}_1-\boldsymbol{\mu}_2^{\top} \boldsymbol{\Sigma}_2^{-1}\left(\boldsymbol{\Sigma}_2 \boldsymbol{\Sigma}_1^{-1}-\mathbf{I}\right)^{-1} \boldsymbol{\mu}_2\\
    &+2 \boldsymbol{\mu}_2^{\top} \boldsymbol{\Sigma}_2^{-1}\left(\boldsymbol{\Sigma}_1^{-1}-\boldsymbol{\Sigma}_2^{-1}\right) \boldsymbol{\Sigma}_1^{-1} \boldsymbol{\mu}_1 + \boldsymbol{\mu}_1^{\top} \boldsymbol{\Sigma}_1^{-1} \boldsymbol{\mu}_1-\boldsymbol{\mu}_2^{\top} \boldsymbol{\Sigma}_2^{-1} \boldsymbol{\mu}_2 \\
    = &-\boldsymbol{\mu}_1^{\top} \boldsymbol{\Sigma}_1^{-1} \boldsymbol{\Sigma}_1 \boldsymbol{\Sigma}_2^{-1}\left(\mathbf{I}-\boldsymbol{\Sigma}_1 \boldsymbol{\Sigma}_2^{-1}\right)^{-1} \boldsymbol{\mu}_1-\boldsymbol{\mu}_2^{\top} \boldsymbol{\Sigma}_2^{-1} \boldsymbol{\Sigma}_2 \boldsymbol{\Sigma}_1^{-1}\left(\boldsymbol{\Sigma}_2 \boldsymbol{\Sigma}_1^{-1}-\mathbf{I}\right)^{-1} \boldsymbol{\mu}_2\\
    & +2 \boldsymbol{\mu}_2^{\top} \boldsymbol{\Sigma}_2^{-1}\left(\boldsymbol{\Sigma}_1^{-1}-\boldsymbol{\Sigma}_2^{-1}\right)^{-1} \boldsymbol{\Sigma}_1^{-1} \boldsymbol{\mu}_1 \\
    = &-\boldsymbol{\mu}_1^{\top}\left(\boldsymbol{\Sigma}_2-\boldsymbol{\Sigma}_1\right)^{-1} \boldsymbol{\mu}_1-\boldsymbol{\mu}_2^{\top}\left(\boldsymbol{\Sigma}_2-\boldsymbol{\Sigma}_1\right)^{-1} \boldsymbol{\mu}_2+2 \boldsymbol{\mu}_2^{\top}\left(\boldsymbol{\Sigma}_2^{-1}-\boldsymbol{\Sigma}_1^{-1}\right)^{-1} \boldsymbol{\mu}_1 \\
    = &-\left(\boldsymbol{\mu}_1-\boldsymbol{\mu}_2\right)^{\top}\left(\boldsymbol{\Sigma}_2-\boldsymbol{\Sigma}_1\right)^{-1}\left(\boldsymbol{\mu}_1-\boldsymbol{\mu}_2\right) .
    \end{split}
\end{equation*}
Therefore, we have
\begin{equation*}
    \begin{split}
        \frac{p_{\boldsymbol{\mu}_1, \boldsymbol{\Sigma}_1}(x)}{p_{\boldsymbol{\mu}_2, \boldsymbol{\Sigma}_2}(x)}
        =    \sqrt{\frac{| \boldsymbol{\Sigma}_2|}{| \boldsymbol{\Sigma}_1|}}\exp\biggl(\biggr.&\biggl.-\frac{1}{2}\{(x-\boldsymbol{\mu}_1)^\top\boldsymbol{\Sigma}_1^{-1} (x-\boldsymbol{\mu}_1)
        - (x-\boldsymbol{\mu}_2)^\top\boldsymbol{\Sigma}_2^{-1} (x-\boldsymbol{\mu}_2)\} \biggr)\\
        = \sqrt{\frac{| \boldsymbol{\Sigma}_2|}{| \boldsymbol{\Sigma}_1|}}\exp\biggl(\biggr.&\biggl.-\frac{1}{2}(x-\boldsymbol{\mu}_*)^\top(\boldsymbol{\Sigma}_1^{-1}-\boldsymbol{\Sigma}_2^{-1})(x-\boldsymbol{\mu}_*) \biggr.\\
       & + \left.\frac{1}{2}(\boldsymbol{\mu}_1-\boldsymbol{\mu}_2)^\top(\boldsymbol{\Sigma}_2-\boldsymbol{\Sigma}_1)^{-1}(\boldsymbol{\mu}_1-\boldsymbol{\mu}_2)\right)\\
       \le  \sqrt{\frac{| \boldsymbol{\Sigma}_2|}{| \boldsymbol{\Sigma}_1|}}\exp\biggl(\biggr.&\biggl.\frac{1}{2}(\boldsymbol{\mu}_1-\boldsymbol{\mu}_2)^\top(\boldsymbol{\Sigma}_2-\boldsymbol{\Sigma}_1)^{-1}(\boldsymbol{\mu}_1-\boldsymbol{\mu}_2) \biggr),
    \end{split}
\end{equation*}
which completes the proof.
\end{proof}

\begin{lemma}
\label{lemma A.2.}
Let $\sigma^2$ be the lower bound of the minimum eigenvalues of $\boldsymbol{\Sigma}_1$ and $\boldsymbol{\Sigma}_2.$
If $\|\boldsymbol{\Sigma}_2-\boldsymbol{\Sigma}_1\|_2\le c \sigma^2$ for a given $c>0,$  
the  following inequalities hold:
\begin{equation*}
    \begin{split}
        x^\top((1+\zeta)\boldsymbol{\Sigma}_2-\boldsymbol{\Sigma}_1)x\ge\sigma^2(\zeta-(1+\zeta)c)\|x\|^2
    \end{split}
\end{equation*}
and
\begin{equation*}
    \begin{split}
        x^\top(\boldsymbol{\Sigma}_1-(1+\zeta)^{-1}\boldsymbol{\Sigma}_2)x\ge
        \frac{(\zeta-c)\sigma^2}{1+\zeta}\|x\|^2,
    \end{split}
\end{equation*}
\end{lemma}
where $\zeta =3\sigma^{2}c.$
\begin{proof}
The first inequality holds because
\begin{equation*}
    \begin{split}
        x^\top((1+\zeta)\boldsymbol{\Sigma}_2-\boldsymbol{\Sigma}_1)x
        &\ge (1+\zeta) x^\top(\boldsymbol{\Sigma}_2-\boldsymbol{\Sigma}_1)x + \zeta x^\top\boldsymbol{\Sigma}_1x\\
        &\ge -(1+\zeta) \|\boldsymbol{\Sigma}_2-\boldsymbol{\Sigma}_1\|_2\|x\|_2^2 + \zeta\sigma^2\|x\|^2\\
        &\ge(\zeta\sigma^2-(1+\zeta)c\sigma^2)\|x\|^2.
    \end{split}
\end{equation*}
The second inequality follows similarly.
\end{proof}

\subsection{Proofs}\label{Proofs of Results}

\paragraph*{Proof of Lemma \ref{lemma 4.1.}} 
\begin{proof} 
Fix $\epsilon > 0$. Let $\{\boldsymbol{g}_{1}, \ldots, \boldsymbol{g}_{N}\}$ with $N=\mathcal{N}(\epsilon,\mathcal{G}(\Theta_n),\|\cdot\|_\infty)$ be a $\epsilon$-covering of $\mathcal{G}(\Theta_n)$. For each $i\in[N]$, let $\theta_i$ be the parameter of $\boldsymbol{g}_{i}$ and let $\boldsymbol{\mu}_{i}=\boldsymbol{\mu}_{\theta_i|\mathcal{D}}$ and $\boldsymbol{\Sigma}_{i}=\boldsymbol{\Sigma}_{\theta_i|\mathcal{D}}$.
Then for any $\theta\in \Theta_n$, letting  $\boldsymbol{\mu}=\boldsymbol{\mu}_{\theta|\mathcal{D}}$ and 
 $\boldsymbol{\Sigma}=\boldsymbol{\Sigma}_{\theta|\mathcal{D}}$ for simplicity, we have 
\begin{equation*}
    \begin{split}
        \|\boldsymbol{\mu}-\boldsymbol{\mu}_i\|_2 & \leq \sqrt{m}\epsilon,
    \end{split}
\end{equation*}
and
    \begin{align*}
         \|\boldsymbol{\Sigma}-\boldsymbol{\Sigma}_{i}\|_2 & \leq \|(\boldsymbol{\Phi}\boldsymbol{\Phi}^{\top} - \boldsymbol{\Phi}_{i}\boldsymbol{\Phi}_{i}^{\top}) \|_2 \\
        & \leq  (\|\boldsymbol{\Phi}\|_{2} + \|\boldsymbol{\Phi}_i\|_{2}) \|\boldsymbol{\Phi} - \boldsymbol{\Phi}_{i}\|_F \\
        & \leq 2mq_nB\epsilon,
    \end{align*}
where the third line follows from that $\|\boldsymbol{\Phi}\|_2,\|\boldsymbol{\Phi}_i\|_2\le \sqrt{mq_n}B $. 
Now, let $\delta=\epsilon\max\{2mq_nB,\sqrt{m}\sigma_{*}\}/\sigma_{*}^2$ so that
$\|\boldsymbol{\mu}-\boldsymbol{\mu}_i\|_2\le \sigma_{*}\delta$ and $\|\boldsymbol{\Sigma}-\boldsymbol{\Sigma}_{i}\|_2\le \sigma_{*}^2\delta$. Let
$\zeta=3\delta$. Then we will show that $[l_{i}, u_{i}] $ is a Hellinger bracket of the density $p_{\boldsymbol{\mu}, \boldsymbol{\Sigma}}(x)$ when we define
\begin{equation*}
    \begin{split}
        u_{i} &= (1+2\zeta)^m p_{\boldsymbol{\mu}_i, (1+\zeta)\boldsymbol{\Sigma}_{i}}\\
        l_{i} &= (1+2\zeta)^{-m} p_{\boldsymbol{\mu}_i, (1+\zeta)^{-1}\boldsymbol{\Sigma}_{i}},
    \end{split}
\end{equation*}
Then by Lemma \ref{lemma A.2.}, $(1+\zeta)\boldsymbol{\Sigma}_{i}-\boldsymbol{\Sigma}$ and $\boldsymbol{\Sigma}-(1+\zeta)^{-1}\boldsymbol{\Sigma}_{i}$ are both positive definite.
So by Lemma \ref{lemma A.1.}, we have for any $\boldsymbol{x}\in \mathbb{R}^m$
\begin{equation*}
    \begin{split}
        \frac{p_{\boldsymbol{\mu}, \boldsymbol{\Sigma}}(\boldsymbol{x})}{u_{i}(\boldsymbol{x})}
        \le (1+2\zeta)^{-m} \sqrt{\frac{| (1+\zeta)\boldsymbol{\Sigma}_{i}|}{| \boldsymbol{\Sigma}|}}\exp\left(\frac{1}{2}(\boldsymbol{\mu}-\boldsymbol{\mu}_i)^\top((1+\zeta)\boldsymbol{\Sigma}_{i}-\boldsymbol{\Sigma})^{-1} (\boldsymbol{\mu}_i-\boldsymbol{\mu})\right).
    \end{split}
\end{equation*}
By Lemma \ref{lemma A.2.} again, we have $\|((1+\zeta)\boldsymbol{\Sigma}_{i}-\boldsymbol{\Sigma})^{-1}\|_2\le(\sigma_{*}^2(\zeta-(1+\zeta)\delta))^{-1}=(\sigma_{*}^2(2-\zeta)\delta)^{-1}\le (\sigma_{*}^2\delta)^{-1}$ for any sufficiently small $\epsilon$. Moreover, by Weyl's inequality,
\begin{equation*}
    \begin{split}
        \frac{| (1+\zeta)\boldsymbol{\Sigma}_{i}|}{| \boldsymbol{\Sigma}|}
        \le (1+\zeta)^m \left(1+\frac{\sigma_{*}^2\delta}{\sigma_{*}^2}\right)^m\le (1+2\zeta)^m.
    \end{split}
\end{equation*}
Thus we have
    \begin{align*}
          \frac{p_{\boldsymbol{\mu}, \boldsymbol{\Sigma}}(\boldsymbol{x})}{u_{i}(\boldsymbol{x})}
          =(1+2\zeta)^{-m/2}\exp\left(\frac{\|\boldsymbol{\mu}-\boldsymbol{\mu}_i\|_2^2}{2\sigma_{*}^2\delta}\right).
    \end{align*}
Using the inequality $\log (1+z)\ge z/2$ for $z\in[0,2]$, we have
\begin{equation*}
    \begin{split}
          \log\frac{p_{\boldsymbol{\mu}, \boldsymbol{\Sigma}}(\boldsymbol{x})}{u_{i}(\boldsymbol{x})}
         &= -\frac{m}{2}\log(1+2\zeta)+\frac{1}{2\sigma_{*}^2\delta}\|\boldsymbol{\mu}-\boldsymbol{\mu}_i\|_2^2\\
         &\le  -\frac{m}{2}\zeta+\frac{m}{2\sigma_{*}^2\delta}(\sigma_{*}\delta)^2\\
         &\le \left(-\frac{3m}{2} + \frac{m}{2}\right)\delta \leq 0,
    \end{split}
\end{equation*}
which implies $p_{\boldsymbol{\mu}, \boldsymbol{\Sigma}}(\boldsymbol{x}) \leq u_{i}(\boldsymbol{x})$.
Similarly, we also have
    \begin{align*}
          \frac{l_{i}(\boldsymbol{x})}{p_{\boldsymbol{\mu}, \boldsymbol{\Sigma}}(\boldsymbol{x})}
         = (1+2\zeta)^{-m/2}\exp\left(\frac{\|\boldsymbol{\mu}-\boldsymbol{\mu}_i\|_2^2}{2\sigma_*^2\delta}\right).
    \end{align*}
and so $l_{i}(\boldsymbol{x})\le p_{\boldsymbol{\mu}, \boldsymbol{\Sigma}}(\boldsymbol{x})$ for any $\boldsymbol{x}\in\mathbb{R}^m$.

We now bound the size of the bracket. Note that
    \begin{align*}
        h^2(l_{i}, u_{i})
        &=(1+2\zeta)^m+(1+2\zeta)^{-m} - (2-h^2(p_{\boldsymbol{\mu_{i}}, (1+\zeta)\boldsymbol{\Sigma}_{i}}, p_{\boldsymbol{\mu_{i}}, (1+\zeta)^{-1}\boldsymbol{\Sigma}_{i}})).
    \end{align*}
Due to the inequality $z^2/2\ge z-\log(1+z)$ for any $z\ge0$,
\begin{equation*}
    \begin{split}
        h^2(p_{\boldsymbol{\mu_{i}}, (1+\zeta)\boldsymbol{\Sigma}_{i}}, p_{\boldsymbol{\mu_{i}}, (1+\zeta)^{-1}\boldsymbol{\Sigma}_{i}})
        &\le \textrm{KL}(p_{\boldsymbol{\mu_{i}}, (1+\zeta)\boldsymbol{\Sigma}_{i}}, p_{\boldsymbol{\mu_{i}}, (1+\zeta)^{-1}\boldsymbol{\Sigma}_{i}})\\
        &=\frac{1}{4}m(-\log(1+\zeta)^2 + (1+\zeta)^2-1)\\
        &\le \frac{m}{8}((1+\zeta)^2-1)^2\\
        &=\frac{m}{4}(\zeta +\zeta^2/2)^2 \\
        &\le \frac{9}{4}m\zeta^2 
    \end{split}
\end{equation*}
Moreover, by taking $\epsilon$ sufficiently small so that $\zeta<3/m$, we have
    \begin{align*}
        (1+2\zeta)^m+(1+2\zeta)^{-m}-2
       & \le 2(1+2\zeta)^m-2\\
       & \le 4m\zeta(1+2\zeta)^{m-1}\\
       &\leq 4m\zeta(1+2/(3m))^{m-1} \le 4e^{2/3}\zeta 
    \end{align*}
Thus, we have $h^2(l_{i}, u_{i})\le (9/4m\zeta +4e^{2/3})\zeta\le(3/4 +4e^{2/3})\zeta \le 26\delta$. Hence, redefining constant as $26\delta\to\delta$, we complete the proof.
\end{proof}

\paragraph*{Proof of Theorem \ref{theorem 4.2.}}\label{Proof: thm4.2}
\begin{proof}
The proof follows a similar reasoning in the proof of Theorem 3 in \citep{chae2023likelihood}, which is based on Theorem 4 in \citep{wong1995probability} with $\alpha = 0+$. We divide the proof into the following four steps.

\textbf{Bounding the estimation error: Check Eq. (3.1) of \citep{wong1995probability}}
For the class of DNN parameters $\Theta_n=\Theta(L_n,r_n,S_n,q_n)$, by Lemma 5 in \citep{schmidt2020nonparametric}, we can get the following covering number bound
    \begin{align*}
            \log \mathcal{N}(\delta, \mathcal{G}(\Theta_n), \left\| \cdot \right\|_{\infty}) 
        &\leq (S_n+1) \log \left(2 \delta^{-1}(L_n+1) d^2 (q_n+1)^2 r_n^{2L_n}\right) \\
        &\lesssim  L_n S_n \log (n\delta^{-1})
    \end{align*}
for any $\delta>0$. Applying Lemma \ref{lemma 4.1.}, for $0<\delta<C_1/q_n $, we have
    \begin{align*}
    \log \mathcal{N}_{[]}(\delta, \mathcal{P}(\Theta_n;\mathcal{D}), h) 
    &\leq \log \mathcal{N}\left(\frac{\sigma_{*}^2\delta}{26\max\{2mq_nB,\sqrt{m}\sigma_{*}\}} , \mathcal{G}(\Theta_n),\|\cdot \|_{\infty}\right)\\  
    &\lesssim S_nL_n \log(n\delta^{-1})
    \end{align*}
Moreover, for a positive constant $\epsilon$ such that $\sqrt{2} \epsilon\le C_1/q_n$, we have
\begin{equation*}
    \begin{split}
    \int_{\epsilon^2 / 2^8}^{\sqrt{2} \epsilon} \sqrt{ \log \mathcal{N}_{[]}(\delta, \mathcal{P}(\Theta_n;\mathcal{D}), h) } d \delta 
    & \lesssim \epsilon \sqrt{ S_nL_n \log(n\epsilon^{-1})}.
    \end{split}
\end{equation*} 
Then the above display is bounded by $n^{1/2}\epsilon^2$ up to an absolute constant when we take $\epsilon=\epsilon_n=\sqrt{ S_n(\log n)^2/n}$ as $L_n\asymp \log n$. Thus, Eq. (3.1) of \citep{wong1995probability} is satisfied.

\textbf{Bounding the Kullback-Liebler approximation error} We first note that
 for any $\theta\in \Theta_n$, 
 we have
\begin{align}
    \operatorname{KL}\left(p_{*} || p_{\boldsymbol{\mu}_{\theta\mid \mathcal{D}}, \boldsymbol{\Sigma}_{\theta\mid \mathcal{D}}}\right)
    &= \frac{1}{2}\left(-\log {\left|\boldsymbol{\Sigma}_{\theta\mid \mathcal{D}}\boldsymbol{\Sigma}^{-1}_{*\mid \mathcal{D}}\right|}+\operatorname{Tr}\left(\boldsymbol{\Sigma}_{\theta \mid \mathcal{D}}\boldsymbol{\Sigma}_{*\mid \mathcal{D}}^{-1}-\mathbb{I}_m \right) \right)\nonumber\\
    &\quad+\frac{1}{2}\left(\boldsymbol{\mu}_{\theta\mid \mathcal{D}}-\boldsymbol{\mu}_{*\mid \mathcal{D}}\right)^{\top} \boldsymbol{\Sigma}_{*\mid \mathcal{D}}^{-1}\left(\boldsymbol{\mu}_{\theta\mid \mathcal{D}}-\boldsymbol{\mu}_{*\mid \mathcal{D}}\right) \nonumber\\
     &= -\frac{1}{2}\log {\left|\mathbb{I}_m+\boldsymbol{\Sigma}_{*\mid \mathcal{D}}^{-1/2}(\boldsymbol{\Sigma}_{\theta \mid \mathcal{D}}-\boldsymbol{\Sigma}_{*\mid \mathcal{D}})\boldsymbol{\Sigma}_{*\mid \mathcal{D}}^{-1/2}\right|}\nonumber\\
     &\quad +\frac{1}{2}\operatorname{Tr}\left(\boldsymbol{\Sigma}_{*\mid \mathcal{D}}^{-1/2}(\boldsymbol{\Sigma}_{\theta \mid \mathcal{D}}-\boldsymbol{\Sigma}_{*\mid \mathcal{D}})\boldsymbol{\Sigma}_{*\mid \mathcal{D}}^{-1/2}\right) 
     +\frac{1}{2\sigma_*^2}\|\boldsymbol{\mu}_{\theta\mid \mathcal{D}}-\boldsymbol{\mu}_{*\mid \mathcal{D}}\|^2 \nonumber\\
    &\le \frac{1}{4}\|\mathbb{I}_m-\boldsymbol{\Sigma}_{\theta \mid \mathcal{D}}\boldsymbol{\Sigma}_{*\mid \mathcal{D}}^{-1}\|_F^2
        +\frac{1}{2\sigma_*^2}\|\boldsymbol{\mu}_{\theta\mid \mathcal{D}}-\boldsymbol{\mu}_{*\mid \mathcal{D}}\|^2, \nonumber\\
    &\le \frac{1}{4\sigma_*^2}\|\boldsymbol{\Sigma}_{\theta \mid \mathcal{D}}-\boldsymbol{\Sigma}_{*\mid \mathcal{D}}\|_F^2
        +\frac{1}{2\sigma_*^2}\|\boldsymbol{\mu}_{\theta\mid \mathcal{D}}-\boldsymbol{\mu}_{*\mid \mathcal{D}}\|^2, \label{approx_terms}
\end{align}
To bound the two terms in Eq. (\ref{approx_terms}), we use the well-known results about the approximation ability of sparse DNNs to H\"older smooth functions \citep[e.g., Theorem 5 of][]{schmidt2020nonparametric}. Namely, there exists $\theta^\dag\in\Theta_n$ such that 
    \begin{align}
    \label{dnn_approx}
       \max\left\{ \|g_{\theta^\dag,1}-\mu_*\|,\max_{j\in[q_n]}\|g_{\theta^\dag,j+1}-\phi_{*,j}\|_\infty\right\}\lesssim (S_n/L_n)^{-\beta/d}.
    \end{align}
For the first term in Eq. (\ref{approx_terms}), we define
    \begin{align*}
       \kappa_n= \|\boldsymbol{\Sigma}_{*>q_n\mid\mathcal{D}} \|_F=\left(\sum_{j>q_n}\lambda_{j}^2\right)^{1/2}
    \end{align*}
and establish the upper bound
    \begin{align*}
           \|\boldsymbol{\Sigma}_{\theta^\dag \mid \mathcal{D}}-\boldsymbol{\Sigma}_{*\mid \mathcal{D}}\|_F
    &\le \|  \boldsymbol{\Phi}_{\theta^\dag \mid \mathcal{D}} \boldsymbol{\Phi}_{\theta^\dag \mid \mathcal{D}}^\top- \boldsymbol{\Phi}_{*\mid \mathcal{D}} \boldsymbol{\Phi}_{*\mid \mathcal{D}}^\top\|_F + \kappa_n\\
    &\le\|  (\boldsymbol{\Phi}_{\theta^\dag \mid \mathcal{D}} -\boldsymbol{\Phi}_{*\mid \mathcal{D}})(\boldsymbol{\Phi}_{\theta^\dag \mid \mathcal{D}} -\boldsymbol{\Phi}_{*\mid \mathcal{D}})^\top 
    +2\boldsymbol{\Phi}_{*\mid \mathcal{D}}(\boldsymbol{\Phi}_{\theta^\dag \mid \mathcal{D}}-\boldsymbol{\Phi}_{*\mid \mathcal{D}})^\top\|_F + \kappa_n\\
    &\le \| \boldsymbol{\Phi}_{\theta^\dag \mid \mathcal{D}} -\boldsymbol{\Phi}_{*\mid \mathcal{D}}\|_F^2 
    + 2\|\boldsymbol{\Phi}_{*\mid \mathcal{D}}\|_2
    \|  \boldsymbol{\Phi}_{\theta^\dag \mid \mathcal{D}} -\boldsymbol{\Phi}_{*\mid \mathcal{D}}\|_F + \kappa_n.
    \end{align*}
But by Eq. (\ref{dnn_approx}), $ \|  \boldsymbol{\Phi}_{\theta^\dag \mid \mathcal{D}} -\boldsymbol{\Phi}_{*\mid \mathcal{D}}\|_F\lesssim \sqrt{q_n} (S_n/L_n)^{-\beta/d}$ which converges to 0 as $n\to\infty$. This implies that, as $\|\boldsymbol{\Phi}_{*\mid \mathcal{D}}\|_2$ is bounded, $\| \boldsymbol{\Phi}_{\theta^\dag \mid \mathcal{D}} -\boldsymbol{\Phi}_{*\mid \mathcal{D}}\|_F^2 $ is smaller than $2\|\boldsymbol{\Phi}_{*\mid \mathcal{D}}\|_2
 \|  \boldsymbol{\Phi}_{\theta^\dag \mid \mathcal{D}} -\boldsymbol{\Phi}_{*\mid \mathcal{D}}\|_F$ eventually. Therefore, we have
   \begin{align*}
        \|\boldsymbol{\Sigma}_{\theta^\dag \mid \mathcal{D}}-\boldsymbol{\Sigma}_{*\mid \mathcal{D}}\|_F^2
   \lesssim q_n(S_n/L_n)^{-2\beta/d} + \kappa_n
    \end{align*}
Moreover, by Eq. (\ref{dnn_approx}), it is immediate that $\|\boldsymbol{\mu}_{\theta^\dag\mid \mathcal{D}}-\boldsymbol{\mu}_{*\mid \mathcal{D}}\|^2\lesssim  (S_n/L_n)^{-2\beta/d}$. Adopting the notation of \cite{wong1995probability}, we have $\delta_n=q_n(S_n/L_n)^{-2\beta/d}+\kappa_n\asymp q_n(S_n/\log n)^{-2\beta/d}+\kappa_n$.

\textbf{Bounding the Kullback-Liebler variation} The last ingredient of the proof is to bound the so-called Kullback-Liebler variation defined as
    \begin{align*}
        \operatorname{KLV}\left(p_{*} || p_{\boldsymbol{\mu}_{\theta\mid \mathcal{D}}, \boldsymbol{\Sigma}_{\theta\mid \mathcal{D}}}\right)
         =  \int \left\{\log \frac{p_*(\boldsymbol{x})}{p_{\boldsymbol{\mu}_{\theta\mid \mathcal{D}}, \boldsymbol{\Sigma}_{\theta\mid \mathcal{D}}}(\boldsymbol{x})}\right\}^2 p_*(\boldsymbol{x})d\boldsymbol{x}.
    \end{align*}
We will find a suitable network parameter to get a manageable upper bound of the above, which is denoted by $\tau_n$ in  \cite{wong1995probability}. We use Lemma \ref{lemma A.2.} for this purpose. As in the argument used in the previous step, we can find a network parameter $\theta^\dag$ such that $\|\boldsymbol{\Sigma}_{\theta^\dag \mid \mathcal{D}}-\boldsymbol{\Sigma}_{*\mid \mathcal{D}}\|_F^2\le C' \delta_n$ for some absolute constant $C'>0$. As $q_n\to\infty$, we have $\xi=\lambda_{\min}(\boldsymbol{\Phi}_{\theta^\dag \mid \mathcal{D}}\boldsymbol{\Phi}_{\theta^\dag \mid \mathcal{D}}^\top)>0$. We then construct the network $\theta^\ddag$ satisfying $\boldsymbol{\mu}_{\theta^\ddag \mid \mathcal{D}}=\boldsymbol{\mu}_{\theta^\dag \mid \mathcal{D}}$ and $\boldsymbol{\Phi}_{\theta^\ddag \mid \mathcal{D}}=(1+(1+C')\delta_n/\xi)^{1/2}\boldsymbol{\Phi}_{\theta^\dag \mid \mathcal{D}}$. Then, by Weyl's inequality,
    \begin{align*}
        \lambda_{\min}(\boldsymbol{\Sigma}_{\theta^\ddag \mid \mathcal{D}}-\boldsymbol{\Sigma}_{*\mid \mathcal{D}})
         &=\lambda_{\min}(\boldsymbol{\Sigma}_{\theta^\dag \mid \mathcal{D}}-\boldsymbol{\Sigma}_{*\mid \mathcal{D}}+(1+C')\delta_n/\xi\boldsymbol{\Phi}_{\theta^\dag \mid \mathcal{D}}\boldsymbol{\Phi}_{\theta^\dag \mid \mathcal{D}}^\top)\\
        &\ge \lambda_{\min}((1+C')\delta_n/\xi\boldsymbol{\Phi}_{\theta^\dag \mid \mathcal{D}}\boldsymbol{\Phi}_{\theta^\dag \mid \mathcal{D}}^\top)-\|\boldsymbol{\Sigma}_{\theta^\dag \mid \mathcal{D}}-\boldsymbol{\Sigma}_{*\mid \mathcal{D}}\|_2\\
        &\ge (1+C')\delta_n-\|\boldsymbol{\Sigma}_{\theta^\dag \mid \mathcal{D}}-\boldsymbol{\Sigma}_{*\mid \mathcal{D}}\|_F\ge \delta_n,
    \end{align*}
which implies that $\boldsymbol{\Sigma}_{\theta^\ddag \mid \mathcal{D}}-\boldsymbol{\Sigma}_{*\mid \mathcal{D}}$ is positive definite. Using Lemma \ref{lemma A.2.}, we have
    \begin{align*}
           \frac{p_{*}(x)}{p_{\boldsymbol{\mu}_{\theta^\ddag \mid \mathcal{D}},\boldsymbol{\Sigma}_{\theta^\ddag \mid \mathcal{D}}}(x)}
    &\leq  \sqrt{\frac{|\boldsymbol{\Sigma}_{\theta^\ddag \mid \mathcal{D}}|}{| \boldsymbol{\Sigma}_{*\mid \mathcal{D}}|}}
    \exp\left(\frac{1}{2}(\boldsymbol{\mu}_{\theta^\ddag \mid \mathcal{D}}-\boldsymbol{\mu}_{*\mid \mathcal{D}})^\top(\boldsymbol{\Sigma}_{\theta^\ddag \mid \mathcal{D}}-\boldsymbol{\Sigma}_{*\mid \mathcal{D}})^{-1} (\boldsymbol{\mu}_{\theta^\ddag \mid \mathcal{D}}-\boldsymbol{\mu}_{*\mid \mathcal{D}})\right)\\
    &\le \sqrt{\frac{|\boldsymbol{\Sigma}_{\theta^\ddag \mid \mathcal{D}}|}{| \boldsymbol{\Sigma}_{*\mid \mathcal{D}}|}}
    \exp\left(\frac{1}{2\delta_n}\|\boldsymbol{\mu}_{\theta^\ddag \mid \mathcal{D}}-\boldsymbol{\mu}_{*\mid \mathcal{D}}\|^2\right)\\
    &\le (1+\|\boldsymbol{\Sigma}_{\theta^\ddag \mid \mathcal{D}}- \boldsymbol{\Sigma}_{*\mid \mathcal{D}}\|_2/\sigma_*^2)^{m/2}\\
   &\le (1+\|(1+C')\delta_n/(\xi\sigma_*^2)\boldsymbol{\Phi}_{\theta^\dag \mid \mathcal{D}}\boldsymbol{\Phi}_{\theta^\dag \mid \mathcal{D}}^\top\|_F+\|\boldsymbol{\Sigma}_{\theta^\dag \mid \mathcal{D}}- \boldsymbol{\Sigma}_{*\mid \mathcal{D}}\|_2/\sigma_*^2)^{m/2}\\
   &\le (1+(1+C')(Bmq_n)^{1/2}\delta_n/(\xi\sigma_*^2)+(C')^{1/2}\delta_n/\sigma_*^2)^{m/2}
    \end{align*}
where we use Weyl's inequality for the third inequality. Therefore, we have
 \begin{align*}
        \operatorname{KLV}\left(p_{*} || p_{\boldsymbol{\mu}_{\theta^\ddag\mid \mathcal{D}}, \boldsymbol{\Sigma}_{\theta^\ddag\mid \mathcal{D}}}\right)
         \lesssim \log(1+(q_n)^{1/2}\delta_n).
    \end{align*}
Adopting the notation of \cite{wong1995probability}, we set $\tau_n= \log(1+(q_n)^{1/2}\delta_n)$.

\textbf{Combining the pieces together}
Let $\epsilon_n^*=\epsilon_n \vee \sqrt{\delta_n}$. Then by Theorem 4 of \citep{wong1995probability}, there exists an absolute constant $C''>0$ such that
\begin{equation*}
    \begin{split}
        P_*\left(h\left(\hat{p}, p_*\right)>C_2\epsilon_n^*\right) 
        \lesssim  e^{-C'' n (\epsilon_n^{*})^2}+\frac{ \tau_n}{n (\epsilon_n^{*})^2},
    \end{split}
\end{equation*} 
which tends to zero as $n\to\infty$ by the assumptions $\epsilon^{*}_{n}q_n\to0$ and $n(\epsilon^{*}_{n})^2\to\infty$.
\end{proof}

\paragraph{Proof of Theorem \ref{theory-new}}

\begin{proof}
Since the total variation norm is upper bounded by the Hellinger distance, 
the total variation norm between $\hat{p}$ and $p_*$ is also upper bounded by $C_2 \epsilon_*$
with probability converging to 1.
In turn, by the definition of the total variation norm, $\sup_{\boldsymbol{x}\in \mathcal{D}^{\mathrm{design}}} 
d_1(\hat{p}_{\boldsymbol{x}}, p_{*,\boldsymbol{x}})$ is upper bounded by $C_2 \epsilon_*$
with probability converging to 1.
Due to the Lipschitz condition of $\hat\mu,\hat\Sigma$ as well as $\mu_*,\Sigma_*,$ 
there exists a constant $L>0$ such that
$d_1(\hat{p}_{\boldsymbol{x}},\hat{p}_{\boldsymbol{x}'}) \le L \|\boldsymbol{x} - \boldsymbol{x}'\|$
and
$d_1(p_{*,\boldsymbol{x}},p_{*,\boldsymbol{x}'}) \le L \|\boldsymbol{x} - \boldsymbol{x}'\|$
for any $\boldsymbol{x}$ and $\boldsymbol{x}'$ in $\mathbb{R}^d.$
Finally, we have
\begin{eqnarray*}
d_1(\hat{p}_{\boldsymbol{x}},p_{*,\boldsymbol{x}})
&\le& d_1(\hat{p}_{\boldsymbol{x}},\hat{p}_{\boldsymbol{x}_{(1)}})
+ d_1(\hat{p}_{\boldsymbol{x}_{(1)}},p_{*,\boldsymbol{x}_{(1)}})
+ d_1(p_{*,\boldsymbol{x}_{(1)}},p_{*,\boldsymbol{x}})\\
&\le& d_1(\hat{p}_{\boldsymbol{x}_{(1)}},p_{*,\boldsymbol{x}_{(1)}})+2L |\boldsymbol{x} - \boldsymbol{x}_{(1)}\|\\
&\le& C_2 \epsilon^*_n +2L |\boldsymbol{x} - \boldsymbol{x}_{(1)}\|
\end{eqnarray*}
with probability converging to 1. The proof is complete by letting $C_3=C_2$ and $C_4=2L.$

\end{proof}

\end{document}